\documentclass{article}

 \usepackage[preprint]{neurips_2026}

\usepackage[utf8]{inputenc} 
\usepackage[T1]{fontenc}    
\usepackage{hyperref}       
\usepackage{url}            
\usepackage{booktabs}       
\usepackage{amsfonts}       
\usepackage{nicefrac}       
\usepackage{microtype}      
\usepackage{xcolor}         
\usepackage{graphicx}
\usepackage{natbib}
\usepackage{doi}
\usepackage{empheq}
\usepackage{enumitem}
\usepackage{caption} 

\usepackage{graphics} 
\usepackage{epsfig} 
\usepackage{times} 
\usepackage{amsmath} 
\usepackage{amssymb}  
\usepackage{subcaption}
\usepackage{multicol}
\usepackage{pgfplots}
\pgfplotsset{compat=1.18}
\usetikzlibrary{patterns}
\usetikzlibrary{shapes.geometric, arrows.meta, positioning, calc}
\usepackage{wrapfig}
\usepackage{makecell}
\usepackage{adjustbox}
\usepackage{tikz}
\usetikzlibrary{fit}
\usetikzlibrary{arrows.meta, calc}

\definecolor{Corange}{RGB}{217,95,2}    \definecolor{CorangeL}{RGB}{253,232,206}
\definecolor{Cblue}{RGB}{21,101,192}    \definecolor{CblueL}{RGB}{221,238,255}
\definecolor{Cgreen}{RGB}{46,125,50}    \definecolor{CgreenL}{RGB}{220,237,200}
\definecolor{Cpurple}{RGB}{106,27,154}  \definecolor{CpurpleL}{RGB}{237,231,246}
\definecolor{Cteal}{RGB}{0,105,92}      \definecolor{CtealL}{RGB}{224,242,241}
\definecolor{Cred}{RGB}{183,28,28}      \definecolor{CredL}{RGB}{255,235,238}

\tikzset{
  myarr/.style={-{Stealth[length=9pt,width=6pt]}, line width=2.2pt, color=#1},
  myarr/.default=black,
  trunk/.style={line width=2.2pt, color=#1, line cap=round},
  trunk/.default=black,
}

\usepackage[utf8]{inputenc} 
\usepackage[T1]{fontenc}    
\usepackage{hyperref}       
\usepackage{url}            
\usepackage{booktabs}       
\usepackage{amsfonts}       
\usepackage{nicefrac}       
\usepackage{microtype}      
\usepackage{xcolor}         
\usepackage{float}

\usepackage{amsmath}
\usepackage{amssymb}
\usepackage{mathtools}
\usepackage{amsthm}

\usepackage{tabularray}
\usepackage{adjustbox}
\usepackage{graphicx}
\usepackage{dsfont}
\usepackage{svg}

 \usepackage[ruled,vlined]{algorithm2e}

\usepackage{xfrac}

\theoremstyle{plain}
\newtheorem{theorem}{Theorem}[section]
\newtheorem{proposition}[theorem]{Proposition}
\newtheorem{lemma}[theorem]{Lemma}
\newtheorem{corollary}[theorem]{Corollary}
\theoremstyle{definition}
\newtheorem{definition}[theorem]{Definition}
\newtheorem{assumption}[theorem]{Assumption}
\theoremstyle{remark}
\newtheorem{remark}[theorem]{Remark}
\title{Curvature-Guided LoRA: \\
Matching Full Fine-Tuning in Function Space}

\author{%
  Frédéric Zheng \\
  KTH, Stockholm, Sweden \\
  \texttt{fzheng@kth.se} \\
  \And
  Alexandre Proutiere \\
  KTH, Stockholm, Sweden \\
  \texttt{alepro@kth.se} \\
}

\begin{document}

\maketitle

\begin{abstract}
Parameter-efficient fine-tuning methods such as LoRA enable efficient adaptation of large pretrained models, but often lag behind full fine-tuning in both convergence speed and final performance. Recent approaches aim to reduce this gap by aligning LoRA parameter updates with those of full fine-tuning, but such parameter-space alignment only indirectly controls model predictions. Instead, we adopt a function-space perspective and formulate the \emph{prediction alignment problem}, whose objective is to match the outputs of LoRA fine-tuning to those of full fine-tuning. We show that this objective naturally leads to a curvature-aware, second-order formulation, where optimal low-rank updates correspond to a Newton-like, curvature-whitened gradient. Based on this insight, we propose Curvature-Guided LoRA (CG-LoRA), an algorithm that selects adaptation directions using local curvature information. Our method is computationally efficient and avoids explicit second-order matrix construction. Experiments on standard natural language understanding benchmarks demonstrate improved performance and faster convergence compared to existing LoRA variants. Code can be found at \url{https://github.com/fredML/CG-LoRA}.
\end{abstract}


\section{Introduction}

Parameter-efficient fine-tuning methods \cite{han2024parameter,houlsby2019parameter}, such as LoRA \cite{hu2022lora}, have become a standard approach for adapting large pretrained models, as they substantially reduce memory and computational costs compared to full fine-tuning \cite{biderman2024lora}. However, a persistent gap remains between LoRA and full fine-tuning, both in convergence speed and in final performance. A recent line of work aims to close this gap by designing alignment-aware variants of LoRA, typically by matching the first-step LoRA update to that of full fine-tuning, or by modifying the gradient structure accordingly \cite{gawang2024lora,zhang2025lora}. While effective, these approaches operate in parameter space and therefore only indirectly control the quantity of primary interest: the model's predictions. This motivates the central question of this work: {\it What is the appropriate performance-based notion of alignment between LoRA and full fine-tuning, and how can LoRA be constructed to achieve it?}

We address this question by formulating the \emph{prediction alignment problem}, whose goal is to ensure that the predictor obtained after LoRA fine-tuning matches as closely as possible the one obtained by full fine-tuning. Rather than aligning parameter updates, we directly target alignment in function space, namely at the level of the model outputs after fine-tuning. This shift in perspective is crucial: distinct parameter updates can induce similar predictions, while small discrepancies in parameters may lead to large functional differences. Our main contribution is to show that the prediction alignment objective naturally leads to a curvature-aware, second-order formulation. Specifically, we demonstrate that aligning LoRA with full fine-tuning requires accounting for the local geometry of the loss landscape, captured through a curvature approximation around the pretrained model. This yields a Newton-like initialization of the low-rank adapters, in which the optimal subspaces are obtained by whitening the gradient using curvature information. In contrast to prior work based solely on first-order gradients, our method leverages second-order structure to identify the directions with the largest impact on predictions. The approach and its benefits are illustrated in Figure~\ref{fig:cglora}. Our contributions are as follows:

\begin{figure}[t!]
     \centering     
     \includegraphics[width=0.55\textwidth]{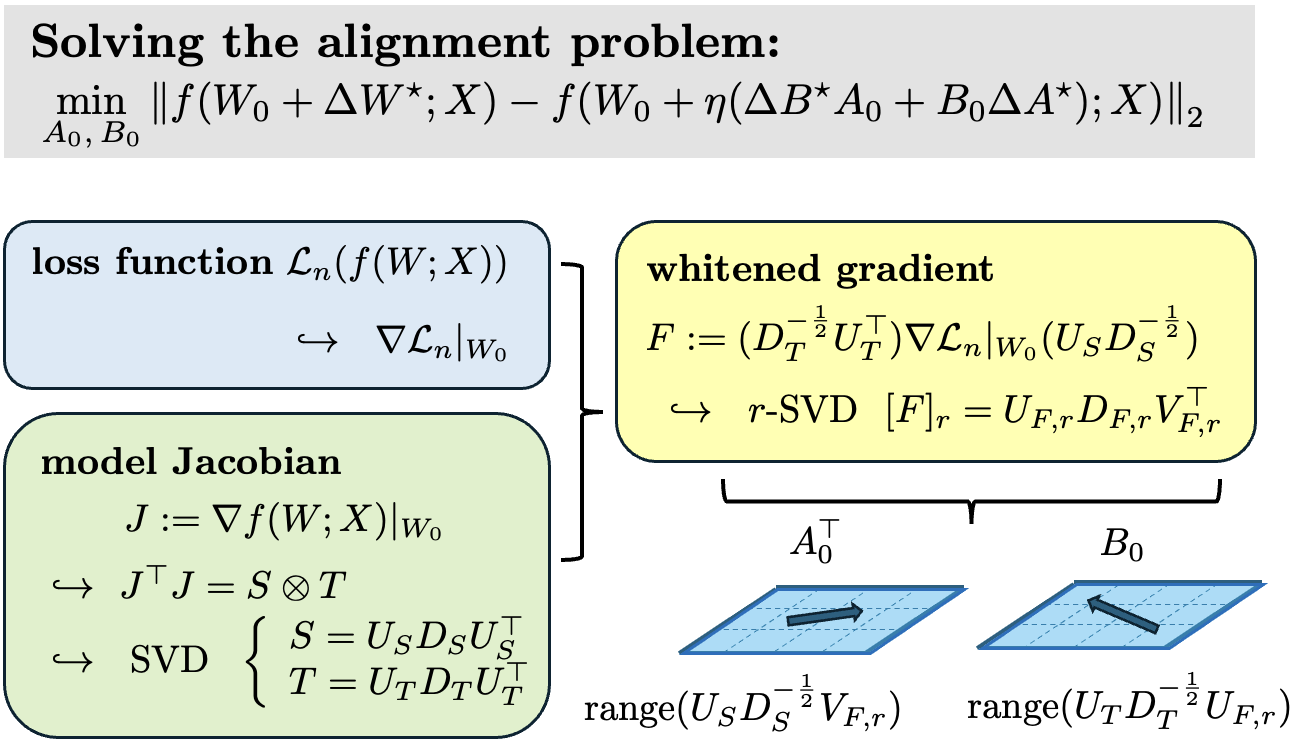}
     \hspace{-0.2cm}
     \raisebox{0.02cm}{\includegraphics[width=0.45\textwidth]{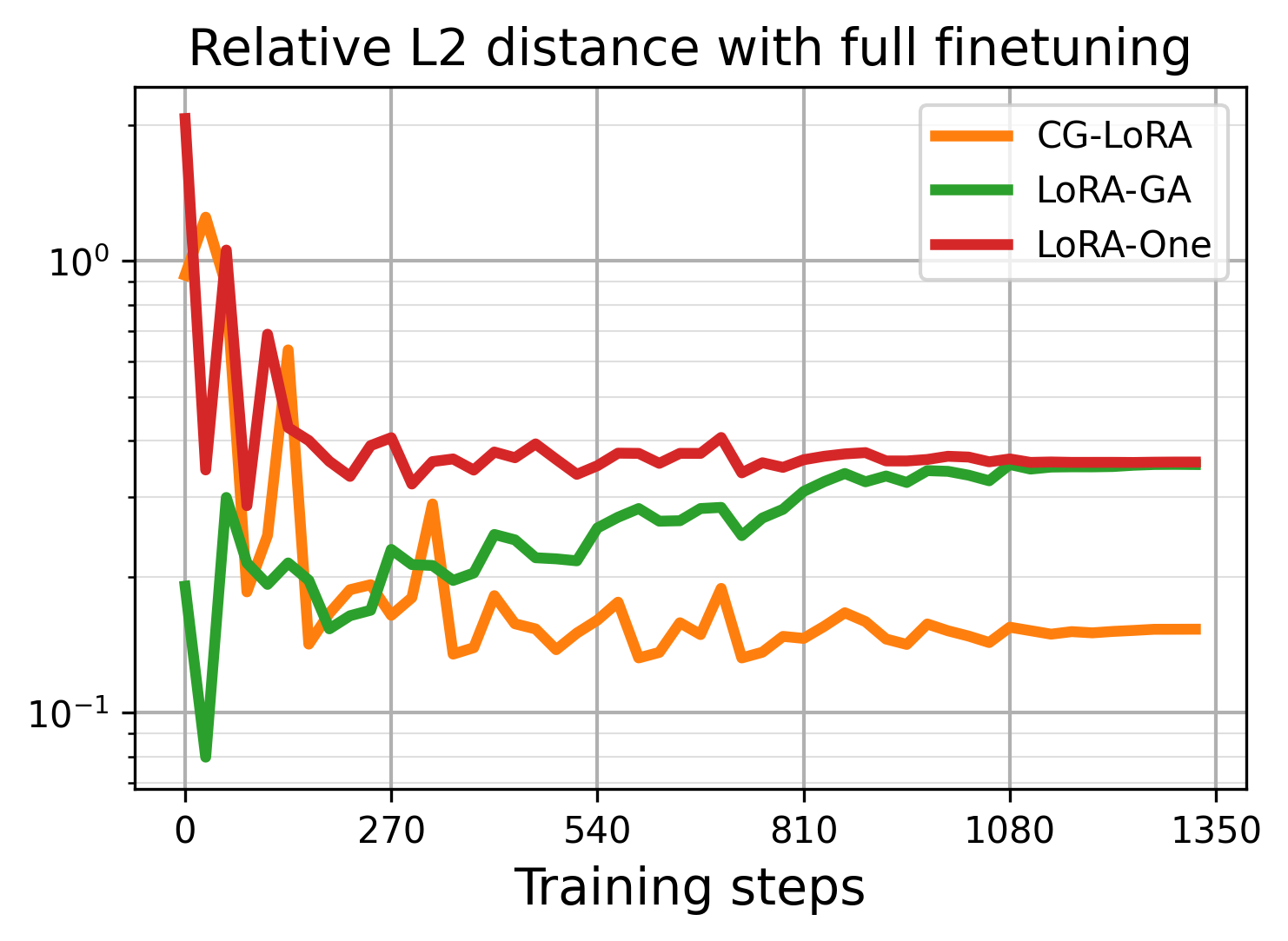}}
\caption{Left: the procedure to solve the alignment problem, as stated in Theorem \ref{thm:opt_align}. Right: Relative distance between LoRA and FFT training logits ${\Vert \hat f_{FT,t}(X)-\hat{f}_{LoRA,t}(X)\Vert_2}/{\Vert \hat f_{FT,t}(X)\Vert_2}$ using RoBERTa-base on CoLA averaged over three runs. LoRA rank $r=8$. }
{\vspace{-0.85cm}}
    \label{fig:cglora}
\end{figure}

(a) We formulate the alignment problem and analyze it through the Neural Tangent Kernel (NTK) approximation. We show that the resulting optimization problem induces a principled choice of the low-rank initialization matrix \(B_0A_0\): the column spaces of \(A_0^\top\) and \(B_0\) should align with the leading singular subspaces of a low-rank approximation of the \emph{whitened gradient}, obtained by combining the model Jacobian with the gradient of the loss function. This result, summarized in Theorem~\ref{thm:opt_align}, provides a function-space characterization of alignment. Our analysis recovers and extends the results of \cite{zhang2025lora}, which were established under highly specific data-generating assumptions, such as a linear ground-truth model or a single ReLU layer. In the main part of the paper, we present the analysis for the squared loss. The cross-entropy loss is treated in the appendix; 
to the best of our knowledge, this is the first theoretical analysis of non-zero LoRA initialization under cross-entropy loss.

(b) Building on this characterization, we propose Curvature-Guided LoRA (CG-LoRA), an initialization and adaptation scheme in which \(A_0\) and \(B_0\) are chosen according to the singular subspaces identified by the prediction-alignment analysis in (a), while also preserving alignment of the first-step model update, in the spirit of \cite{gawang2024lora,zhang2025lora}. We further show that the method can be implemented efficiently, without explicitly forming large curvature matrices, by relying on structured approximations.

(c) Empirically, we show that CG-LoRA converges faster and achieves better performance than existing LoRA variants on standard benchmarks, illustrating the value of function-space alignment and second-order information for parameter-efficient adaptation. We also demonstrate that CG-LoRA is computationally efficient and provides substantial memory savings.

{\bf Notation.} For all matrices $M = (M_1 \;\cdots\; M_{d_2}) \in \mathbb{R}^{d_1 \times d_2}$
with columns $M_j \in \mathbb{R}^{d_1}$ and $N$, we define the
\emph{column-stacked vectorization} and the \emph{Kronecker product} as
$    \operatorname{vec}(M)
    := \bigl(M_1^\top,\;\ldots,\;M_{d_2}^\top\bigr)^\top \in \mathbb{R}^{d_1 d_2},\quad
    M \otimes N
    :=
    \begin{pmatrix}
        M_{1,1}\,N & \cdots & M_{1,d_2}\,N \\
        \vdots     & \ddots & \vdots        \\
        M_{d_1,1}\,N & \cdots & M_{d_1,d_2}\,N
    \end{pmatrix},
    \label{eq:vec_kron}
$
where $M_{i,j}$ denotes the $(i,j)$-th entry of $M$. We denote by $\|\cdot\|_2$ the Euclidean norm for vectors and by
$\|\cdot\|_F$ the Frobenius norm for matrices. For a matrix $M  \in \mathbb{R}^{d_1 \times d_2}$, we denote by $M^\dagger$ its Moore-Penrose pseudo-inverse and $\Pi_M=MM^\dagger=M(M^\top M)^\dagger M^\top$ the orthogonal projection over the column space of $M$. We denote by $s_1(M)\geq...\geq s_{\min(d_1,d_2)}(M)\geq 0$ the ordered singular values of $M$ and by $s_{\min}^+(M)=s_{\operatorname{rank}(M)}(M)$ the smallest non-zero singular value of $M$. We call a thin singular value decomposition (SVD) of $M$ any decomposition $M=U\operatorname{Diag}(s_1(M),..., s^{+}_{\min} (M))V^\top$ with semi-orthogonal $U\in\mathbb{R}^{d_1\times \operatorname{rank}(M)}, V \in\mathbb{R}^{d_2\times \operatorname{rank}(M)}$. The set of invertible squared matrices of size $r$ is denoted by $\mathbb{GL}_r.$ The identity matrix of size $(d,d)$ is denoted by $I_d$, and the rectangular matrix of size $(d_1,d_2)$ filled with zeros is denoted by $0_{d_1,d_2}$. Finally, for $r\in \mathbb{N}$, $[r]:=\{1,\ldots,r\}$.

\section{Related work}

\paragraph{Low-rank adaptation.}
Following the introduction of LoRA \cite{hu2022lora}, a substantial body of work has studied how to improve its empirical performance and clarify its theoretical behavior.
On the optimization side, LoRA$+$ \cite{hayou2024lora+} argues that using different learning rates for the two low-rank factors is preferable to tying them.  From a computational perspective, several works aim at either reducing the number of total parameters \cite{kopiczko2023vera} or the complexity cost of each operation \cite{chen2025lora}.
In a related direction, rank-stabilized LoRA (rsLoRA) \cite{kalajdzievski2023rank} studies the rank-dependent scaling used in LoRA and proposes a modified scaling. A complementary line of work investigates how the chosen rank controls the expressive power of LoRA updates and when LoRA can match full fine-tuning.
\cite{zeng2023expressive} provides theoretical results on the expressive power of LoRA in neural networks. 
In the neural tangent kernel lazy regime, \cite{jang2024lora-ntk} derives sufficient rank conditions under which the LoRA objective can recover solutions comparable to full fine-tuning.

The main motivation of our work is the recent focus on \emph{initialization} of low-rank adapters.
The original LoRA initialization sets the initial update to zero (by initializing one factor at zero and the other randomly \cite{he2015delving}), which ensures the fine-tuning trajectory starts exactly at the pre-trained solution. 
Beyond this baseline, \cite{li2025beyond} argues that allowing non-zero initializations can improve robustness. \cite{li2024crucial} proposes a Nyström initialization based on findings on low-rank matrix factorization, and applies it to LoRA. 
Several works also propose to leverage information in the pre-trained weights and/or the downstream task to construct data- or geometry-informed initial adapters.
PiSSA \cite{meng2024pissa} initializes adapters using the SVD of the pre-trained weight matrix, effectively seeding the update in directions aligned with dominant singular components of $W_0$. 
Other approaches \cite{gawang2024lora,zhang2025lora} attempt at leveraging low-rank subspaces where gradient descent happens,  \cite{aghajanyan2021intrinsic, gur2018gradient}, and construct SVD-based initializations from the pretrained loss-gradient. 
The above works can be viewed as refinements of the original LoRA recipe (e.g., learning-rate choice, scaling, rank, initialization) while preserving the basic low-rank update parameterization.
For completeness, we also highlight methods that modify the allocation or the parameterization itself.
AdaLoRA \cite{zhang2023adalora} adaptively allocates rank (i.e., parameter budget) across layers during training. DoRA \cite{liu2024dora} proposes to decompose the pre-trained weight and identify trainable magnitude and directions leveraging more efficient allocation. \cite{zhang2024spectral} parameterizes LoRA  within a singular-subspace representation of $W_0$ rather than via an additive low-rank update. From the gradient descent perspective, \cite{zhang2024riemannian} applies a preconditioning step to the usual gradients of the adapters while
\cite{wang2024lora}  proposes a novel gradient update expression.

\paragraph{Dynamics of LoRA.}
The Neural Tangent Kernel (NTK) has become a central tool to understand gradient descent dynamics in wide neural networks through linearization around initialization, \cite{jacot2018neural,lee2019wide}. By analyzing the NTK, a rich literature has been developed on deriving generalization bounds for overparametrized neural networks, see \cite{allen2019learning, arora2019fine, oymak2019generalization} and references herein.
While classical results focus on random initialization for the layer weight matrices, \cite{malladi2023kernel} provides justification and numerical evidence that linearization and fixed features Jacobian behaviour (we refer the reader to their Definition 3.2 for details) can persist when fine-tuning from a {pre-trained} initialization.
Studying LoRA through this lens has yielded insight into rank requirements and optimization landscape properties; allowing \cite{jang2024lora-ntk} to apply Burer-Monteiro type arguments, while \cite{liu2025optimization} analyzes a $\operatorname{PL}^\star$ condition on the kernel. Beyond the  lazy regime studied in the NTK, dynamics of LoRA are harder to characterize and fewer works exist.  \cite{kim2025lora-ntk} extends the work  of \cite{jang2024lora-ntk} undertaken in the NTK regime. The works of \cite{dayi2024gradient} and \cite{zhang2025lora} analyze in-depth the gradient descent dynamics in very specialized instances, namely rank-1 LoRA or linear/single ReLU assumptions for the finetuning tasks. 




\section{Model and objective}

We consider a pretrained model $f$ that we wish to fine-tune using the dataset $\{(x_i,y_i)\}_{i=1}^n$. Specifically, we focus on fine-tuning the weight matrix $W\in\mathbb{R}^{d_\mathrm{out}\times d_\mathrm{in}}$ of a given layer, and we denote by $f(W;x)\in\mathbb{R}^C$ the model output as a function of the input $x$ and this matrix. We begin by introducing notation that will be used throughout the remainder of the paper. The output of the model can be written component-wise as
$f(W;\,x)
    := \bigl(f^{(1)}(W;\,x),\;\ldots,\;f^{(C)}(W;\,x)\bigr)^\top \in \mathbb{R}^C.
$   
%
%
We define the Jacobian of $f$ at $(W,\,x)$ as
$J_{f,W,x}
    :=
    (
        \operatorname{vec}\nabla_W f^{(1)}(W;\,x),\;
        \ldots,
        \operatorname{vec}\nabla_W f^{(C)}(W;\,x)
    )^\top
    \in \mathbb{R}^{C\times d_{\mathrm{out}} d_{\mathrm{in}}},
$    
as well as the Neural Tangent Kernel (NTK) linearization of $f$ at $W$ in direction $\Delta W \in
\mathbb{R}^{d_{\mathrm{out}} \times d_{\mathrm{in}}}$ as $f(W+\Delta W;\,x) \approx f(W;\,x) + J_{f,W,x}\operatorname{vec}(\Delta W)$. This approximation has been shown to hold throughout the fine-tuning procedure for the squared loss in infinite-width networks with random initialization \cite{jacot2018neural, lee2019wide}, uniformly over time (see Theorem~2.1 in~\cite{lee2019wide}). For pretrained LLMs used as initialization, this approximation has also been supported empirically \cite{malladi2023kernel, afzallinearization}. In the following, we refer to the {\it NTK regime} as the setting in which this linear approximation holds.

To simplify the notation, we introduce $X := (x_1, \ldots, x_n)^\top$ and
$Y := (y_1^\top, \ldots, y_n^\top)^\top \in \mathbb{R}^{nC}$, and define 
$
    f(W;\,X)
    := \bigl(f(W;\,x_1)^\top,\;\ldots,\;f(W;\,x_n)^\top\bigr)^\top
    \in \mathbb{R}^{nC}
$ and $
    J_{f,W,X}:= (J_{f,W,x_1}^\top \ \cdots J_{f,W,x_n}^\top)^\top\in\mathbb{R}^{(nC)\times (d_\mathrm{out}d_\mathrm{in})}.
$
For a given loss function $\ell$ and training points $X = (x_1, \ldots, x_n)^\top$ and
$Y = (y_1^\top, \ldots, y_n^\top)^\top \in \mathbb{R}^{nC}$, we define the empirical loss of predictors $ f(W;X)\in\mathbb{R}^{nC}$ as following:
$\mathcal{L}_n( f(W;X)):=\sum_{i=1}^n \ell(f(W;x_i),y_i).$

 \subsection{Parameter-efficient fine-tuning}

Assume that the model starts with pretrained matrix $W_0$. Fine tuning consists in optimizing $W=W_0+\Delta W$ around $W_0$ to minimize the empirical loss. When no constraints are imposed on $\Delta W$, we talk about Full Fine-Tuning (FFT):
\begin{equation}
  \textrm{min}_{\Delta W \in
\mathbb{R}^{d_{\mathrm{out}} \times d_{\mathrm{in}}}}  \;\mathcal{L}_n(f(W_0 + \Delta W;\, X)).
    \label{eq:ft}
\end{equation}
In the NTK regime, this optimization problem becomes: 
\begin{align}
    \textrm{min}_{\Delta W \in
\mathbb{R}^{d_{\mathrm{out}} \times d_{\mathrm{in}}}}  \;\mathcal{L}_n\Big(f(W_0;X)+J_{f,W_0,X}\operatorname{vec}(\Delta W)\Big). 
     \tag{Full Fine-Tuning}
     \label{eq:ft_ntk}
\end{align} 
We denote by $\Delta W^\star$ the minimum-norm solution, i.e., the parameter update found by gradient descent after convergence. The NTK fine-tuned model can be hence  written as $f(W_0;x)+ J_{f,W_0,x}\operatorname{vec}(\Delta W^\star)$ (see Theorem 3.2 in \cite{arora2019exact} for a precise formulation).

LoRA constrains the weight perturbation $\Delta W$ to take the form $\eta BA$, where $A \in \mathbb{R}^{r \times d_{\mathrm{in}}}$ and $B \in \mathbb{R}^{d_{\mathrm{out}} \times r}$. Here, $r$ denotes the rank of the perturbation and is assumed to be much smaller than the input and output dimensions, i.e., $r \ll \min(d_{\mathrm{in}}, d_{\mathrm{out}})$. The parameter $\eta > 0$ is a fixed scaling factor that controls the magnitude of the adaptation~\cite{hu2022lora}.

 
\paragraph{Reparameterization around a fixed LoRA initialization $(A_0, B_0)$.} In the default initialization of LoRA \cite{hu2022lora}, $A_0$ is sampled using Kaiming uniform initialization while $B_0=0$. This implies that $B_0A_0=0$, ensuring that the initial weight matrix remains equal to the pretrained matrix $W_0$. In this work, we investigate improved initialization strategies and allow for choices such that $B_0A_0 \neq 0$. To this end, we write $A = A_0 + \Delta A$ and $B = B_0 + \Delta B$. With this notation, the product $BA$ can be approximated as $BA \approx B_0A_0 + B_0\Delta A + \Delta B A_0$, where we neglect the second-order term $\Delta B \Delta A$. Fine-tuning then begins from $\Delta A = 0$ and $\Delta B = 0$. Some recent works explicitly control the initialization offset. For instance, in LoRA-GA \cite{gawang2024lora}, the authors ensure that fine-tuning starts exactly from the pretrained matrix $W_0$ by shifting the latter to $W_0 - \eta B_0A_0$. In contrast, other approaches such as LoRA-One \cite{zhang2025lora} propose to exploit the flexibility offered by the initialization and deliberately modify the starting point of the weights through an appropriate choice of $A_0$ and $B_0$, so that $W_0 + \eta B_0A_0$ provides a better starting point than $W_0$ itself.

We adopt the same framework as that used in LoRA-GA, and search for a fine-tuned weight matrix of the form $W_0 + \Delta W = W_0 + \eta(B_0\Delta A + \Delta B A_0)$. 
It allows for a simpler, more intuitive analysis. In Section \ref{sec:algo}, we show how we can additionally incorporate the shift influence in our algorithmic design. 

The resulting LoRA fine-tuning procedure can therefore be written as follows:
\begin{equation}
  \textrm{min}_{\Delta A \in \mathbb{R}^{r \times d_{\mathrm{in}}}, \Delta B \in \mathbb{R}^{d_{\mathrm{out}} \times r}}  \;\mathcal{L}_n(f(W_0 + \eta(B_0\Delta A + \Delta B A_0);X))
    \label{eq:lora}
\end{equation}
In the NTK regime, this optimization problem becomes: 
\begin{align}
\textrm{min}_{\Delta A \in \mathbb{R}^{r \times d_{\mathrm{in}}}, \Delta B \in \mathbb{R}^{d_{\mathrm{out}} \times r}} \; \mathcal{L}_n\Big(f(W_0;X)+\eta J_{f,W_0,X}\operatorname{vec}(
           \Delta B A_0 + B_0 \Delta A )\Big).
     \tag{LoRA}
     \label{eq:lora_ntk} 
\end{align}
We denote by $\Delta A^\star, \Delta B^\star$ the minimum-norm solutions, i.e., the parameter update found by gradient descent after convergence. 

\subsection{Objective: the prediction alignment problem}

We aim to design a LoRA fine-tuning procedure that yields a model as close as possible to that obtained via full fine-tuning. Prior work~\cite{gawang2024lora} focuses on selecting $A_0, B_0$ so as to align the first-step parameter updates between LoRA and full fine-tuning. For gradient descent, this is equivalent to matching the initial loss gradients. Under additional assumptions on the data distribution, namely that $y$ is generated by a linear or single-layer ReLU function of isotropic inputs $x$, \cite{zhang2025lora} extends this analysis beyond the first step and characterizes the alignment of iterative gradient updates within a subspace, using a similar SVD-based initialization of the loss gradient. However, for highly nonlinear architectures such as large language models, even small parameter perturbations can lead to significant discrepancies between the full fine-tuning and LoRA losses, making gradient alignment an unreliable global surrogate. To address this limitation, we instead seek to choose $A_0, B_0$ such that, after fine-tuning, the resulting predictor closely matches that obtained via full fine-tuning. Specifically, we consider the optimization problem:
 \begin{equation}
     \min_{A_0,\,B_0}
     \left\lVert
         f\!\left(W_0+{\Delta W}^\star;\,X\right)
         -
         f\!\left(W_0+\eta({\Delta B}^\star A_0 + B_0{\Delta A}^\star);\,X\right)
     \right\rVert_2.
     \label{eq:pred_obj}
 \end{equation}

\section{Solving the Prediction Alignment Problem}

Next, we derive key results for solving the alignment problem (\ref{eq:pred_obj}). For clarity of exposition, we focus on the squared loss; the extension to cross-entropy loss for classification is done in Appendix \ref{app:cross_ent}.

\subsection{NTK loss minimizers}

\begin{proposition}
Consider the squared loss $\ell(f(x), y)=\frac{1}{2}\Vert y-f(x)\Vert_2^2$. Let $Z=Y-f(W_0;X)$ denote the fine-tuning residuals. \\
(i) Full Fine-Tuning (FFT). The minimum-norm solution \\${\Delta W}^\star = \arg\min_{\Delta W}\,\mathcal{L}_n\left(f(W_0;X)+J_{f,W_0,X}\operatorname{vec}(\Delta W\right))$ is given by
\begin{equation}
    \operatorname{vec}(\Delta W^\star)=J_{f,W_0,X}^\dagger Z.
    \label{eq:ft_ntk_preds}
\end{equation}
(ii) LoRA. Define the LoRA Jacobian as 
$J_{f,W_0,A_0,B_0,X} = J_{f,W_0,X}(I_{d_\mathrm{in}}\otimes B_0 \quad A_0^\top \otimes I_{d_\mathrm{out}}).$
The minimum-norm solution $\bigl(\Delta A^\star,\,\Delta B^\star\bigr)
= \arg\min_{\Delta A,\,\Delta B}\,\mathcal{L}_n(f(W_0;X)+\eta J_{f,W_0,X}\operatorname{vec}(
           \Delta B A_0 + B_0 \Delta A ))$ is given  by
\begin{equation}       
        (\operatorname{vec}(
   \Delta A^\star)^\top
    \operatorname{vec}(\Delta B^\star)^\top
    )^\top
 = \frac{1}{\eta}J_{f,W_0,A_0,B_0,X}^\dagger Z.
\label{eq:lora_ntk_preds}
\end{equation}
\label{prop:ntk_preds}
\end{proposition}

\begin{corollary}
The objective function in (\ref{eq:pred_obj}) is equal to $\Vert (\Pi_{J_{f,W_0,X}} - \Pi_{J_{f,W_0,A_0,B_0,X}})Z \Vert_2$ for any $A_0, B_0$. Hence (\ref{eq:pred_obj}) is equivalent to 
\begin{equation}
    \min_{A_0,\,B_0}
    \Vert (\Pi_{J_{f,W_0,X}} - \Pi_{J_{f,W_0,A_0,B_0,X}})Z \Vert_2.
    \label{eq:pred_obj2}
\end{equation}
\label{corr:diff_proj}
\end{corollary}
\vspace{-1em}
Interestingly, our formulation does not depend on the choice of $\eta$. Therefore, our results are compatible with commonly chosen values  $\eta=\frac{1}{r}$ or its rank stabilized version \cite{kalajdzievski2023rank} $\eta=\frac{1}{\sqrt r}$.
\subsection{The K-FAC design}
The expression of the LoRA Jacobian shows that the initialization matrices $(A_0,B_0)$ affect predictions through the right singular subspace of $J_{f,W_0,X}$. This subspace coincides with the eigen-subspace of $J_{f,W_0,X}^\top J_{f,W_0,X}$ which is the Generalized Gauss-Newton (GGN) approximation of the Hessian of the squared loss w.r.t. $W$. 
Let $h_i$ and $u_i$ the input/output of linear layer $W$ w.r.t. example $x_i$. Assume first no weight sharing, i.e., $h_i\in\mathbb{R}^{d_\mathrm{in}}$ and $u_i\in\mathbb{R}^{d_\mathrm{out}}$ are vectors, then by the chain rule, one has $\nabla_W f^{(c)}(W_0;x_i) = \nabla_{u_i} f^{(c)}(W_0;x_i) h_i^\top=\delta_i^{(c)}h_i^\top$, 
where $\delta_i^{(c)}:=\nabla_{u_i} f^{(c)}(W_0;x_i) $ is a pre-activation derivative. Let $\delta_i=(\delta_i^{(1)} \ \cdots \delta_i^{(C)})\in\mathbb{R}^{d_\mathrm{out}\times C}$, then (see Appendix \ref{app:kfac} for details)
$    J_{f,W_0,X}^\top J_{f,W_0,X}=\sum_{i=1}^n(h_i h_i^\top)\otimes (\delta_i \delta_i^\top)$.  
Since exact GGN computations are typically intractable at scale, K-FAC \cite{martens2015optimizing} introduces a structured Kronecker-factor approximation that enables efficient approximate natural-gradient updates \cite{amari1998natural}, and has proven competitive empirically. 
Subsequent work refines or analyzes these approximations, including EKFAC \cite{george2018fast} and mean-field analyses of Fisher statistics \cite{karakida2019universal}. Regarding LoRA finetuning, a few works have also recently  started incorporating K-FAC designs, which we discuss in Appendix \ref{app:compa_kfac}.  
We adopt this approximation, formalized in the following assumption.

\begin{assumption}
     The GGN approximation of the squared-loss Hessian w.r.t. $W$ satisfies the following \textit{K-FAC} approximation:
    \begin{equation}J_{f,W_0,X}^\top J_{f,W_0,X} = S\otimes T
    \end{equation}
where the definition of $S$ and $T$ depends on whether the layer is weight-shared. \\
(i) Without weight sharing: $S = \frac{1}{n}\sum_{i=1}^nh_i h_i^\top$ and   $T=\sum_{i=1}^n \delta_i\delta_i^\top$.\\
(ii) With weight sharing:  we consider the two variants introduced by \cite{eschenhagen2023kronecker}, 1) \textit{K-FAC-reduce} for non-tokenized outputs and, 2) \textit{K-FAC-expand} for tokenized outputs. We refer the reader to Appendix \ref{app:kfac} for their exact expressions.
\label{assump:kfac}
\end{assumption}

The weight-sharing case requires special care, since both approximations simplify the exact token interactions. More precisely, K-FAC-expand discards off-diagonal cross-token terms, whereas K-FAC-reduce pools across tokens. Eschenhagen et al.~\cite{eschenhagen2023kronecker} establish theoretical exactness of these approximations in simplified settings, and also demonstrate their empirical effectiveness in modern architectures. Modeling richer cross-token structure remains an interesting direction for future work.

\subsection{The whitened gradient and optimal curvature-guided initialization}

To build intuition for solving~\eqref{eq:pred_obj}, it is useful to view prediction alignment as a problem of allocating the limited LoRA rank to the directions that matter most in function space. An effective initialization should therefore preserve the dominant directions through which low-rank parameter updates induce significant changes in the model outputs. If this geometry is ignored, part of the available rank may be wasted on directions that have large parameter-space signal but remain nearly invisible at the level of predictions. This distinction is largely absent from existing gradient-alignment analyses of LoRA, which are typically developed under restrictive linear assumptions. In such settings, parameter perturbations translate directly into output perturbations, and the relevant geometry is therefore much simpler. In deep nonlinear models, however, raw parameter-space gradients can be misleading: directions with large gradient norm may induce only weak functional variation if they are poorly excited by the pretrained model across samples and tokens. This motivates the use of a \emph{whitened gradient}, which incorporates the local geometry of the pretrained model. By suppressing directions that are weakly represented in function space and emphasizing those along which low-rank updates most effectively change predictions, the whitened gradient identifies the subspaces where the LoRA rank is best spent.

\begin{definition}
    Let $r_S=\mathrm{rank}(S), r_T=\mathrm{rank}(T)$, and let thin SVDs $S=U_S D_S U_S^\top$ and $T = U_TD_TU_T^\top$  where $U_S\in\mathbb{R}^{d_\mathrm{in}\times r_S}, D_S\in\mathbb{R}^{r_S\times r_S}$ and $U_T\in\mathbb{R}^{d_\mathrm{out}\times r_T},D_T\in\mathbb{R}^{r_T\times r_T}$. Define the \textit{whitened} gradient as
    \begin{equation}F:=D_T^{-1/2} U_T^\top \nabla_W\mathcal{L}_n(f(W;X))|_{W_0}U_SD_S^{-1/2}\in\mathbb{R}^{r_T\times r_S}.
    \label{eq:F}\end{equation}
    For $r\leq\operatorname{rank}(F)$, let $F_r= U_{F,r}D_{F,r}V_{F,r}^\top,$ denote the best rank-$r$ approximation of $F$. The corresponding full column rank left and right modes are $L_r=U_TD_T^{-1/2}U_{F,r}$ and $R_r=U_SD_S^{-1/2}V_{F,r}$.
\end{definition}
Constructed from the thin SVDs of \(S\) and \(T\), the whitened gradient emerges as the key object in our solution to the prediction-alignment problem. For ease of exposition, we state only the low-rank case $r\leq \operatorname{rank}(F)$ and refer the reader to Theorem \ref{thm:full_opt_align} in Appendix for the full version.

\begin{theorem}
(i) If Assumption \ref{assump:kfac} holds, then for any $(A_0,B_0)$,
    \begin{align}
        \Vert (\Pi_{J_{f,W_0,X}} - \Pi_{J_{f,W_0,A_0,B_0,X}})Z \Vert_2 &= \Vert (I_{r_T}-\Pi_{D_T^{1/2}U_T^\top B_0})F(I_{r_S}-\Pi_{D_S^{1/2} U_S^\top A_0^\top})\Vert_F\\
        &\geq \big(\sum_{i\geq 2r+1}s_i^2(F)\big)^{1/2}.
        \label{eq:lb}
    \end{align}
(ii) Assume $r\leq \operatorname{rank}(F)$. Let $\mathcal I,\mathcal J$ be index sets of cardinality $r$ and define $U_{F,\mathcal{J}}$ and $V_{F,\mathcal{I}}$ the left and right subsets of singular vectors of $F$ indexed by $\mathcal{J}$ and $\mathcal{I}$, respectively. Any       $$\begin{cases}
        A_0^\top \in \mathcal{G}_{A,{\cal I}}=\operatorname{range}( U_SD_S^{-1/2}V_{F,\mathcal I} Q_A: Q_A\in\mathbb{GL}_r) \\
        B_0 \in \mathcal{G}_{B,{\cal J}}=\operatorname{range}(U_TD_T^{-1/2}U_{F,\mathcal J} Q_B:Q_B\in\mathbb{GL}_r) 
    \end{cases}$$
satisfies $ \displaystyle \Vert (\Pi_{J_{f,W_0,X}} - \Pi_{J_{f,W_0,A_0,B_0,X}})Z \Vert_2=\big(\sum_{i\in(\mathcal I\cup \mathcal J)^c} s_{i}^2(F)\big)^{1/2}.$ \\
Hence, if we select ${\cal I}\cup {\cal J}=[\min(\operatorname{rank}(F),2r)]$, in view of (\ref{eq:lb}), $(A_0, B_0)$ solves the alignment problem (\ref{eq:pred_obj}).
\medskip
 \label{thm:opt_align}
\end{theorem}

\begin{remark}\label{rem:F} 

Observe that: $\forall i=1,\ldots,\operatorname{rank}(F),$
\begin{equation}\frac{s_i(\nabla_W\mathcal{L}_n(f(W;X))|_{W_0})}{\sqrt{s_1(J_{f,W_0,X}^\top J_{f,W_0,X})}}
    \ \leq\ 
    s_i(F)
    \ \leq\ 
    \frac{s_i(\nabla_W\mathcal{L}_n(f(W;X))|_{W_0})}{\sqrt{s_{\min}^{+}(J^\top_{f,W_0,X} J_{f,W_0,X})}}.
    \label{eq:spectrum_f}
\end{equation} Together with Theorem \ref{thm:opt_align}, this suggests that LoRA fine-tuning should be applied to layers whose corresponding NTK has a large minimum eigenvalue. This connection opens up new directions for designing theoretically grounded rank-selection rules, by leveraging existing results on the NTK spectrum of deep neural networks \cite{murray2022characterizing,nguyen2021tight,bombari2022memorization}. We leave this investigation for future work.
\end{remark}

\section{The CG-LoRA algorithm}\label{sec:algo}

In the preceding analysis, we identified the subspaces in which \(A_0\) and \(B_0\) should lie in order to optimize the prediction alignment. Several degree of freedom remain: we can select (i) ${\cal I}$ and ${\cal J}$, (ii) the relative scaling of the columns of \(A_0\) and \(B_0\), governed by the choice of \((Q_A,Q_B)\), and (iii) the global scaling of \(\|A_0\|\) and \(\|B_0\|\) under an appropriate norm. We discuss these algorithmic choices in this section and then provide a complete description of CG-LoRA. The corresponding pseudo-code is given in Appendix~\ref{app:algo}, Algorithms~\ref{alg:cg-lora}, \ref{alg:subspace}, and~\ref{alg:bal_svd}.

\subsection{Design choices}

For algorithm design, one may complement the prediction-alignment objective with alignment of the initial gradient updates. Such first-step alignment can also control subsequent updates under additional assumptions, for instance when the true model is linear \cite{zhang2025lora}. This can be achieved by first choosing \({\cal I}={\cal J}=[r]\) (potentially incurring a mild prediction-alignment cost controlled by Theorem~\ref{thm:opt_align}), and then selecting \(Q_A\) and \(Q_B\) according to the following proposition.


\begin{proposition}
The solutions $(Q_A,Q_B)$ of $\min_{(Q_A,Q_B)\in GL_r}\Vert B_0A_0+\nabla_W\mathcal{L}_n(f(W;X))|_{W_0}\Vert_F$, where $A_0^\top=R_rQ_A$ and $B_0=L_rQ_B$ (as prescribed by Theorem \ref{thm:opt_align} with $\mathcal{I}=\mathcal{J}=[r]$), are defined through the following property: $Q:=Q_BQ_A^\top$ is invertible and satisfies 
        \begin{equation}
            L_rQR_r^\top=M:=-\Pi_{L_r} \nabla_W\mathcal{L}_n(f(W;X))|_{W_0}\Pi_{R_r}.
            \label{eq:q_val}
        \end{equation}
        In particular, since $L_r$ and $R_r$ are full column rank, there is a unique invertible minimizer \\  
        $Q^\star = -L_r^\dagger \nabla_W\mathcal{L}_n(f(W;X))|_{W_0}(R_r^\top)^\dagger=-(U_{F,r}^\top D_T^{-1}U_{F,r})^{-1}D_{F,r}(V_{F,r}^\top D_S^{-1}V_{F,r})^{-1}.$     
 \label{prop:q_star}
    \end{proposition}

There are still infinitely many pairs $(Q_A, Q_B)$ satisfying $Q_B Q_A^\top = Q^\star$. To remove this non-uniqueness, we restrict our attention to balanced realizations.

    \begin{definition}[Balanced realizations]
        Let $M=-\Pi_{L_r} \nabla_W\mathcal{L}_n(f(W;X))|_{W_0}\Pi_{R_r}$ with thin SVD $M=U_MD_MV_M^\top$. Then $B_0=U_MD_M^{1/2}, A_0^\top = V_MD_M^{1/2}$
        is called a \textit{balanced} realization. 
    \end{definition}
So far, we have provided a balanced choice of $A_0 \in \mathcal{G}_{A,[r]}$ and $B_0 \in \mathcal{G}_{B,[r]}$. However, prior work \cite{gawang2024lora, li2025beyond} has shown that an additional global scaling can improve forward and backward stability (see Theorem~3.2 in~\cite{gawang2024lora}). In conjunction with $\eta = \frac{1}{\sqrt{r}}$, we adopt the following scaling, inspired by these works:
$
\|A_0\|_2 = \|B_0\|_2 = d_{\mathrm{out}}^{1/4}/\gamma,
$
where $\gamma$ is a hyperparameter. 

\subsection{Algorithm}

Based on our analysis and design choices, we obtain the CG-LoRA algorithm, whose pseudo-code is given in Algorithm~\ref{alg:cg-lora}. CG-LoRA consists of three main steps: (i) computing the curvature subspaces, as described in Algorithm~\ref{alg:subspace}; (ii) computing the whitened gradient \(F\); and (iii) computing a balanced realization via Algorithm~\ref{alg:bal_svd}. Importantly, none of these steps requires forming or storing any \((d,d)\) matrix, for \(d \in \{d_{\mathrm{in}},d_{\mathrm{out}}\}\). This contrasts with LoRA-GA/One, which stores full loss gradients and, in practice, often must offload them to the CPU to keep GPU memory usage tractable. In step (iii), Algorithm~\ref{alg:bal_svd} returns $\operatorname{BalancedSVD}(L_r,R_r,D_{F,r})$, which provides a balanced realization of the low-rank initialization. This property is proved in Appendix \ref{app:algo}. CG-LoRA is summarized by:
    
\begin{equation}\begin{cases}
        (A_0,B_0) \gets \operatorname{Balanced SVD}(L_r,R_r, D_{F,r}), \\
        (A_0,B_0) \gets (\frac{d_\mathrm{out}^{1/4}}{\gamma}A_0,\frac{d_\mathrm{out}^{1/4}}{\gamma}B_0 ).
\end{cases}  \tag{CG-LoRA}\end{equation} 
Its alignment guarantees are provided by Theorem \ref{thm:opt_align} and Proposition \ref{prop:q_star}.


In CG-LoRA, only pre-activation derivatives are required. In Appendix~\ref{app:hutchinson}, we show how to compute them either exactly or approximately using a Hutchinson trace estimator. To reduce memory usage, we use a hooking mechanism to run Algorithm~\ref{alg:cg-lora} for each layer, discard all intermediate buffers, and then proceed to the next layer. This keeps the memory footprint minimal.

\section{Experiments on Natural Language Tasks}\label{sec:exp}
 \subsection{Implementation details}

We evaluate LoRA algorithms by fine-tuning LLaMa 2-7B \cite{touvron2023llama}
,T5-base \cite{raffel2020exploring} 
and RoBERTa-base \cite{liu2019roberta}. 
We compare 1) CG-LoRA, 2) LoRA-One, 3) LoRA-GA, 4) LoRA+, 5) rsLoRA,  
and consider a small-budget LoRA setting in which only the query--value projection matrices are adapted. This keeps the number of trainable parameters small: adapting all linear layers can require up to roughly five times more parameters and twice the training time for LLaMa 2-7B 
, while query--value adaptation still provides reasonably strong adaptation capacity; see Section~7.1 of \cite{hu2022lora}. 

Although our main analysis is developed for the squared loss, the definition of \(F\) in (\ref{eq:F}) can be heuristically applied to cross-entropy, as done in prior LoRA work \cite{zhang2025lora}. In this work, we go further and  instead treat the cross-entropy case explicitly in Appendix~\ref{app:cross_ent}. This leads to the corrected Algorithm~\ref{alg:cg-lora_ce}, with theoretical guarantees in Theorem~\ref{thm:cg_lora_multi}; this is the version used in all experiments.

For the K-FAC design, although only RoBERTa-base yields non-tokenized outputs (head classifier), we will use the \textit{reduce} settings for all three models as \cite{eschenhagen2023kronecker} showed that both settings \textit{reduce} and \textit{expand} have close performance while the former is computationally less intensive.
We also empirically compare two initialization choices: starting from the original pretrained weights \(W_0\), denoted by CG-LoRA (no shift), and starting from the shifted weights \(W_0+\eta B_0A_0\), denoted by CG-LoRA (shift). To ensure a fair comparison, all methods are run under identical experimental conditions, rather than reusing results reported in prior work. Implementation details and an ablation study are provided in Appendices~\ref{app:exp} and~\ref{app:abla}.

\subsection{Natural Language Understanding}

We fine-tune on several GLUE \cite{wang2018glue} classification tasks and evaluate accuracy on a separate test set. For T5-base, class scores (logits) are computed by teacher forcing each candidate class. To reduce randomness, we repeat each experiment with three different random seeds. Due to space constraints, the T5-base results are reported in Appendix~\ref{app:results}.


\begin{table}[t]
\small
\centering
\caption{Test accuracy of fine-tuned RoBERTa-base. LoRA rank $r=8$. Query/Value LoRA layers.}
\vspace{1em}
\label{tab:glue-lora}
\adjustbox{width=0.9\linewidth}{
\begin{tabular}{lccccc}
\toprule
& \textbf{MNLI} & \textbf{SST-2} & \textbf{CoLA} & \textbf{QNLI} & \textbf{MRPC}  \\
\midrule
Full
& 86.87$_{\pm0.24}$
& 93.99 $_{\pm 0.30}$
& 80.88 $_{\pm0.55}$ 
& 92.01$_{\pm 0.12}$
& $87.25_{\pm 0.80}$ 
 \\
rsLoRA
& 84.86$_{\pm0.02}$
& 92.88$_{\pm0.18}$
& $78.49_{\pm 0.58}$
& $91.21_{\pm 0.08}$
& $78.10_{\pm0.92}$
 \\

LoRA+
& $85.28_{\pm0.14}$
& $92.50_{\pm 0.38}$
& 79.32$_{\pm 0.27}$ 
& 89.69$_{\pm 1.04}$
& $82.84_{\pm 2.30}$
\\
\midrule

 
LoRA-GA
& 84.51$_{\pm 0.23}$
& $92.66_{\pm0.24}$
& 69.12$_{\pm0.00}$ 
& 89.75$_{\pm 0.33}$
& $72.30_{\pm 3.41}$
 \\
 LoRA-One
& 85.11$_{\pm0.21}$
& $78.89_{\pm 2.49}$
& $72.70_{\pm5.06}$ 
& 63.74$_{\pm19.44}$
& $74.67_{\pm5.67}$\\

\midrule

CG-LoRA (no shift) 
& 84.49$_{\pm 0.24}$
& \textbf{93.27}$_{\pm0.46}$
& \textbf{80.76}$_{\pm0.45}$ 
& \textbf{91.38}$_{\pm0.21}$
& \textbf{85.04}$_{\pm1.70}$ \\
CG-LoRA (shift)
& \textbf{85.57}$_{\pm0.11}$
& 92.73$_{\pm0.19}$
& $77.56_{\pm1.02}$ 
& \textbf{91.42}$_{\pm0.10}$
&  78.75$_{\pm 2.37}$\\
\bottomrule
\end{tabular}}
{\vspace{-0.5cm}}
\end{table}

For RoBERTa-base, the results in Table~\ref{tab:glue-lora} show that CG-LoRA outperforms the other methods. On RoBERTa-base, the early gradients are noisy because the classification head is randomly initialized and must also be fine-tuned. Since both LoRA-GA and LoRA-One initialize their adapters directly from the SVD of the loss gradient, their choices of \(A_0\) and \(B_0\) can be affected by this noise. In particular, on several datasets, these methods do not converge within a single training pass. Kaiming-based initialization methods, such as rsLoRA and LoRA+, are less sensitive to this issue, but they do not exploit information from the pretrained model. In contrast, CG-LoRA mitigates gradient noise by whitening the gradient with function-space curvature information, while still leveraging signal from the pretrained model. As a result, CG-LoRA (no shift) exhibits improved convergence, as shown in Figure~\ref{fig:roberta_cola_loss} and in the additional results in Appendix~\ref{app:results}, particularly on small datasets such as CoLA and MRPC. On larger datasets, CG-LoRA (no shift) and CG-LoRA (shift) perform comparably. 

\subsection{Natural Language Generation}

We consider three fine-tuning tasks: 
(i) fine-tuning on a 100k subset sampled from MetaMathQA \cite{yu2023metamath} and evaluating mathematical reasoning on GSM8K \cite{cobbe2021training}; 
(ii) fine-tuning on a 52k subset sampled from WizardLM \cite{xu2023wizardlm} and evaluating broad question understanding on MMLU \cite{hendrycks2020measuring}; and 
(iii) fine-tuning on a 100k subset sampled from Code-Feedback \cite{zheng2024opencodeinterpreter} and evaluating code generation on HumanEval \cite{chen2021evaluating}. We follow the evaluation protocol of LoRA-GA. The corresponding metrics are: 
(i) exact-match accuracy on GSM8K, comparing the generated numerical answer with the test answer ending in ``\#\#\#\# \{numerical value\}''; 
(ii) multiple-choice accuracy on MMLU; and 
(iii) pass@1 on HumanEval.

\begin{wrapfigure}{r}{0.65\textwidth}
\centering
\small
\centering
\vspace{-1em}
\captionof{table}{Test accuracy of finetuned LLaMa 2-7B. LoRA rank $r=8$. Query/Value LoRA layers.}
\label{tab:llama}
\begin{tabular}{lccc}
\toprule
 &\textbf{GSM8K} & \textbf{MMLU} & \textbf{HumanEval} 
\\ \midrule
rsLoRA
 & 
$50.64_{\pm 0.16}$& $45.01_{\pm 0.45}$ 
&  $21.34_{\pm 0.17}
$ \\ \midrule
 LoRA-GA
 &$43.66_{\pm 0.78}$
& $25.70_{\pm 1.34}$
&  $20.73_{\pm 0.21}$
\\
 LoRA-One
 &$52.16_{\pm 0.32}$
& 45.20$_{\pm 0.15}$
& $20.12_{\pm 0.38}$
\\ \midrule
CG-LoRA (shift)
 &\textbf{53.44}$_{\pm 0.24}$
& 43.65$_{\pm 0.36}$
&  20.73$_{\pm 0.29}$

\\
CG-LoRA (no shift)
 & \textbf{56.10}$_{\pm 0.21}$
& \textbf{45.50}$_{\pm 0.22}$
&  \textbf{23.17}$_{\pm 0.18}$
\\
\end{tabular}
{\vspace{-0.2cm}}
\end{wrapfigure}

Results are reported in Table~\ref{tab:llama}. Across all three tasks, CG-LoRA outperforms all baselines by a significant margin. Notably, it achieves performance close to previously reported results obtained by adapting {\it all} linear layers, while using only a fraction of the trainable parameters. We further analyze the behavior of the different LoRA methods through their training losses and loss-gradient norms in Figures~\ref{fig:llama_metamath}, \ref{fig:llama_code}, and~\ref{fig:llama_wizard} in Appendix~\ref{app:exp}. All methods appear to fit the training data, reaching similar final training losses. However, LoRA-GA and LoRA-One exhibit noisier gradient dynamics, whereas CG-LoRA shows a smoother decrease in gradient norm. This suggests that, on complex reasoning tasks, our initialization better exploits the adaptation capacity of the pretrained model, in line with our theoretical insights.

\subsection{Computational cost and memory footprint}

We empirically demonstrate the low overhead of CG-LoRA. We consider the most demanding setting, in which LoRA is applied to all linear layers. As shown in Table~\ref{tab:cost}, the peak GPU memory allocation of CG-LoRA is lower than that required by the fine-tuning process itself. Thus, CG-LoRA introduces no additional peak-memory overhead relative to fine-tuning. In terms of runtime, our method is also competitive, especially for larger models such as LLaMA~2-7B, where it is significantly faster than LoRA-GA/One. A complementary FLOP analysis is provided in Table~\ref{tab:flops} in Appendix~\ref{app:cost}.

\begin{table}[h!]
\small
\vspace{-0.3cm}
\centering
\caption{Cost on LLaMa 2-7B. All-linear LoRA layers. Measured on a single NVIDIA A100. }
    \begin{tabular}{ccccc }
    \toprule
          \makecell{LoRA training \\ (memory)}  & \makecell{CG-LoRA \\ (memory)} & \makecell{CG-LoRA \\ (init time)} & \makecell{LoRA-GA/One \\(memory)} & \makecell{LoRA-GA/One \\ (init time) }  
       \\
       
       \midrule
       23 GiB & 17.7GiB & 20s & 18.8 GiB & 79s
    \end{tabular}
{\vspace{-0.5cm}}
    \label{tab:cost}
    \end{table}


\section{Conclusion}

In this work, we introduced the prediction alignment problem for parameter-efficient fine-tuning, shifting the focus from aligning parameter updates to aligning model outputs. We showed that this objective naturally leads to a curvature-aware, second-order formulation, and proposed CG-LoRA, an efficient initialization method based on curvature-whitened gradients that avoids explicit construction of large curvature matrices. Although our theory is developed in the NTK setting, the resulting design principle remains effective in practical finite-model regimes, suggesting that function-space alignment provides a useful guide for LoRA initialization beyond first-order gradient matching.

Several directions remain open. On the practical side, it would be valuable to investigate richer curvature approximations that better capture cross-token interactions, and to evaluate CG-LoRA across a wider range of architectures, tasks, and modalities. On the theoretical side, an important next step is to understand finite-time fine-tuning dynamics beyond the NTK approximation, and to connect rank selection more directly to the spectrum of the underlying neural tangent kernel.

\newpage
\bibliographystyle{plain}
\bibliography{refs}

@InProceedings{jang2024lora-ntk,
  title = 	 {{L}o{RA} Training in the {NTK} Regime has No Spurious Local Minima},
  author =       {Jang, Uijeong and Lee, Jason D. and Ryu, Ernest K.},
  booktitle = 	 {Proceedings of the 41st International Conference on Machine Learning},
  pages = 	 {21306--21328},
  year = 	 {2024},
  editor = 	 {Salakhutdinov, Ruslan and Kolter, Zico and Heller, Katherine and Weller, Adrian and Oliver, Nuria and Scarlett, Jonathan and Berkenkamp, Felix},
  volume = 	 {235},
  series = 	 {Proceedings of Machine Learning Research},
  month = 	 {21--27 Jul},
  publisher =    {PMLR},
  url = 	 {https://proceedings.mlr.press/v235/jang24d.html},
}

@inproceedings{kim2025lora-ntk,
author = {Kim, Junsu and Kim, Jaeyeon and Ryu, Ernest K.},
title = {{LoRA training provably converges to a low-rank global minimum or it fails loudly (but it probably won't fail)}},
year = {2025},
publisher = {JMLR.org},
booktitle = {Proceedings of the 42nd International Conference on Machine Learning},
articleno = {1187},
numpages = {24},
location = {Vancouver, Canada},
series = {ICML'25}
}

@article{han2024parameter,
  title={Parameter-efficient fine-tuning for large models: A comprehensive survey},
  author={Han, Zeyu and Gao, Chao and Liu, Jinyang and Zhang, Jeff and Zhang, Sai Qian},
  journal={arXiv preprint arXiv:2403.14608},
  year={2024}
}

@inproceedings{houlsby2019parameter,
  title={Parameter-efficient transfer learning for NLP},
  author={Houlsby, Neil and Giurgiu, Andrei and Jastrzebski, Stanislaw and Morrone, Bruna and De Laroussilhe, Quentin and Gesmundo, Andrea and Attariyan, Mona and Gelly, Sylvain},
  booktitle={International conference on machine learning},
  pages={2790--2799},
  year={2019},
  organization={PMLR}
}

@article{dayi2024gradient,
  title={Gradient dynamics for low-rank fine-tuning beyond kernels},
  author={Dayi, Arif Kerem and Chen, Sitan},
  journal={arXiv preprint arXiv:2411.15385},
  year={2024}
}

@inproceedings{kopiczko2023vera,
  title={{VeRA: Vector-based Random Matrix Adaptation}},
  author={Kopiczko, Dawid Jan and Blankevoort, Tijmen and Asano, Yuki M},
  booktitle={The Twelfth International Conference on Learning Representations},
  year={2024}
}

@inproceedings{li2024crucial,
  title={On the Crucial Role of Initialization for Matrix Factorization},
  author={Li, Bingcong and Zhang, Liang and Mokhtari, Aryan and He, Niao},
  booktitle={The Thirteenth International Conference on Learning Representations},
  year={2025}
}

@inproceedings{liu2025optimization,
  title={On the optimization landscape of low rank adaptation methods for large language models},
  author={Liu, Xu-Hui and Du, Yali and Wang, Jun and Yu, Yang},
  booktitle={The Thirteenth International Conference on Learning Representations},
  year={2025}
}

@inproceedings{zeng2023expressive,
  title={The Expressive Power of Low-Rank Adaptation},
  author={Zeng, Yuchen and Lee, Kangwook},
  booktitle={The Twelfth International Conference on Learning Representations},
  year={2024}
}

@article{gawang2024lora,
  title={{LoRA-GA: Low-rank adaptation with gradient approximation}},
  author={Wang, Shaowen and Yu, Linxi and Li, Jian},
  journal={Advances in Neural Information Processing Systems},
  volume={37},
  pages={54905--54931},
  year={2024}
}

@article{hu2022lora,
  title={{LoRA: Low-rank adaptation of large language models.}},
  author={Hu, Edward J and Shen, Yelong and Wallis, Phillip and Allen-Zhu, Zeyuan and Li, Yuanzhi and Wang, Shean and Wang, Liang and Chen, Weizhu and others},
  journal={ICLR},
  volume={1},
  number={2},
  pages={3},
  year={2022}
}

@article{george2018fast,
  title={Fast approximate natural gradient descent in a kronecker factored eigenbasis},
  author={George, Thomas and Laurent, C{\'e}sar and Bouthillier, Xavier and Ballas, Nicolas and Vincent, Pascal},
  journal={Advances in neural information processing systems},
  volume={31},
  year={2018}
}

@article{saha2026grit,
  title={{GRIT--Geometry-Aware PEFT with K-FACPreconditioning, Fisher-Guided Reprojection, andDynamic Rank Adaptation}},
  author={Saha, Pritish and Rajbangshi, Chandrav and Goyal, Rudra and Goyal, Mohit and Deo, Anurag and Roy, Biswajit and Singh, Ningthoujam Dhanachandra and Goswami, Raxit and Das, Amitava},
  journal={arXiv preprint arXiv:2601.00231},
  year={2026}
}

@article{amari1998natural,
  title={Natural gradient works efficiently in learning},
  author={Amari, Shun-Ichi},
  journal={Neural computation},
  volume={10},
  number={2},
  pages={251--276},
  year={1998},
  publisher={MIT Press}
}

@article{yang2023bayesian,
  title={Bayesian low-rank adaptation for large language models},
  author={Yang, Adam X and Robeyns, Maxime and Wang, Xi and Aitchison, Laurence},
  journal={arXiv preprint arXiv:2308.13111},
  year={2023}
}

@article{zhang2025lorada,
  title={{LoRA-DA: Data-Aware Initialization for Low-Rank Adaptation via Asymptotic Analysis}},
  author={Zhang, Qingyue and Chu, Chang and Peng, Tianren and Li, Qi and Luo, Xiangyang and Jiang, Zhihao and Huang, Shao-Lun},
  journal={arXiv preprint arXiv:2510.24561},
  year={2025}
}

@inproceedings{zhang2025lora,
  title={{LoRA-One: one-step full gradient could suffice for fine-tuning large language models, provably and efficiently}},
  author={Zhang, Yuanhe and Liu, Fanghui and Chen, Yudong},
  booktitle={Proceedings of the 42nd International Conference on Machine Learning},
  volume={267},
  year={2025},
  organization={PMLR}
}

@inproceedings{afzallinearization,
  title={Linearization Explains Fine-Tuning in Large Language Models},
  author={Afzal, Zahra Rahimi and Esmaeilbeig, Tara and Soltanalian, Mojtaba and Ohannessian, Mesrob I},
  booktitle={The Thirty-ninth Annual Conference on Neural Information Processing Systems},
  year={2025}
}

@article{oymak2019generalization,
  title={Generalization guarantees for neural networks via harnessing the low-rank structure of the jacobian},
  author={Oymak, Samet and Fabian, Zalan and Li, Mingchen and Soltanolkotabi, Mahdi},
  journal={arXiv preprint arXiv:1906.05392},
  year={2019}
}

@inproceedings{hayou2024lora+,
  title={{LoRA+: Efficient Low Rank Adaptation of Large Models}},
  author={Hayou, Soufiane and Ghosh, Nikhil and Yu, Bin},
  booktitle={International Conference on Machine Learning},
  pages={17783--17806},
  year={2024},
  organization={PMLR}
}

@article{meng2024pissa,
  title={{PiSSA: Principal singular values and singular vectors adaptation of large language models}},
  author={Meng, Fanxu and Wang, Zhaohui and Zhang, Muhan},
  journal={Advances in Neural Information Processing Systems},
  volume={37},
  pages={121038--121072},
  year={2024}
}

@inproceedings{malladi2023kernel,
  title={A kernel-based view of language model fine-tuning},
  author={Malladi, Sadhika and Wettig, Alexander and Yu, Dingli and Chen, Danqi and Arora, Sanjeev},
  booktitle={International Conference on Machine Learning},
  pages={23610--23641},
  year={2023},
  organization={PMLR}
}

@inproceedings{murray2022characterizing,
  title={{Characterizing the Spectrum of the NTK via a Power Series Expansion}},
  author={Murray, Michael and Jin, Hui and Bowman, Benjamin and Montufar, Guido},
  booktitle={International Conference on Learning Representations},
  year={2023}
}

@article{chen2025lora,
  title={{CE-LoRA: Computation-Efficient LoRA Fine-Tuning for Language Models}},
  author={Chen, Guanduo and He, Yutong and Hu, Yipeng and Yuan, Kun and Yuan, Binhang},
  journal={arXiv preprint arXiv:2502.01378},
  year={2025}
}

@inproceedings{aghajanyan2021intrinsic,
  title={Intrinsic dimensionality explains the effectiveness of language model fine-tuning},
  author={Aghajanyan, Armen and Gupta, Sonal and Zettlemoyer, Luke},
  booktitle={Proceedings of the 59th annual meeting of the association for computational linguistics and the 11th international joint conference on natural language processing (volume 1: long papers)},
  pages={7319--7328},
  year={2021}
}

@article{gur2018gradient,
  title={Gradient descent happens in a tiny subspace},
  author={Gur-Ari, Guy and Roberts, Daniel A and Dyer, Ethan},
  journal={arXiv preprint arXiv:1812.04754},
  year={2018}
}

@article{zhang2024spectral,
  title={Spectral Adapter: Fine-Tuning in Spectral Space},
  author={Zhang, Fangzhao and Pilanci, Mert},
  journal={Advances in Neural Information Processing Systems},
  volume={37},
  pages={130819--130847},
  year={2024}
}

@inproceedings{karakida2019universal,
  title={{Universal statistics of Fisher information in deep neural networks: Mean field approach}},
  author={Karakida, Ryo and Akaho, Shotaro and Amari, Shun-ichi},
  booktitle={The 22nd International Conference on Artificial Intelligence and Statistics},
  pages={1032--1041},
  year={2019},
  organization={PMLR}
}

@article{jacot2018neural,
  title={{Neural tangent kernel: Convergence and generalization in neural networks}},
  author={Jacot, Arthur and Gabriel, Franck and Hongler, Cl{\'e}ment},
  journal={Advances in neural information processing systems},
  volume={31},
  year={2018}
}

@inproceedings{martens2015optimizing,
  title={Optimizing neural networks with kronecker-factored approximate curvature},
  author={Martens, James and Grosse, Roger},
  booktitle={International conference on machine learning},
  pages={2408--2417},
  year={2015},
  organization={PMLR}
}

@article{lee2019wide,
  title={Wide neural networks of any depth evolve as linear models under gradient descent},
  author={Lee, Jaehoon and Xiao, Lechao and Schoenholz, Samuel and Bahri, Yasaman and Novak, Roman and Sohl-Dickstein, Jascha and Pennington, Jeffrey},
  journal={Advances in neural information processing systems},
  volume={32},
  year={2019}
}

@inproceedings{wang2018glue,
  title={{GLUE: A multi-task benchmark and analysis platform for natural language understanding}},
  author={Wang, Alex and Singh, Amanpreet and Michael, Julian and Hill, Felix and Levy, Omer and Bowman, Samuel},
  booktitle={Proceedings of the 2018 EMNLP workshop BlackboxNLP: Analyzing and interpreting neural networks for NLP},
  pages={353--355},
  year={2018}
}

@article{raffel2020exploring,
  title={Exploring the limits of transfer learning with a unified text-to-text transformer},
  author={Raffel, Colin and Shazeer, Noam and Roberts, Adam and Lee, Katherine and Narang, Sharan and Matena, Michael and Zhou, Yanqi and Li, Wei and Liu, Peter J},
  journal={Journal of machine learning research},
  volume={21},
  number={140},
  pages={1--67},
  year={2020}
}

@article{touvron2023llama,
  title={{Llama 2: Open foundation and fine-tuned chat models}},
  author={Touvron, Hugo and Martin, Louis and Stone, Kevin and Albert, Peter and Almahairi, Amjad and Babaei, Yasmine and Bashlykov, Nikolay and Batra, Soumya and Bhargava, Prajjwal and Bhosale, Shruti and others},
  journal={arXiv preprint arXiv:2307.09288},
  year={2023}
}

@article{yu2023metamath,
  title={{Metamath: Bootstrap your own mathematical questions for large language models}},
  author={Yu, Longhui and Jiang, Weisen and Shi, Han and Yu, Jincheng and Liu, Zhengying and Zhang, Yu and Kwok, James T and Li, Zhenguo and Weller, Adrian and Liu, Weiyang},
  journal={arXiv preprint arXiv:2309.12284},
  year={2023}
}

@article{cobbe2021training,
  title={Training verifiers to solve math word problems},
  author={Cobbe, Karl and Kosaraju, Vineet and Bavarian, Mohammad and Chen, Mark and Jun, Heewoo and Kaiser, Lukasz and Plappert, Matthias and Tworek, Jerry and Hilton, Jacob and Nakano, Reiichiro and others},
  journal={arXiv preprint arXiv:2110.14168},
  year={2021}
}

@article{hendrycks2020measuring,
  title={Measuring massive multitask language understanding},
  author={Hendrycks, Dan and Burns, Collin and Basart, Steven and Zou, Andy and Mazeika, Mantas and Song, Dawn and Steinhardt, Jacob},
  journal={arXiv preprint arXiv:2009.03300},
  year={2020}
}

@inproceedings{zheng2024opencodeinterpreter,
  title={{Opencodeinterpreter: Integrating code generation with execution and refinement}},
  author={Zheng, Tianyu and Zhang, Ge and Shen, Tianhao and Liu, Xueling and Lin, Bill Yuchen and Fu, Jie and Chen, Wenhu and Yue, Xiang},
  booktitle={Findings of the Association for Computational Linguistics: ACL 2024},
  pages={12834--12859},
  year={2024}
}

@article{chen2021evaluating,
  title={Evaluating large language models trained on code},
  author={Chen, Mark and Tworek, Jerry and Jun, Heewoo and Yuan, Qiming and Pinto, Henrique Ponde De Oliveira and Kaplan, Jared and Edwards, Harri and Burda, Yuri and Joseph, Nicholas and Brockman, Greg and others},
  journal={arXiv preprint arXiv:2107.03374},
  year={2021}
}

@article{arora2019exact,
  title={On exact computation with an infinitely wide neural net},
  author={Arora, Sanjeev and Du, Simon S and Hu, Wei and Li, Zhiyuan and Salakhutdinov, Russ R and Wang, Ruosong},
  journal={Advances in neural information processing systems},
  volume={32},
  year={2019}
}

@article{kalajdzievski2023rank,
  title={A rank stabilization scaling factor for fine-tuning with {LoRA}},
  author={Kalajdzievski, Damjan},
  journal={arXiv preprint arXiv:2312.03732},
  year={2023}
}

@article{biderman2024lora,
  title={{LoRA Learns Less and Forgets Less}},
  author={Biderman, Dan and Portes, Jacob and Ortiz, Jose Javier Gonzalez and Paul, Mansheej and Greengard, Philip and Jennings, Connor and King, Daniel and Havens, Sam and Chiley, Vitaliy and Frankle, Jonathan and others},
  journal={Transactions on Machine Learning Research},
  year={2024}
}

@inproceedings{he2015delving,
  title={Delving deep into rectifiers: Surpassing human-level performance on imagenet classification},
  author={He, Kaiming and Zhang, Xiangyu and Ren, Shaoqing and Sun, Jian},
  booktitle={Proceedings of the IEEE international conference on computer vision},
  pages={1026--1034},
  year={2015}
}

@inproceedings{xu2023wizardlm,
  title={{WizardLM: Empowering Large Pre-Trained Language Models to Follow Complex Instructions}},
  author={Xu, Can and Sun, Qingfeng and Zheng, Kai and Geng, Xiubo and Zhao, Pu and Feng, Jiazhan and Tao, Chongyang and Lin, Qingwei and Jiang, Daxin},
  booktitle={The Twelfth International Conference on Learning Representations},
  year={2024}
}

@article{liu2019roberta,
  title={{RoBERTa: A Robustly Optimized Bert Pretraining Approach}},
  author={Liu, Yinhan and Ott, Myle and Goyal, Naman and Du, Jingfei and Joshi, Mandar and Chen, Danqi and Levy, Omer and Lewis, Mike and Zettlemoyer, Luke and Stoyanov, Veselin},
  journal={arXiv preprint arXiv:1907.11692},
  year={2019}
}

@inproceedings{zhang2024riemannian,
  title={{Riemannian preconditioned LoRA for fine-tuning foundation models}},
  author={Zhang, Fangzhao and Pilanci, Mert},
  booktitle={Proceedings of the 41st International Conference on Machine Learning},
  pages={59641--59669},
  year={2024}
}

@inproceedings{arora2019fine,
  title={Fine-grained analysis of optimization and generalization for overparameterized two-layer neural networks},
  author={Arora, Sanjeev and Du, Simon and Hu, Wei and Li, Zhiyuan and Wang, Ruosong},
  booktitle={International conference on machine learning},
  pages={322--332},
  year={2019},
  organization={PMLR}
}

@inproceedings{puiu2022randomized,
  title={Randomized k-facs: Speeding up k-fac with randomized numerical linear algebra},
  author={Puiu, Constantin Octavian},
  booktitle={International Conference on Intelligent Data Engineering and Automated Learning},
  pages={411--422},
  year={2022},
  organization={Springer}
}

@inproceedings{wang2024lora,
  title={{LoRA-Pro: Are Low-Rank Adapters Properly Optimized?}},
  author={Wang, Zhengbo and Liang, Jian and He, Ran and Wang, Zilei and Tan, Tieniu},
  booktitle={The Thirteenth International Conference on Learning Representations},
  year={2025}
}

@inproceedings{li2025beyond,
  title={{Beyond Zero Initialization: Investigating the Impact of Non-Zero Initialization on LoRA Fine-Tuning Dynamics}},
  author={Li, Shiwei and Luo, Xiandi and Tang, Xing and Wang, Haozhao and Chen, Hao and Li, Yuhua and Li, Ruixuan and others},
  booktitle={Forty-second International Conference on Machine Learning},
  year={2025}
}

@inproceedings{zhang2023adalora,
  title={Adaptive Budget Allocation for Parameter-Efficient Fine-Tuning},
  author={Zhang, Qingru and Chen, Minshuo and Bukharin, Alexander and He, Pengcheng and Cheng, Yu and Chen, Weizhu and Zhao, Tuo},
  booktitle={The Eleventh International Conference on Learning Representations},
  year={2023}
}

@article{eschenhagen2023kronecker,
  title={Kronecker-factored approximate curvature for modern neural network architectures},
  author={Eschenhagen, Runa and Immer, Alexander and Turner, Richard and Schneider, Frank and Hennig, Philipp},
  journal={Advances in Neural Information Processing Systems},
  volume={36},
  pages={33624--33655},
  year={2023}
}

@article{eckart1936approximation,
  title={The approximation of one matrix by another of lower rank},
  author={Eckart, Carl and Young, Gale},
  journal={Psychometrika},
  volume={1},
  number={3},
  pages={211--218},
  year={1936},
  publisher={Springer-Verlag}
}

@article{mirsky1960symmetric,
  title={Symmetric gauge functions and unitarily invariant norms},
  author={Mirsky, Leon},
  journal={The quarterly journal of mathematics},
  volume={11},
  number={1},
  pages={50--59},
  year={1960},
  publisher={Oxford University Press}
}

@inproceedings{liu2024dora,
  title={Dora: Weight-decomposed low-rank adaptation},
  author={Liu, Shih-Yang and Wang, Chien-Yi and Yin, Hongxu and Molchanov, Pavlo and Wang, Yu-Chiang Frank and Cheng, Kwang-Ting and Chen, Min-Hung},
  booktitle={Forty-first International Conference on Machine Learning},
  year={2024}
}

@inproceedings{nguyen2021tight,
  title={{Tight bounds on the smallest eigenvalue of the neural tangent kernel for deep ReLU networks}},
  author={Nguyen, Quynh and Mondelli, Marco and Montufar, Guido F},
  booktitle={International Conference on Machine Learning},
  pages={8119--8129},
  year={2021},
  organization={PMLR}
}

@article{allen2019learning,
  title={Learning and generalization in overparameterized neural networks, going beyond two layers},
  author={Allen-Zhu, Zeyuan and Li, Yuanzhi and Liang, Yingyu},
  journal={Advances in neural information processing systems},
  volume={32},
  year={2019}
}

@article{bombari2022memorization,
  title={Memorization and optimization in deep neural networks with minimum over-parameterization},
  author={Bombari, Simone and Amani, Mohammad Hossein and Mondelli, Marco},
  journal={Advances in Neural Information Processing Systems},
  volume={35},
  pages={7628--7640},
  year={2022}
}


\appendix


\newpage 
\section{Additional notations}

For the sake of clarity, we will adopt the following simpler notations in the appendix:
$$J:=J_{f,W_0,X}, \ J_r := J_{f,W_0,A_0,B_0,X}, \ H:=(I_{d_\mathrm{in}}\otimes B_0 \ A_0^\top\otimes I_{d_\mathrm{out}}), \ \nabla = \nabla_W\mathcal{L}_n( f(W;X)|_{W_0}).$$
\section{Proofs for the squared loss}
\label{app:proofs}
\subsection{Proofs of Proposition \ref{prop:ntk_preds} and Corollary \ref{corr:diff_proj}}

\begin{proof}
    (i) consider the full fine-tuning case. Let $\theta = \mathrm{vec}(\Delta W)$. The minimization problem can be rewritten as
    $$\min_{\theta}L_n(\theta)=\min_{\theta}\frac{1}{2}\Vert f(W_0;X)+J\theta-Y\Vert_2^2.$$ 
    This is a standard least square problem whose solutions verify the normal equation:
    $$J^\top (f(W_0;X)-Y+ J\theta^\star)=0 .$$
    The minimum norm solution is given by the Moore-Penrose inverse    
    $$\theta^\star = (J^\top J)^\dagger J^\top Z=J^\dagger Z .$$
    The NTK prediction is then given by 
    $$\hat f(X) = f(W_0;X)+ J\theta^\star = f(W_0;X)+ JJ^\dagger Z = f(W_0;X)+\Pi_J Z .$$
    (ii) For LoRA, we simply remark that for  $\theta = \begin{pmatrix}
        \mathrm{vec}(\Delta A)\\ \mathrm{vec}(\Delta B)
    \end{pmatrix}$, one has 
    $$\mathrm{vec}(\Delta BA_0+B_0\Delta A)= H\theta$$
    by usual kronecker properties. Hence the minimization problem can be rewritten as
    $$\min_{\theta}L_n(\theta)=\frac{1}{2}\Vert f(W_0;X) +\eta JH\theta-Y\Vert_2^2.$$ 
    The minimum norm solution is given by 
    $$\theta^\star=\frac{1}{\eta}(JH)^\dagger Z.$$
    The NTK prediction is then given by 
    $$\hat f_{A_0,B_0}(X)= f(W_0;X)+ \eta JH\theta^\star = f(W_0;X)+ \Pi_{JH} Z   .$$   
    This also straightforwardly implies the result of Corollary \ref{corr:diff_proj}.
    
\end{proof}

\subsection{Proof of Theorem \ref{thm:opt_align}}

We first establish two lemmas that highlight the action of $H$ over the right singular subspace of $J$.

\begin{lemma}
Consider the compact SVD $J= UD V^\top$. Then 
\begin{equation}
    JJ^\dagger  - J_rJ_r^\dagger = U(I-PP^\dagger)U^\top     
    \end{equation}   
where $P=D V^\top H$.
\label{lem: diff_kernels}
\end{lemma}

\begin{proof}
By definition $J_r=JH=UD V^\top H = UP$ and by property of the Moore Penrose inverse, we have $J_r^\dagger = P^\dagger U^\top$ since $U$ is semi-orthogonal. Therefore
$$JJ^\dagger -J_rJ_r^\dagger = UU^\top - UPP^\dagger U^\top = U(I-PP^\dagger)U^\top.$$    
\end{proof}

\begin{lemma}
Assume that Assumption \ref{assump:kfac} holds. Let $r_S=\mathrm{rank}(S), r_T=\mathrm{rank}(T)$ and $q=r_Sr_T$. Consider the thin SVDs 
    $$S=U_S D_S U_S^\top, \quad T = U_TD_TU_T^\top$$ where $U_S\in\mathbb{R}^{d_\mathrm{in},r_S}$ and $U_T\in\mathbb{R}^{d_\mathrm{out},r_T}$. Then there exists semi-orthogonal $U\in\mathbb{R}^{n\times q}$ such that $J$ admits the following thin SVD
     $$J=U(D_S^{1/2}\otimes D_T^{1/2}) (U_S\otimes U_T)^\top$$
     and    \begin{equation}
JJ^\dagger  - J_rJ_r^\dagger = U\left[(I-\Pi_{D_S^{1/2}U_S^\top A_0^\top})\otimes (I-\Pi_{D_T^{1/2}U_T^\top B_0})\right]U^\top.
\label{eq:k_vs_kr}
\end{equation}
\label{lemma:k_vs_kr}
\end{lemma}

\begin{proof}
    The first point is immediate from the \textit{K-FAC} assumption. We focus on the second point.  
    $I-PP^\dagger$ is equal to the orthogonal projection over the null space of $P$ which we will characterize.

    First observe that by mixed product property, 
    $$P = (D_S^{1/2} U_S^\top \otimes D_T^{1/2} U_T^\top B_0 \quad  D_S^{1/2}U_S^\top A_0^\top \otimes D_T^{1/2}U_T^\top).$$
    Therefore its column space is the sum of the two columns spaces.
  \begin{align*}
        \operatorname{col}(P) &= \operatorname{col}(D_S^{1/2}U_S^\top \otimes D_T^{1/2}U_T^\top B_0) + \operatorname{col}(D_S^{1/2} U_S^\top A_0^\top \otimes D_T^{1/2}U_T^\top) \\
        &= \operatorname{col}(D_S^{1/2} U_S^\top) \otimes \operatorname{col}(D_T^{1/2}U_T^\top B_0) + \operatorname{col}(D_S^{1/2} U_S^\top A_0^\top) \otimes \operatorname{col}(D_T^{1/2}U_T^\top)
        \\&= \mathbb{R}^{r_S}\otimes \operatorname{col}(D_T^{1/2} U_T^\top B_0) + \operatorname{col}(D_S^{1/2} U_S^\top A_0^\top) \otimes \mathbb{R}^{r_T}
    \end{align*}
    where $\mathcal{F}\otimes \mathcal{G}$ denotes the tensor product between vector space $\mathcal{F}$ and $\mathcal{G}$. Using the following  well-known linear algebra facts
    $$(\mathcal{F}+\mathcal{G})^\perp = \mathcal{F}^\perp \cap \mathcal{G}^\perp, \qquad (\mathbb{R}^k\otimes \mathcal{G})^\perp = \mathbb{R}^k \otimes \mathcal{G}^\perp, \qquad ( \mathcal{F}\otimes \mathbb{R}^k)^\perp =  \mathcal{F}^\perp \otimes \mathbb{R}^k$$
    then we immediately obtain that 
\begin{align*}
    \operatorname{col}(P)^\perp &= (\mathbb{R}^{r_S}\otimes \operatorname{col}(D_T^{1/2}U_T^\top B_0)^\perp)\cap (\operatorname{col}(D_S^{1/2}U_S^\top A_0^\top)^\perp \otimes \mathbb{R}^{r_T}) \\
    &= \operatorname{col}(D_S^{1/2} U_S^\top A_0^\top)^\perp\otimes \operatorname{col}(D_T^{1/2} U_T^\top B_0)^\perp. 
\end{align*}
    This implies the result namely 
    \begin{align*}
        I-PP^\dagger = \Pi_{\operatorname{col}(P)^\perp} &= \Pi_{\operatorname{col}(D_S^{1/2} U_S^\top A_0^\top)^\perp\otimes \operatorname{col}(D_T^{1/2} U_T^\top B_0)^\perp} \\
        &= \Pi_{\operatorname{col}(D_S^{1/2} U_S^\top A_0^\top)^\perp}\otimes \Pi_{\operatorname{col}(D_T^{1/2} U_T^\top B_0)^\perp}.
    \end{align*}
\end{proof}

\begin{proof}{(of Theorem \ref{thm:opt_align})}
    Our proof uses the following standard result   
    $$\operatorname{vec}(ABC)=(C^\top \otimes A)\operatorname{vec}(B), \quad \Vert \operatorname{vec}(A)\Vert_2=\Vert A\Vert_F.$$   
    This immediately implies from Lemma \ref{lemma:k_vs_kr} the first stated point, namely
\begin{align*}
    &\Vert (\Pi_{J} - \Pi_{J_r})Z \Vert_2 \\&= \Vert\left[(I_{r_S}-\Pi_{D_S^{1/2}U_S^\top A_0^\top})\otimes (I_{r_T}-\Pi_{D_T^{1/2}U_T^\top B_0})\right]U^\top Z\Vert_2 \\
    &=\Vert\left[(I_{r_S}-\Pi_{D_S^{1/2}U_S^\top A_0^\top})\otimes (I_{r_T}-\Pi_{D_T^{1/2}U_T^\top B_0})\right]\left[D_S^{-1/2}U_S^\top\otimes D_T^{-1/2}U_T^\top\right]\operatorname{vec}(\nabla)\Vert_2\\
    &= \Vert\left[(I_{r_S}-\Pi_{D_S^{1/2}U_S^\top A_0^\top})D_S^{-1/2}U_S^\top\otimes (I_{r_T}-\Pi_{D_T^{1/2}U_T^\top B_0})D_T^{-1/2}U_T^\top\right]\operatorname{vec}(\nabla)\Vert_2 \\
    &= \Vert (I_{r_T}-\Pi_{D_T^{1/2}U_T^\top B_0})F(I_{r_S}-\Pi_{D_S^{1/2}U_S^\top A_0^\top})\Vert_F \\
    &= \Vert F - \underbrace{F\Pi_{D_S^{1/2}U_S^\top A_0^\top}- \Pi_{D_T^{1/2}U_T^\top B_0}F(I_{r_S} - \Pi_{D_S^{1/2}U_S^\top A_0^\top})}_{\operatorname{rank}\leq 2r}\Vert_F \\
    &\geq \min_{\operatorname{rank}(E)\leq 2r} \Vert F-E\Vert_F \\&\geq (\sum_{i\geq 2r+1} s_i^2(F))^{1/2} \quad \text{by Eckart-Young-Mirsky theorem \cite{eckart1936approximation, mirsky1960symmetric}.}
    \end{align*}
    
For the second point, we prove a more general result.

\begin{theorem}\label{thm:full_opt_align}
(i) Consider $r\leq \operatorname{rank}(F)$. Let index sets $\mathcal I,\mathcal J$ be such that $|\mathcal I|=|\mathcal J|=r$ and define $U_{F,\mathcal{J}}$ and $V_{F,\mathcal{I}}$ the left and right subsets of singular vectors of $F$ indexed by $\mathcal{J}$ and $\mathcal{I}$, respectively. Let invertible matrices $Q_A,Q_B\in\mathbb{GL}_r$.  
Let
     $$\begin{cases}
        A_0^\top = U_SD_S^{-1/2}V_{F,\mathcal I} Q_A,\\
        B_0 = U_TD_T^{-1/2}U_{F,\mathcal J} Q_B. 
    \end{cases}$$
Then: 
$$\Vert (\Pi_{J} - \Pi_{J_r})Z \Vert_2=(\sum_{i\in(\mathcal I\cup \mathcal J)^c} s_{i}^2(F))^{1/2}.$$ 
Furthermore, choosing $\mathcal{I}\cup\mathcal{J}=[\min(\operatorname{rank}(F),2r)]$ solves the alignment problem (\ref{eq:pred_obj}).

(ii) Consider $r>\operatorname{rank}(F)$.  Let invertible matrices $Q_A\in\mathbb{GL}_{\operatorname{rank}(F)},Q_B\in\mathbb{GL}_{\operatorname{rank}(F)}$. Let 
     $$\begin{cases}
    A_0^\top =(U_SD_S^{-1/2}V_{F} Q_A, \ W_1),\\ 
    B_0 = (U_TD_T^{-1/2}U_{F} Q_B, \ W_2).
    \end{cases}
    $$ 
    for any matrices $W_1\in\mathbb{R}^{d_\mathrm{in}\times (r-\operatorname{rank}(F)) },W_2\in\mathbb{R}^{d_\mathrm{out}\times (r-\operatorname{rank}(F)) }$.
    Then:   
     $$\Vert (\Pi_{J} - \Pi_{J_r})Z \Vert_2=0.$$
     which solves the alignment problem (\ref{eq:pred_obj}).
\end{theorem}
\end{proof}
\begin{proof}{(of Theorem \ref{thm:full_opt_align})}
    We have shown earlier that
$$
    \Vert (\Pi_{J} - \Pi_{J_r})Z \Vert_2=\Vert (I_{r_T}-\Pi_{D_T^{1/2}U_T^\top B_0})F(I_{r_S}-\Pi_{D_S^{1/2}U_S^\top A_0^\top})\Vert_F $$
For (i), plugging in the values of $A_0,B_0$ directly gives  $\Pi_{D_T^{1/2}U_T^\top B_0}=\Pi_{U_{F,\mathcal J}}$ and $\Pi_{D_S^{1/2}U_S^\top A_0^\top}=\Pi_{V_{F,\mathcal I}}$ which implies the expression  
$$\Vert (\Pi_{J} - \Pi_{J_r})Z \Vert_2=(\sum_{i\in(\mathcal I\cup \mathcal J)^c} s_{i}^2(F))^{1/2}.$$ 
Let $I\cup J=[\min(\operatorname{rank}(F),2r)]$. If $\operatorname{rank}(F)\geq 2r$ then $(\sum_{i\in(\mathcal I\cup \mathcal J)^c} s_{i}^2(F))^{1/2}=(\sum_{i\geq 2r+1} s_{i}^2(F))^{1/2}$ which achieves the lower bound. In the other case $\operatorname{rank}(F)< 2r$, we have $(\sum_{i\in(\mathcal I\cup \mathcal J)^c} s_{i}^2(F))^{1/2} =(\sum_{i\geq\operatorname{rank}(F)+1}s_{i}^2(F))^{1/2}=0$. In both cases, we solve the alignment problem (\ref{eq:pred_obj}).

For (ii), observe that the column space of $D_T^{1/2}U_T^\top B_0$ contains at least the column space of $U_F$, and similarly the column space of $D_S^{1/2}U_S^\top A_0^\top$ contains at least the column space of $V_F$. Therefore this directly implies 
$$\Vert (\Pi_{J} - \Pi_{J_r})Z \Vert_2=0.$$
This solves the alignment problem (\ref{eq:pred_obj}).
\end{proof}

\subsection{Proof of Proposition \ref{prop:q_star}}

The optimization problem can be rewritten as a standard orthogonal projection since $$\min_{(Q_A,Q_B)\in GL_r}\Vert B_0A_0+\nabla\Vert_F =\min_{M\in \mathcal{Q}} \Vert M-(-\nabla)\Vert_F,$$
where $\mathcal{Q}:=\operatorname{range}\{L_rQR_r, Q\in\mathbb{R}^{r\times r}\}$.

We will show that $M^\star = -\Pi_{L_r}\nabla \Pi_{R_r}$ is the orthogonal projection of $-\nabla$ onto $\mathcal{Q}$ and therefore $Q^\star= -L_r^\dagger \nabla R_r^\dagger$ is indeed a minimizer of our objective. Then,  we will show that $Q^\star$ is unique and invertible.

For the first point, observe that
$$-\nabla - M^\star = -(I-\Pi_{L_r})\nabla - \Pi_{L_r}\nabla (I-\Pi_{R_r}),$$
therefore for any $Q$, we have 
$$\langle -\nabla - M^\star, L_rQR_r \rangle =-\langle (I-\Pi_{L_r})\nabla , L_rQR_r\rangle - \langle \Pi_{L_r}\nabla (I-\Pi_{R_r}), L_rQR_r\rangle=0.$$
Hence $M^\star$ is  the orthogonal projection of $-\nabla$, which implies that $Q^\star$ is a solution of the original problem.

Since $L_r,R_r$ are full column rank, we can also write 
\begin{align*}
Q^\star &= -(L_r^\top L_r)^{-1}L_r^\top \nabla R_r (R_r^\top R_r)^{-1}\\
&= -(U_{F,r}^\top D_T^{-1}U_{F,r})^{-1} U_{F,r}^\top \underbrace{D_T^{-1/2}U_T^\top \nabla U_S D_S^{-1/2}}_{=F} V_{F,r}(V_{F,r}^\top D_S^{-1}V_{F,r})^{-1}\\
&=-(U_{F,r}^\top D_T^{-1}U_{F,r})^{-1} D_{F,r} (V_{F,r}^\top D_S^{-1}V_{F,r})^{-1}.
\end{align*}
The last expression of $Q^\star$ proves its invertibility.

Uniqueness follows from the injectivity of $Q\rightarrow L_rQR_r$ since 
$$L_rQR_r^\top=0\implies L_r^\dagger L_r Q R_r^\top(R_r^\top)^\dagger = 0\implies Q=0.$$


\newpage
\section{Additional theoretical details}

\subsection{The GGN derivation and K-FAC designs}\label{app:kfac}

We already showed for a linear layer $W$ that $\nabla_Wf^{(c)}(W_0;x_i)=\delta_i^{(c)}h_i^\top$ hence \\ $J_{f,W_0,x_i}=(\operatorname{vec}(\delta_i^{(1)}h_i^\top) \ ... \ \operatorname{vec}(\delta_i^{(c)}h_i^\top))^\top = (h_i\otimes \delta_i^{(1)} \ ... \ h_i\otimes \delta_i^{(C)})^\top = (h_i\otimes \delta_i)^\top$. Therefore 
$$J_{f, W_0,X}^\top J_{f, W_0,X}= \sum_{i=1}^n J_{f, W_0,x_i}^\top J_{f, W_0,x_i} = \sum_{i=1}^n  (h_i\otimes \delta_i)(h_i\otimes \delta_i)^\top = \sum_{i=1}^n (h_ih_i^\top)\otimes (\delta_i\delta_i^\top).$$
The K-FAC approximation replaces this sum of Kronecker product into a Kronecker product of two sums by adding missing cross-sample terms $(h_ih_i^\top) \otimes (\delta_j\delta_j^\top)$ for $i\neq j$.

Assume now weight sharing, then
we have $\nabla_W f^{(c)}(W_0;x_i)=\sum_{w=1}^\Omega\delta_{i,w}^{(c)}h_{i,w}^\top$ where now we consider input token $h_{i,w}$  and derivative $\delta_{i,w}$ w.r.t output token $u_{i,w}$. Then  
$$J_{f,W_0,x_i}=(\operatorname{vec}(\sum_{w=1}^\Omega\delta_{i,w}^{(1)}h_{i,w}^\top) \ ... \ \operatorname{vec}(\sum_{w=1}^\Omega\delta_{i,w}^{(C)}h_{i,w}^\top))^\top=\sum_{w=1}^\Omega(h_{i,w}\otimes \delta_{i,w})^\top.$$
Therefore, we now have token-level cross interactions as following
$$J_{f, W_0,X}^\top J_{f, W_0,X}= \sum_{i=1}^n  J_{f, W_0,x_i}^\top J_{f, W_0,x_i} = \sum_{i=1}^n \sum_{w=1}^\Omega \sum_{w'=1}^\Omega  (h_{i,w}h_{i,w'}^\top)\otimes (\delta_{i,w}\delta_{i,w'}^\top).$$
This case is more delicate and recent works \cite{eschenhagen2023kronecker} considered two possibilities: 
\begin{enumerate}
    \item \textit{reduce} settings which corresponds to non-tokenized final output (tokenization may happen before) with corresponding \textit{K-FAC-reduce} factors
$$S = \frac{1}{n\Omega^2}\sum_{i=1}^n (\sum_{\omega=1}^\Omega h_{i,\omega})(\sum_{\omega=1}^\Omega h_{i,\omega})^\top ,\quad  T=\sum_{i=1}^n (\sum_{\omega=1}^\Omega \delta_{i,\omega})(\sum_{\omega=1}^\Omega \delta_{i,\omega})^\top.$$
\item \textit{expand} settings which corresponds to tokenized final output with corresponding \textit{K-FAC-reduce} factors
$$S = \frac{1}{n\Omega}\sum_{i=1}^n \sum_{\omega=1}^\Omega h_{i,\omega} h_{i,\omega}^\top ,\quad  T=\sum_{i=1}^n \sum_{\omega=1}^\Omega \delta_{i,\omega}\delta_{i,\omega}^\top.$$
\end{enumerate}

K-FAC reduce is computationally less intensive, and \cite{eschenhagen2023kronecker} show that it can reach similar performance to K-FAC expand even in the expand settings.

\subsection{Comparison with previous K-FAC/LoRA methods}\label{app:compa_kfac}

It is worth mentioning several related works that also incorporate K-FAC into LoRA, although in ways that differ substantially from our approach.

Yang et al.~\cite{yang2023bayesian} propose a Bayesian analysis of LoRA in a post-finetuning setting. More precisely, they apply a Laplace approximation on top of trained LoRA adapters and use a K-FAC approximation of the Fisher to obtain a tractable posterior over the LoRA parameters.

Saha et al.~\cite{saha2026grit} instead use K-FAC as an online preconditioner for LoRA optimization. Their Kronecker factors are built in the current rank-r LoRA coordinates  $Ah_i$ and $B^\top \delta_i$, so the resulting approximation acts within the existing LoRA subspace rather than being derived from the full GGN geometry of the underlying layer. Accordingly, their method is primarily optimization-driven and empirical, whereas our analysis is centered on a full-layer curvature model together with a function-level prediction-alignment problem that we formulate and solve.

In a more related approach, Zhang et al.~\cite{zhang2025lorada} proposes to extend the work of LoRA-GA/One by considering aligning the natural gradient instead of the standard gradient to obtain efficient initialization $A_0,B_0$. While their method accounts for the pretrained Jacobian as we do, their analysis remains at the parameter level, whereas our guarantees are stated at the function level. 


 \subsection{Comparison with previous initialization methods}\label{app:loraone}

Recently, a growing number of works have started designing ``alignment-aware'' versions of LoRA, where the goal is to minimize a criterion measuring the discrepancy between LoRA and full fine-tuning with respect to some chosen quantity. For example, LoRA-GA~\cite{gawang2024lora} selects \(A_0,B_0\) so as to align the first-step LoRA update with that of full fine-tuning, namely by minimizing \(\|\Delta_1(BA)-\Delta_1(W)\|_F\), where \(\Delta_1(BA):=\eta(B_0\Delta A+\Delta B A_0)\). In a similar spirit, \cite{wang2024lora} modifies the structure of the first-step LoRA gradients in order to better match the first-step full fine-tuning gradient. 

These works provide meaningful notions of alignment, but they remain tied to a specific proxy, namely first-step parameter-update alignment. As a result, they do not directly address what alignment should mean at the level of the final trained model, nor how such alignment affects downstream performance.
These limitations are partially addressed in \cite{zhang2025lora}, which considers a restricted setting where \(y=g(W_0+\Delta^\star;x)\), with \(g\) either a single linear map or a single ReLU function of isotropic inputs \(x\), and with an unknown true low-rank matrix \(\Delta^\star\). 
$$\begin{cases}
    f(W;x)=Wx \qquad \text{with true data} \qquad y=(W_0+\Delta^\star)x \\
    f(W;x)=\mathrm{Relu}(Wx) \qquad \text{with true data} \qquad y=\mathrm{Relu}((W_0+\Delta^\star) x)
\end{cases}$$
for some true unknown low-rank update $\Delta^\star$ and isotropic Gaussian vectors $(x_i)_{i=1}^n\sim \mathcal{N}(0,I_{d_\mathrm{in}})$. For such settings, they show that $A_0,B_0$ based on an SVD of loss gradient  $\nabla_W\mathcal{L}_n(f(W_0;X))$ alone yield small identification error dynamics  \(\|B_tA_t-\Delta^\star\|_F\) (see their Theorem C.17 and Theorem 4.2). This constitutes an important step toward a performance-based notion of alignment. However, this analysis still relies on strong assumptions on the data-generating process.
 
 In the following, we show that the whitened gradient essentially (rigorously computing their difference would require careful bounding in high probability which is outside the scope of this paper) becomes the standard loss gradient, and therefore our initialization procedure  recovers theirs. \begin{lemma}\label{lem:loraone}
 Assume $x_1,...,x_n\sim  \mathcal{N}(0,I_{d_\mathrm{in}})$.
 
 (i) Consider the linear case. Then both the NTK and \textit{K-FAC} approximations are exact and we have 
     $$J_{f,W_0,X}^\top J_{f,W_0,X}=(\sum_{i=1}^n x_ix_i^\top)\otimes I_C.$$
     Hence the \textit{K-FAC} factors are $S= \sum_{i=1}^n x_ix_i^\top =X^\top X$ and $T= I_C$. Furthermore
     $$\mathbb{E}(S) = nI_n$$
     (ii) Consider the single ReLU case. Then 
     $$J_{f,W_0,X}^\top J_{f,W_0,X}=\sum_{i=1}^n (x_ix_i^\top)\otimes \operatorname{diag}(1(W_0x_i>0)).$$
     Hence the \textit{K-FAC} factors are $S= \sum_{i=1}^n x_ix_i^\top=X^\top X$ and $T= \sum_{i=1}^n D_i$ where $D_i = \operatorname{diag}(1(W_0x_i>0))$.
     Furthermore
     $$\mathbb{E}(S) = nI_{d_\mathrm{in}}, \quad \mathbb{E}(T) =\frac{n}{2} I_C.$$
 \end{lemma}

\begin{proof} 
Let us first observe that in the single layer case, we have $d_\mathrm{out}=C, \ h_i=x_i$ for all $i=1,...,n$. By definition of the Gauss-Newton approximation, we have
$$J^\top J=\sum_{i=1}^n (h_ih_i^\top)\otimes (\delta_i\delta_i^\top)=\sum_{i=1}^n (x_ix_i^\top) \otimes (\delta_i\delta_i^\top)$$
(i) In the linear case, by definition $u_i=W_0x_i=f(W_0;x_i)$ hence $\delta_i=\nabla_{u_i}f(W_0;x_i)=I_C$ which implies 
$$J^\top J=\sum_{i=1}^n (x_ix_i^\top)\otimes I_C= (\sum_{i=1}^n x_ix_i^\top)\otimes I_C$$
by distributivity of the Kronecker product. Hence $S=\sum_{i=1}^n x_ix_i^\top=X^\top X$ and $T=I_C$, and since $x_i\sim \mathcal{N}(0, I_{d_\mathrm{in}})$ then $\mathbb{E}(S)=nI_{d_\mathrm{in}}$.

(ii) In the ReLu case, we still have $S = X^TX$ with $\mathbb{E}(S)=nI_{d_\mathrm{in}}$. Also, it is well-known that the derivative is the Heaviside function indicating which neuron (=row) of $W_0$ is activated, i.e $\delta_i=\operatorname{Diag}(1\{W_0x_i>0\})$. Hence , this implies that  $T=\sum_{i=1}^n\delta_i\delta_i^\top = \sum_{i=1}^n D_i$.

Let us rewrite $T=\operatorname{Diag}(m_1,...,m_C)$ where $m_j=\sum_{i=1}^n 1\{(W_0)_j^\top x_i>0\}$ counts the number of samples activating neuron $j$. Since each $x_i$ are iid following a symmetric probability distribution w.r.t the origin, then we have 
$$\forall j=1,...,C, \quad \mathbb{P}((W_0)_j^\top x_i>0)=\frac{1}{2}.$$
This implies 
$$\mathbb{E}(T)=\frac{n}{2}I_C.$$
\end{proof}

\newpage
\section{Algorithmic details}\label{app:algo}

\begin{algorithm}[h!]
\caption{CG-LoRA}
\label{alg:cg-lora}
\DontPrintSemicolon
\SetKwInOut{Input}{Input}
\SetKwInOut{Output}{Output}

\Input{Batch $\{(x_i,y_i)\}_{i=1}^{B_\mathrm{init}}$; rank $r$; number of LoRA layers $p$; pretrained weights $\{W_0^{(k)}\}_{k=1}^p$. Hyperparameters: oversampling factor $m$; number of power iterations $q$; scaling factor $\gamma$.}
\Output{$\{(A_0^{(k)},B_0^{(k)})\}_{k=1}^p$.}

Compute layer inputs $\{h^{(k)}\}_{k=1}^p$ by a forward pass\;
Compute layer outputs $\{\delta^{(k)}\}_{k=1}^p$ by a backward pass\;

\For{$k \gets 1$ \KwTo $p$}{\BlankLine
$(d_\mathrm{in},d_\mathrm{out})\gets\operatorname{shape}(W_0^{(k)})$ \;
\tcp{Pooling if tokenized inputs/outputs}
  \If{$\operatorname{shape}(h^{(k)})=(B_\mathrm{init},\Omega,d_\mathrm{in})$}{
  $h^{(k)}\gets \sum_{\omega=1}^\Omega h^{(k)}[:,\omega,:]$ \quad 
  $\delta^{(k)}\gets \sum_{\omega=1}^\Omega \delta^{(k)}[:,\omega,:]$} 
  $(U_S,D_S,U_T,D_T) \gets
  \operatorname{Subspace}\!\bigl(h^{(k)},\delta^{(k)},m,q\bigr)$ \tcp*[r]{$O(\max(d_\mathrm{in},d_\mathrm{out})m(qn+m))$, see Alg \ref{alg:subspace}}
 $F  \gets   D_T^{-1/2}U_{T}^\top \nabla_W\mathcal{L}_n(f(W;X))|_{W_0}U_SD_S^{-1/2}$\tcp*[r]{$O(n\max(d_\mathrm{in},d_\mathrm{out})m+nm^2)$}
  $U_{F,r}, D_{F,r},V_{F,r}\gets  \operatorname{SVD}_r(F)$ \tcp*[r]{$O(m^2r)$}
  $L\gets U_TD_T^{-1/2}U_{F,r}, \quad R\gets U_SD_S^{-1/2}V_{F,r}$\tcp*[r]{$O(\max(d_\mathrm{in},d_\mathrm{out})mr)$}
  $(A_0^{(k)},B_0^{(k)})\gets \operatorname{BalancedSVD} (L,R, D_{F,r})$
  \tcp*[r]{$O(\max(d_\mathrm{in},d_\mathrm{out})r^2)$, see Alg \ref{alg:bal_svd}}
  $A_0^{(k)}\gets \frac{d_\mathrm{out}^{1/4}}{\gamma}A_0^{(k)} , \quad B_0^{(k)} \gets \frac{d_\mathrm{out}^{1/4}}{\gamma }B_0^{(k)}$ }
\Return $\{(A_0^{(k)},B_0^{(k)})\}_{k=1}^p$\;
\end{algorithm}

\subsection{Randomized SVD for Jacobian singular subspace}

Our first subroutine namely Algorithm~\ref{alg:subspace}, also considered in~\cite{puiu2022randomized}, uses a standard randomized routine combining sketching, QR decomposition, and Rayleigh--Ritz projection, optionally with power iterations, to compute $U_S, D_S$ and $U_T, D_T$. Crucially, these methods avoid forming the full matrices $S$ and $T$. Similarly for $F$, we do not compute the full loss gradient. Instead, for linear layers, it can be written as (with $\Omega=1$ when there is no weight-sharing):
$$\nabla_W \mathcal{L}_n(f(W;X))|_{W_0}=\sum_{i=1}^n\sum_{\omega=1}^\Omega \mu_{i,\omega}h_{i,\omega}^\top,$$
where $\mu_{i,\omega}$ is the pre-activation token derivative of the loss gradient (as opposed to previously defined $\delta_{i,\omega}$ which is associated to the output signal). As a consequence, we have:
$$F=D_T^{-1/2}[\sum_{i=1}^n \sum_{\omega=1}^\Omega(U_T^\top \mu_{i,\omega} )(U_S^\top h_{i,\omega})^\top] D_S^{-1/2}.$$
\SetKwFunction{Subspace}{subspace}

\begin{algorithm}[t]
\caption{\textsc{subspace}}
\label{alg:subspace}
\DontPrintSemicolon
\SetKwInOut{Input}{Input}
\SetKwInOut{Output}{Output}

\Input{Input $h\in\mathbb{R}^{d_{\mathrm{in}}\times n}$,  pre-activation derivative $\delta\in\mathbb{R}^{d_{\mathrm{out}}\times nC}$; oversampling factor $m$; number of power iterations $q$; }
\Output{$(U_S,D_S,U_T,D_T)$.}

\BlankLine
\tcp{$S$-side (with $S = hh^\top$)}
$\Omega_S \in \mathbb{R}^{d_{\mathrm{in}}\times m}$ with i.i.d.\ $\mathcal{N}(0,1)$ entries\;
\For{$t=1,..,1+q$}{
$Y_S \gets h(h^\top\Omega_S)$\tcp*[r]{$O(nd_{\mathrm{in}}m)$}
Compute thin QR: $Y_S = Q_S R_S$, $Q_S\in\mathbb{R}^{d_{\mathrm{in}}\times m}$\tcp*[r]{$O(d_{\mathrm{in}}m^2)$}
$\Omega_S\gets Q_S$}
$(U_S,D_S) \gets \mathrm{SVD}\left((Q_S^\top h)(h^\top Q_S)\right)$\tcp*[r]{$O(nd_{\mathrm{in}}m+ nm^2+m^3)$}
$U_S \gets Q_S U_S$\tcp*[r]{$O(d_{\mathrm{in}}m^2)$}

\BlankLine
\tcp{$T$-side (with $T = \delta\delta^\top$)}
$\Omega_T \in \mathbb{R}^{d_{\mathrm{out}}\times m}$ with i.i.d.\ $\mathcal{N}(0,1)$ entries\;
\For{$t=1,...,1+q$}{
$Y_T \gets \delta(\delta^\top\Omega_T)$\tcp*[r]{$O(nd_\mathrm{out}m)$, see Section \ref{app:hutchinson}}
Compute thin QR: $Y_T = Q_T R_T$, $Q_T\in\mathbb{R}^{d_{\mathrm{out}}\times m}$\tcp*[r]{$O(d_{\mathrm{out}}m^2)$}
$\Omega_T\gets Q_T$}

$(U_T,D_T) \gets \mathrm{SVD}\left((Q_T^\top\delta)(\delta^\top Q_T)\right)$\tcp*[r]{$O(nd_\mathrm{out}m+nm^2+m^3)$, see Section \ref{app:hutchinson}}
$U_T \gets Q_T U_T$\tcp*[r]{$O(d_{\mathrm{out}}m^2)$}

\Return $(U_S,D_S,U_T,D_T)$\;
\end{algorithm}


\subsection{Balanced realizations}
We now describe Algorithm \ref{alg:bal_svd}. Let us define $J_{L_r}=L_r^\top L_r= U_{F,r}^\top D_T^{-1}U_{F,r}$ and $J_{R_r} = R_r^\top R_r=V_{F,r}^\top D_S^{-1}V_{F,r}.$
Let $\operatorname{chol}(J_{L_r})$, $\operatorname{chol}(J_{R_r})$ represent the lower triangular Cholesky factors such that 
$$J_{L_r}=\operatorname{chol}(J_{L_r})\operatorname{chol}(J_{L_r})^\top , \quad J_{R_r}=\operatorname{chol}(J_{R_r})\operatorname{chol}(J_{R_r})^\top.$$
Then
$$M= L_rQ^\star R_r^\top=L_r\operatorname{chol}(J_{L_r})^{-\top}[\operatorname{chol}(J_{L_r})^\top Q^\star \operatorname{chol}(J_{R_r})] \operatorname{chol}(J_{R_r})^{-1}R_r^\top$$          
Let 
\begin{align*}
    K=\operatorname{chol}(J_{L_r})^\top Q^\star \operatorname{chol}(J_{R_r})&=-\operatorname{chol}(J_{L_r})^\top J_{L_r}^{-1}D_{F,r}J_{R_r}^{-1}\operatorname{chol}(J_{R_r}) \\
    &= -\operatorname{chol}(J_{L_r})^{-1}D_{F,r}\operatorname{chol}(J_{R_r})^\top
\end{align*} with SVD $K=U_KD_KV_K^\top$ then 
$$M=[L_r\operatorname{chol}(J_{L_r})^{-\top}  U_K]D_K [R\operatorname{chol}(J_{R_r})^{-\top}V_K]^\top$$
is a valid SVD since 
\begin{align*}
    (L_r\operatorname{chol}(J_{L_r})^{-\top}U_K)^\top L_r\operatorname{chol}(J_{L_r})^{-\top} U_K&= U_K^\top \operatorname{chol}(J_{L_r})^{-1}L_r^\top L_r \operatorname{chol}(J_{L_r})^{-\top} U_K\\
    &=U_K^\top U_K = I,
\end{align*} and similarly for $R\operatorname{chol}(J_{R_r})^{-\top}V_K$. Therefore, this justifies choosing $B_0\propto L_r\operatorname{chol}(J_{L_r})^{-\top}U_K D_K^{1/2}$ and $A_0^\top\propto R\operatorname{chol}(J_{R_r})^{-\top}V_KD_K^{1/2}$.

\begin{algorithm}[h!]
\caption{Balanced SVD}
\label{alg:bal_svd}
\DontPrintSemicolon
\SetKwInOut{Input}{Input}
\SetKwInOut{Output}{Output}

\Input{Left modes $L\in\mathbb{R}^{d_\mathrm{out}\times r}$. Right modes $R\in \mathbb{R}^{d_\mathrm{in}\times r}$. Scaling matrix $D\in\mathbb{R}^{r\times r}$}
\Output{$(A_0,B_0)$}
  $J_L\gets L^\top L, \quad J_R \gets R^\top R$ \tcp*[r]{$O(\max(d_\mathrm{in},d_\mathrm{out})r^2)$}
  $\operatorname{chol}(J_L), \operatorname{chol}(J_R) \gets$ Cholesky of $J_L,J_R$\tcp*[r]{$O(r^3)$}
  $K\gets -\operatorname{chol}(J_L)^{-1} D  \operatorname{chol}(J_R)^{-\top}$\tcp*[r]{$O(r^3)$}
  $U_K,D_K,V_K^\top \gets$ SVD of $K$ \tcp*[r]{$O(r^3)$}
 $$\begin{cases}
     A_0^\top \gets \frac{1}{D_{K}(1,1)^{1/2}}R\operatorname{chol}(J_R)^{-\top}V_KD_K^{1/2}\\  B_0 \gets \frac{1}{ D_{K}(1,1)^{1/2}}L\operatorname{chol}(J_L)^{-\top}U_KD_K^{1/2}
     \end{cases}$$\tcp*[r]{$O(\max(d_\mathrm{in},d_\mathrm{out})r^2)$}
\Return $(A_0,B_0)$\;
\end{algorithm}

\subsection{Computational complexity}\label{app:cost}

In Table \ref{tab:flops}, we count the number of FLOPs required by CG-LoRA applied to one layer. We do not include here the number of additional backward passes needed, which will be further discussed in Subsection \ref{app:nb_passes}. We state our results with the big $O$ notation which hides the  usual small constants  associated to matrix multiplication, thin Householder QR, SVD decomposition. We also report the memory required which essentially consists in storing smaller $(n,d)$, $(n,m)$ or $(m,m)$ matrices.

\begin{table}[H]
    \centering
    \caption{Number of floating point operations of CG-LoRA (per layer). $q,m$ are the number of power iterations and oversampling factors respectively in Algorithm \ref{alg:subspace}. $d_\mathrm{in}, d_\mathrm{out}$ are the dimensions of pretrained parameter $W_0$. $r$ is the LoRA rank. Big $O$ notation hides constants associated to matrix multiplication, thin Householder QR, SVD decomposition.   }
    \vspace{1em}
    \begin{tabular}{ccc }
    \toprule
       Algorithm    &  Computational cost & Memory cost 
       \\
       
      \midrule
       
        Algorithms \ref{alg:cg-lora} and \ref{alg:cg-lora_ce}  & Algos \ref{alg:subspace} + \ref{alg:bal_svd} & Algos \ref{alg:subspace} + \ref{alg:bal_svd} \\        \hline
        \\
        Algorithm \ref{alg:subspace} & $O(\max (d_\mathrm{in},d_\mathrm{out})((q+1)mn+m^2)$ &  $O(\max (d_\mathrm{in},d_\mathrm{out})(m+n) ) $ \\ \hline \\
        Algorithm \ref{alg:bal_svd} & $O(\max (d_\mathrm{in},d_\mathrm{out})r^2)$ &  $O(\max (d_\mathrm{in},d_\mathrm{out})r)$
    \end{tabular}
    \label{tab:flops}
\end{table}
\newpage
\section{Cross-entropy loss}\label{app:cross_ent}

In the cross-entropy loss case, we are not aware of any existing theoretical results supporting a non-zero LoRA initialization. In our proposed framework, we encounter a main difficulty compared to the squared-loss case, as no closed form solutions exist for both \eqref{eq:ft_ntk} and \eqref{eq:lora_ntk}. 
Despite those challenges, we aim at proposing a principled methodology by solving the NTK alignment problem with quadratic approximations of the cross-entropy loss.  We validate numerically our cross-entropy-loss-specific initialization given in Theorem \ref{thm:cg_lora_multi} in Appendix \ref{app:ce_vs_squared}.

\begin{proposition}\label{prop:gen_loss_ntk_preds}
    Let $(x,y)\rightarrow l(f(x),y)$ be a loss function. Let $\theta$ be a vector. Consider the quadratic expansion around pretrained logits as following
    $$\mathcal{L}_n(f(W_0;X)+J\theta)= \mathcal{L}_n(f(W_0;X)) + \nabla_f\mathcal{L}_n(f)|_{f=f(W_0;X)}^\top J\theta + \frac{1}{2}\theta^\top J^\top \nabla^2_f\mathcal{L}_n(f)|_{f=f(W_0;X)} J\theta.$$
    (i) Full Fine-tuning (FFT): The minimum norm solution of \ref{eq:ft_ntk} is given by 
    $$\operatorname{vec}(\Delta W^\star)=-(J^\top  \nabla^2_f\mathcal{L}_n(f)|_{f=f(W_0;X)} J)^\dagger J^\top \nabla_f \mathcal{L}_n(f)|_{f=f(W_0;X)} $$
    (ii) LoRA. Let $J_r = J(I_{d_\mathrm{in}}\otimes B_0 \quad A_0^\top \otimes I_{d_\mathrm{out}})$. The minimum norm solution of
    \ref{eq:lora_ntk} is given by 
    $$(\operatorname{vec}(
   \Delta A^\star)^\top
    \operatorname{vec}(\Delta B^\star)^\top
    )^\top=-(J_r^\top  \nabla^2_f\mathcal{L}_n(f)|_{f=f(W_0;X)} J_r)^\dagger J_r^\top \nabla_f \mathcal{L}_n(f)|_{f=f(W_0;X)}. $$
\end{proposition}
\begin{definition}
    Under this quadratic approximation, the NTK predictors  are defined as follows.
\begin{itemize}
        \item Full finetuning:
        $$\hat f(X):= f(W_0;X)+J(J^\top  \nabla^2_f\mathcal{L}_n(f)|_{f=f(W_0;X)} J)^\dagger J^\top(-\nabla_f \mathcal{L}_n(f)|_{f=f(W_0;X)}),$$
        \item  
        LoRA:   
        $$\hat f_{A_0,B_0}(X):= f(W_0;X)+J_r(J_r^\top \nabla^2_f\mathcal{L}_n(f)|_{f=f(W_0;X)}J_r)^{\dagger}J_r^\top(-\nabla_f \mathcal{L}_n(f)|_{f=f(W_0;X)}).$$
    \end{itemize}
\end{definition}
In the squared loss case, the quadratic approximation is exact by definition and we have $\nabla^2_f\mathcal{L}_n(f)|_{f=f(W_0;X)}=I$ and $\nabla_f \mathcal{L}_n(f)|_{f=f(W_0;X)}=f(W_0;X)-Y$, hence we recover Proposition \ref{prop:ntk_preds}. For the cross-entropy loss case, we consider two cases: 1) binary case, 2) multiclass case which we analyze below.
\subsection{Binary case}

For binary labels, $C=1$ as we only predict the probability of class $1$ (probability of class $0$ can be deduced directly).
\subsubsection{Definitions and Assumptions}
\begin{lemma}\label{lemma:hess_binary}

Let $\sigma:t\in\mathbb{R}\rightarrow \frac{1}{1+e^{-t}}$ the logistic function. Let the pre-trained predicted probabilities $p_i:=\sigma(f(W_0;x_i))$ and $p= (p_1, ...., p_n)^\top\in\mathbb{R}^n$. Define 
$$\mathcal{C}_0:=\operatorname{diag}(p\odot(1-p))\in\mathbb{R}^{n\times n},$$
where $\odot$ is the entry-wise product. Then, $\nabla^2_f\mathcal{L}_n(f)|_{f=f(W_0;X)}=\mathcal{C}_0$ and $\nabla_f \mathcal{L}_n(f)|_{f=f(W_0;X)}=p-Y.$ 
\end{lemma}
Compared to the squared loss, the Gauss-Newton terms are deformed by the loss curvature $(J^\top \mathcal{C}_0J)^{\dagger}J^\top$, which requires an updated \textit{K-FAC} assumption.
\begin{assumption}
    Let $S = \frac{1}{n}\sum_{i=1}^nh_i h_i^\top$, $T=\sum_{i=1}^n \delta_i\delta_i^\top$ and $O= \sum_{i=1}^n  p_i(1-p_i)\delta_i\delta_i^\top$, with prior pooling if the inputs/outputs are tokenized. The following are verified: \begin{itemize}
        \item  
        $J^\top J=S\otimes T$.
        \item  $ 
    J^\top \mathcal{C}_0 J 
    = S\otimes O$
    \item For $i=1,...,n, \quad 0<p_i<1.$ 
    \end{itemize}
    \label{assump:kfac_2}
    
\end{assumption}
The last assumption is required for $\mathcal{C}_0$ to be invertible. For our main result, we will need to introduce the following reweighting  $\Phi$.
\begin{proposition}
    Let Assumption \ref{assump:kfac_2} be true. Let  $\Phi := D_T^{-1/2}U_T^\top O U_T D_T^{-1/2}\in\mathbb{R}^{r_T\times r_T}$. 
    Then $\Phi$ is symmetric positive definite.
    \label{prop:phi_inv}
\end{proposition}
\begin{definition}
    Let $\Phi := D_T^{-1/2}U_T^\top O U_T D_T^{-1/2}$ and the \textit{whitened} binary-cross entropy loss gradient 
    $$F:=\Phi^{-1/2}D_T^{-1/2} U_T^\top \nabla U_SD_S^{-1/2}\in\mathbb{R}^{r_T\times r_S}.$$
    For $r\leq \operatorname{rank}(F)$, we also define its best rank-$r$ approximation 
    $$F_r= U_{F,r}D_{F,r}V_{F,r}^\top.$$
\end{definition}

\subsubsection{The prediction alignment problem}
\begin{proposition}
Let Assumption \ref{assump:kfac_2} be true. For  any $(A_0,B_0)$,  the following holds 
\begin{align*}\Vert \hat{f}(X)-\hat f_{A_0,B_0}(X)\Vert &= \Vert \Phi^{-1/2} (I_{r_T}-\Pi_{\Phi^{1/2}D_T^{1/2}U_T^\top B_0})F(I_{r_S}-\Pi_{D_S^{1/2} U_S^\top A_0^\top})\Vert_F \nonumber \\
        &\geq (\sum_{i\geq 2r+1}s_i^2(\Phi^{-1/2}F))^{1/2}
        \end{align*}
\label{prop:diff_preds_binary}
\end{proposition}
Unlike for the squared loss case, this objective becomes weighted on the left by $\Phi^{-1/2}$.


Let $(A_0,B_0)$ and $P_A=\Pi_{D_S^{1/2}U_S^\top A_0^\top}, \ P_B=\Pi_{\Phi^{1/2}D_T^{1/2}U_T^\top B_0}$. Then we can also write
\begin{align*}
    \Vert \Phi^{-1/2}(I_{r_T}-P_B)F(I_{r_S}-P_A)\Vert_F &=\Vert  \Phi^{-1/2}(I_{r_T}-P_B)\Phi^{1/2} [\Phi^{-1/2}F(I_{r_S}-P_A)] \Vert_F\\
    &\leq  \Vert\Phi^{-1/2}(I_{r_T}-P_B)\Phi^{1/2}\Vert_2 \Vert \Phi^{-1/2}F(I_{r_S}-P_A)\Vert_F .
\end{align*}
This decomposition leads to the following result which gives us candidates that almost achieve the lower bound in Proposition \ref{prop:diff_preds_binary}. For ease of exposition, we will consider the case $r\leq\operatorname{rank}(F)$.

\begin{proposition}\label{prop:upp_bound_preds}
Let $Q_A, Q_B\in\mathbb{GL}_r$. Let $A_0^\top=U_SD_S^{-1/2}V_{\Phi^{-1/2}F,r}Q_A$ and $B_0=U_TD_T^{-1/2}\Phi^{-1/2}EQ_B$ for some full column-rank $E\in\mathbb{R}^{r_T\times r}$ to be determined. 

Let $P_A=\Pi_{D_S^{1/2}U_S^\top A_0^\top}, \ P_B=\Pi_{\Phi^{1/2}D_T^{1/2}U_T^\top B_0}$. Then 
$$\Vert \Phi^{-1/2}(I_{r_T}-P_B)F(I_{r_S}-P_A)\Vert_F\leq \Vert\Phi^{-1/2}(I_{r_T}-\Pi_E)\Phi^{1/2}\Vert_2(\sum_{i\geq r+1}s_i^2(\Phi^{-1/2}F))^{1/2}.$$
\end{proposition}
To test the consistency of our proposal, we perform again the alignment of $B_0A_0$ w.r.t first step loss gradient. 
\begin{proposition}\label{prop:opt_q}
    Let a full column-rank $E\in\mathbb{R}^{r_T\times r}$. The optimization problem
    $$\arg\min_{Q_A,Q_B\in\mathbb{GL}_r}\Vert B_0A_0+\nabla\Vert_F=\arg\min_{Q\in\mathbb{GL}_r}\Vert U_TD_T^{-1/2}\Phi^{-1/2}E QV_{\Phi^{-1/2}F,r}^\top D_S^{-1/2}U_S^\top+\nabla\Vert_F$$
    admits as solution
    \begin{align*}
        Q^\star&= -(U_TD_T^{-1/2}\Phi^{-1/2}E )^\dagger \nabla  (V_{\Phi^{-1/2}F,r}^\top D_S^{-1/2}U_S^\top)^\dagger\\&=-(E^\top\Phi^{-1/2}D_T^{-1}\Phi^{-1/2}E)^{-1}E^\top FV_{\Phi^{-1/2}F,r}(V_{\Phi^{-1/2}F,r}^\top D_S^{-1}V_{\Phi^{-1/2}F,r})^{-1}
    \end{align*}
    and $\operatorname{rank}(B_0A_0)=\operatorname{rank}(Q^\star)=\operatorname{rank}(E^\top FV_{\Phi^{-1/2}F,r})$
\end{proposition}
From the above proposition, the main criterion to choose $E$ is to ensure that $E^\top FV_{\Phi^{-1/2}F,r}$ is full rank. A sufficient choice is $E=\operatorname{orth}(FV_{\Phi^{-1/2}F,r})$ where $\operatorname{orth}(M)$ is the semi-orthogonal matrix obtained by QR decomposition of $M$. We may now state the CG-LoRA initialization for the binary cross-entropy loss.
\begin{theorem}
Let Assumption \ref{assump:kfac_2} be true. Let $E=\operatorname{orth}(FV_{\Phi^{-1/2}F,r})$. Let top $r$ right modes $ R_r=U_SD_S^{-1/2}V_{\Phi^{-1/2}F,r}$ and  left modes $L_r=U_TD_T^{-1/2}\Phi^{-1/2}E$.   Consider the following initialization scheme
\begin{equation}\begin{cases}
        (A_0,B_0) \gets \operatorname{Balanced SVD}(L_r,R_r, E^\top FV_{\Phi^{-1/2}F,r}) \\
        (A_0,B_0) \gets (\frac{d_\mathrm{out}^{1/4}}{\gamma}A_0,\frac{d_\mathrm{out}^{1/4}}{\gamma}B_0 )
    \end{cases}  \tag{CG-LoRA}\end{equation} 
    then 
    $\Vert \hat{f}(X) - \hat{f}_{A_0,B_0}(X) \Vert_2\leq\Vert \Phi^{-1/2}(I-\Pi_E)\Phi^{1/2}\Vert_2(\sum_{i\geq r+1}s_i^2(\Phi^{-1/2}F))^{1/2}.$

    \label{thm:cg_lora_binary_improv}
\end{theorem}

\subsection{Multi-class case}
We consider now the usual multiclass framework where each $y_i\in\mathbb{R}^C$ is a one hot label i.e., all its coordinates are zero except for the true label coordinate which is equal to 1.
\subsubsection{Definitions and Assumptions}

\begin{lemma}\label{lemma:hess_multi}
    Let $p_i=\operatorname{softmax}(f(W_0;x_i))\in\mathbb{R}^C$ and $p= (p_1^\top,..., p_n^\top )^\top \in\mathbb{R}^{nC}$. Define 
    $$\Lambda_i:=\operatorname{diag}(p_i)-p_ip_i^\top\in\mathbb{R}^{C\times C},$$
    and $\Lambda:=\operatorname{blockdiag}(\Lambda_1,...,\Lambda_n)\in\mathbb{R}^{nC\times nC}.$ Then $\nabla_f^2\mathcal{L}_n(f)|_{f=f(W_0;X)}=\Lambda$ and $\nabla_f\mathcal{L}_n(f)|_{f=f(W_0;X)}=p-Y.$
\end{lemma}
 The main challenge arises from the following observation: for any $i=1,...,n$, we have $\Lambda_i 1_C=0$. Therefore, each $\Lambda_i$ is singular regardless of $p_i$, which  reflects the invariance of softmax by constant shifting (adding the same constant to logits does not change probabilities). In particular, this prevents our previous analysis in the binary case to extend cleanly since $\nabla_f^2\mathcal{L}_n(f)|_{f=f(W_0;X)}$ is required positive definite. To remove that singularity, we must project the predictions onto the identifiable $C-1$ dimensional space, consisting of vectors orthogonal to the constant vector $1_C$. More precisely, each prediction $ \hat f(x_i)$ is replaced by $\Psi\Psi^\top \hat f(x_i)$ where $\Psi$ is a semi orthogonal matrix whose columns form a basis of $1_C^\perp$. After concatenation, the identifiable prediction alignment problem becomes
     \begin{equation}
         \Vert (I_n\otimes \Psi\Psi^\top)\left(\hat f(X)-\hat f_{A_0,B_0}(X)\right)\Vert_2= \Vert (I_n\otimes \Psi^\top)\left(\hat f(X)-\hat f_{A_0,B_0}(X)\right)\Vert_2. 
     \end{equation}
 \begin{lemma}\label{lemma:psi_inv}
     Let $\Psi\in\mathbb{R}^{C\times (C-1)}$ a semi-orthogonal matrix such that $\Psi\Psi^\top = I_C-\frac{1}{C}1_C1_C^\top$.
     For $i=1,...,n$, let the reduced
     $\tilde \Lambda_i:=\Psi^\top \Lambda_i \Psi$. If $p_{i,c}>0$ for all  $c=1,...,C$, then $\tilde \Lambda_i$ is invertible.

    \end{lemma}
    Let $\tilde \Lambda=(\tilde \Lambda_1,..., \tilde \Lambda_n)$. The following proposition allows us to rewrite the identifiable prediction alignment problem in terms of $\tilde \Lambda$ instead of $\Lambda$.
    \begin{proposition}\label{prop:proj_ntk_pred}
        Let $\tilde J:=(I_n\otimes \Psi^\top)J$ and $\tilde J_r:=(I_n\otimes \Psi^\top)J_r$, and the reduced gradient $\tilde g := (I_n\otimes \Psi^\top )(p-Y)$. Then, we have $\tilde J^\top \tilde \Lambda\tilde J = J^\top \Lambda J $ and $\tilde J^\top \tilde g = J^\top(p-Y).$ In particular, this implies for full finetuning
        $$(I_n\otimes \Psi^\top)\hat f(X)=(I_n\otimes \Psi^\top) f(W_0;X)+\tilde J(\tilde J^\top \tilde \Lambda\tilde J)^{\dagger}\tilde J^\top(-\tilde g),$$
        and for LoRA, 
        $$(I_n\otimes \Psi^\top)\hat f_{A_0,B_0}(X)=(I_n\otimes \Psi^\top) f(W_0;X)+\tilde J_r(\tilde J_r^\top \tilde \Lambda\tilde J_r)^{\dagger}\tilde J_r^\top(-\tilde g).$$
    \end{proposition}
Having resolved the singularity problem, we can extend the binary class analysis mutatis mutandis. We first start by stating the \textit{K-FAC} approximations for $\tilde J^\top \tilde J$ and $\tilde J^\top \tilde \Lambda\tilde J$.

\begin{assumption}
    Let $\tilde \delta_i := \delta_i\Psi$,  
    $\tilde T:=\sum_{i=1}^n \tilde \delta_i\tilde \delta_i^\top$, and $\Theta:= \sum_{i=1}^n \delta_i  \Lambda_i \delta_i^\top$, with prior pooling if tokenized inputs/outputs.  The following holds 
    \begin{itemize}
        \item $\tilde J^\top \tilde J
        = S\otimes \tilde T$.
    \item $\tilde J^\top \tilde \Lambda\tilde J
    = S\otimes \Theta$
    \item   $p_{i,c}>0$ for all $i=1,..,n$, $c=1,...,C$.
    \end{itemize}

    \label{assump:kfac_3}
\end{assumption}




\subsubsection{The prediction alignment problem}
\begin{algorithm}[t]
\caption{CG-LoRA for cross entropy}
\label{alg:cg-lora_ce}
\DontPrintSemicolon
\SetKwInOut{Input}{Input}
\SetKwInOut{Output}{Output}

\Input{Batch $\{(x_i,y_i)\}_{i=1}^{B_\mathrm{init}}$; rank $r$; number of LoRA layers $p$; pretrained weights $\{W_0^{(k)}\}_{k=1}^p$. \\Hyperparameters: oversampling factor $m$; number of power iterations $q$; scaling $\gamma$.}
\Output{$\{(A_0^{(k)},B_0^{(k)})\}_{k=1}^p$.}

Compute layer inputs $\{h^{(k)}\}_{k=1}^p$ by a forward pass\;
Compute layer outputs $\{\delta^{(k)}\}_{k=1}^p$ by a backward pass\;

\For{$k \gets 1$ \KwTo $p$}{\BlankLine
$(d_\mathrm{in},d_\mathrm{out})\gets\operatorname{shape}(W_0^{(k)})$ \;
\tcp{Pooling if tokenized inputs/outputs}
  \If{$\operatorname{shape}(h^{(k)})=(B_\mathrm{init},\Omega,d_\mathrm{in})$}{
  $h^{(k)}\gets \sum_{\omega=1}^\Omega h^{(k)}[:,\omega,:]$ \quad 
  $\delta^{(k)}\gets \sum_{\omega=1}^\Omega \delta^{(k)}[:,\omega,:]$} 
  $(U_S,D_S,U_{\tilde T},D_{\tilde T}) \gets
  \operatorname{Subspace}\!\bigl(h^{(k)},\delta^{(k)},m,q\bigr)$ \tcp*[r]{$O(\max(d_\mathrm{in},d_\mathrm{out})m(qn+m))$}
  {$O(nd_\mathrm{out}m)$, see Section \ref{app:hutchinson}}
  $\Phi\gets D_{\tilde T}^{-1/2}U_{\tilde T}^\top \Theta U_{\tilde T}D_{\tilde T}^{-1/2}$\tcp*[r]{$O(n\max(d_\mathrm{in},d_\mathrm{out})m+nm^2)$}
  $F  \gets \Phi^{-1/2}D_{\tilde T}^{-1/2}U_{\tilde T}^\top \nabla_W\mathcal{L}_n(f(W;X))|_{W_0}U_SD_S^{-1/2}$\tcp*[r]{$O(n\max(d_\mathrm{in},d_\mathrm{out})m+nm^2)$}
  $U_{F,r}, D_{F,r},V_{F,r}\gets  \operatorname{SVD}_r(F)$ \tcp*[r]{$O(m^2r)$} 
  $V_{\Phi^{-1/2}F,r}\gets  \text{right} \operatorname{SVD}_r(\Phi^{-1/2}F)$ \tcp*[r]{$O(m^2r)$}  
  $E\gets \operatorname{orth}(FV_{\Phi^{-1/2}F,r})$\tcp*[r]{$O(m^2r)$}
  $L\gets U_{\tilde T}D_{\tilde T}^{-1/2}\Phi^{-1/2}E, \quad R\gets U_SD_S^{-1/2}V_{\Phi^{-1/2}F,r}$\tcp*[r]{$O(\max(d_\mathrm{in},d_\mathrm{out})mr)$}
  $(A_0^{(k)},B_0^{(k)})\gets \operatorname{BalancedSVD} (L,R, E^\top FV_{\Phi^{-1/2}F,r})$\tcp*[r]{$O(\max(d_\mathrm{in},d_\mathrm{out})r^2)$}
  $A_0^{(k)}\gets \frac{d_\mathrm{out}^{1/4}}{\gamma}A_0^{(k)} , \quad B_0^{(k)} \gets \frac{d_\mathrm{out}^{1/4}}{\gamma }B_0^{(k)}$ }
\Return $\{(A_0^{(k)},B_0^{(k)})\}_{k=1}^p$\;
\end{algorithm}
Under Assumption \ref{assump:kfac_3}, all results derived for the binary cross entropy loss can be transposed mutatis mutandis to the multiclass case by updating $T$ by $\tilde T$ and $O$ by $\Theta$. For ease of exposition, we only state a summarized result with $r\leq\operatorname{rank}(F)$. 
\begin{theorem}\label{thm:cg_lora_multi}
Let Assumption \ref{assump:kfac_3} be true. Let $\Phi := D_{\tilde T}^{-1/2}U_{\tilde T}^\top \Theta U_{\tilde T} D_{\tilde T}^{-1/2}$. Then $\Phi$ is positive definite.

Let $F =\Phi^{-1/2}D_{\tilde T}^{-1/2} U_{\tilde T}^\top \nabla  U_SD_S^{-1/2}\in\mathbb{R}^{r_T\times r_S}$
Let $E=\operatorname{orth}(FV_{\Phi^{-1/2}F,r})$. Let top $r$ right modes $ R_r=U_SD_S^{-1/2}V_{\Phi^{-1/2}F,r}$ and  left modes $L_r=U_TD_T^{-1/2}\Phi^{-1/2}E$.   Consider the following initialization scheme
\begin{equation}\begin{cases}
        (A_0,B_0) \gets \operatorname{Balanced SVD}(L_r,R_r, E^\top FV_{\Phi^{-1/2}F,r}) \\
        (A_0,B_0) \gets (\frac{d_\mathrm{out}^{1/4}}{\gamma}A_0,\frac{d_\mathrm{out}^{1/4}}{\gamma}B_0 )
    \end{cases}  \tag{CG-LoRA}\end{equation} 
    then 
    $\Vert \hat{f}(X) - \hat{f}_{A_0,B_0}(X) \Vert_2\leq\Vert \Phi^{-1/2}(I-\Pi_E)\Phi^{1/2}\Vert_2(\sum_{i\geq r+1}s_i^2(\Phi^{-1/2}F))^{1/2}.$

\end{theorem}



\subsection{Proofs}

\subsubsection{Proof of Proposition \ref{prop:gen_loss_ntk_preds}}

Let $\theta = \mathrm{vec}(\Delta W)$. The normal equations give
$$J^\top \nabla_f\mathcal{L}_n(f)|_{f=f(W_0;X)} + J^\top \nabla^2_f\mathcal{L}_n(f)|_{f=f(W_0;X)}J\theta=0,$$
from which we immediately deduce the minimum norm solution as 
$$\theta^\star = -(J^\top \nabla^2_f\mathcal{L}_n(f)|_{f=f(W_0;X)}J)^\dagger J^\top \nabla_f\mathcal{L}_n(f)|_{f=f(W_0;X)}$$
The NTK prediction is then given by 
$$\hat f(X) = f(W_0;X)+J\theta^\star =f(W_0;X)+J (J^\top \nabla^2_f\mathcal{L}_n(f)|_{f=f(W_0;X)}J)^\dagger J^\top (-\nabla_f\mathcal{L}_n(f)|_{f=f(W_0;X)}).$$
The proof for LoRA is identical by replacing the full finetuning Jacobian $J$ by its counterpart $J_r$.

\subsubsection{Proof of Lemma \ref{lemma:hess_binary}}
The binary cross entropy loss is defined as 
$$l(f_i,y_i)=-y_i\log p_i -(1-y_i)\log(1-p_i),$$
where $p_i=\frac{1}{1+e^{-f_i}}$.
Hence 
$$\nabla_fl(f_i,y_i)=\frac{\partial l}{\partial p}\frac{\partial p}{\partial f}=(-\frac{y_i}{p_i}+\frac{1-y_i}{1-p_i})p_i(1-p_i)=-y_i(1-p_i)+(1-y_i)p_i=p_i-y_i.$$
Furthermore $\nabla_f^2l(f_i,y_i)=\nabla_f(p_i-y_i)=p_i(1-p_i)$.
Since $\mathcal{L}_n(f(W_0,X))=\sum_{i=1}^n l(f_i,y_i)$ then we immediately deduce $\nabla_f\mathcal{L}_n(f)|_{f=f(W_0;X)} = p-Y$, and $\nabla_f^2\mathcal{L}_n(f)|_{f=f(W_0;X)} = \operatorname{diag}(p\odot (1-p)).$

\subsubsection{Proof of Proposition \ref{prop:phi_inv}}
\begin{proof}
To show invertibility of $\Phi$, it is equivalent to show that $O$ is positive definite on $\operatorname{Im}(T)$. Indeed, let $x\in \mathbb{R}^{r_T}$ then
$$x^\top \Phi x= (U_TD_T^{-1/2}x)^\top O (U_TD_T^{-1/2}x) $$
where $U_TD_T^{-1/2}x\in\operatorname{Im}(T)$.

Now, let $\Delta=(\delta_1,...,\delta_n)\in\mathbb{R}^{ d_\mathrm{out}\times nC} $ then $T=\Delta \Delta^\top$ and $O=\Delta \mathrm{Diag}(p_1(1-p_1), ...,p_n(1-p_n)) \Delta^\top $.
In particular since each $0<p_i<1$ then $\operatorname{Im}(O)=\operatorname{Im}(T)$. Furthermore, $\operatorname{Im}(O)\perp \operatorname{Ker}(O^\top)=\operatorname{Ker}(O)$  since $O$ is symmetric. This shows that $\operatorname{Im}(T)\perp \operatorname{Ker}(O)$ and in particular that  
$O$ is positive definite on $\operatorname{Im}(T)$ and $\Phi$ is invertible.    
\end{proof}

\subsubsection{Proof of Proposition \ref{prop:diff_preds_binary}}
\begin{proof}
As for the squared-loss case, the proof is decomposed into two steps:
\begin{enumerate}
    \item Rewrite the expressions of $\hat f(X)$ and $\hat f_{A_0,B_0}(X)$ as orthogonal projections of the finetuning residuals.
    \item Compute the difference between the projections.
\end{enumerate}
    \paragraph{Step 1}: Let us introduce whitened coordinates to make appear orthogonal projections which are simpler to manipulate. 
    
    Let $J=UDV^\top$ be a thin SVD. Let $\Gamma=\mathcal{C}_0^{1/2}J$. Since $\mathcal{C}_0$ is invertible, we can write
    $$\hat f(X)  = f(W_0;X)+\mathcal{C}_0^{-1/2}\Gamma(\Gamma^\top \Gamma )^{\dagger}\Gamma^\top \mathcal{C}_0^{-1/2}(Y-p_0)$$   
    $\Pi_\Gamma= \Gamma(\Gamma^\top \Gamma )^{\dagger}\Gamma^\top=\Gamma \Gamma^\dagger$ is simply the projection over $\operatorname{col}(\Gamma)=\operatorname{col}(\mathcal{C}_0^{1/2}U)$. Hence, we can also rewrite 
    $$\Pi_\Gamma=\Pi_{\mathcal{C}_0^{1/2}U}= \mathcal{C}_0^{1/2} U (U^\top \mathcal{C}_0 U)^\dagger U^\top \mathcal{C}_0^{1/2}.$$
      Let $G=U^\top \mathcal{C}_0U$ then $x^\top Gx = (Ux)^\top \mathcal{C}_0(Ux)>0$ since $\mathcal{C}_0\succ0$. This implies that $G$ is non singular and we can write
    $$\hat f(X)  = f(W_0;X)+U G^{-1} U^\top(Y-p_0)$$
    For LoRA, we obtain instead
     $$\hat f_{A_0,B_0}(X)  = f(W_0;X)+\mathcal{C}_0^{-1/2}\Pi_{\Gamma_r} \mathcal{C}_0^{-1/2}(Y-p_0)$$    where $\Gamma_r=\mathcal{C}_0^{1/2}JH$. In particular $\operatorname{col}(\Gamma_r)=\operatorname{col}(\mathcal{C}_0^{1/2}UP)$ where $P=DV^\top H$.
     Using the same calculus as above, this implies
    \begin{align*}
        \hat f_{A_0,B_0}(X)  &= f(W_0;X)+UP((UP)^\top \mathcal{C}_0  UP)  (UP)^\top(Y-p_0) \\
&=f(W_0;X)+UP(P^\top GP)  P^\top U^\top (Y-p_0) \\
&= f(W_0;X)+UG^{-1/2}G^{1/2}P(P^\top GP)(G^{1/2}P)^\top G^{-1/2} U^\top(Y-p_0)\\
&=f(W_0;X)+UG^{-1/2} \Pi_Q G^{-1/2} U^\top(Y-p_0)
    \end{align*}
    where $Q:=G^{1/2}P$.
    Taking the difference, we obtain  
$$\Vert \hat{f}(X)-\hat f_{A_0,B_0}(X)\Vert = \Vert UG^{-1/2}(I-\Pi_Q)G^{-1/2}U^\top (Y-p_0)\Vert_2$$  
Finally, recall that
\begin{align*}
    U^\top (Y-p_0) &= (D_S^{-1/2} U_S^\top \otimes D_T^{-1/2}U_T^\top)  J^\top (Y-p_0)\\
    &=-(D_S^{-1/2} U_S^\top \otimes D_T^{-1/2}U_T^\top) \operatorname{vec}(\nabla),
\end{align*}
which leads to our intermediate value
$$\Vert \hat{f}(X)-\hat f_{A_0,B_0}(X)\Vert = \Vert G^{-1/2}(I-\Pi_Q)G^{-1/2}(D_S^{-1/2} U_S^\top \otimes D_T^{-1/2}U_T^\top) \operatorname{vec}(\nabla)\Vert_2$$  
\paragraph{Step 2}:
We need now to characterize  the column space of $Q$. First observe that
\begin{align*}
    G&=U^\top\mathcal C_0U\\
    &=D^{-1}V^\top (J^\top\mathcal{C}_0 J)VD^{-1}\\
    &=(D_S^{-1/2} U_S^\top \otimes D_{ T}^{-1/2} U_{ T}^\top) (J^\top\mathcal{C}_0 J)(U_S D_S^{-1/2}\otimes U_{ T} D_{ T}^{-1/2})\\
    &=(D_S^{-1/2} U_S^\top \otimes D_{ T}^{-1/2} U_{ T}^\top) (S\otimes O)(U_S D_S^{-1/2}\otimes U_{ T} D_{ T}^{-1/2}) \\
    &=I_{r_\mathrm{r_S}}\otimes (D_{ T}^{-1/2} U_{ T}^\top O   U_{ T} D_{ T}^{-1/2}) \\
    &= I_{r_S}\otimes \Phi.
\end{align*}
hence by usual properties of the Kronecker product, we also have $G^{1/2}=I_{r_S} \otimes \Phi^{1/2}$. This gives
\begin{align*}
    Q &= G^{1/2} (D_S^{1/2} U_S^\top \otimes D_T^{1/2} U_T^\top) (I_{d_\mathrm{in}}\otimes B_0 \quad A_0^\top \otimes I_{d_\mathrm{out}}) \\
    &= (I_{r_S} \otimes \Phi^{1/2})(D_S^{1/2} U_S^\top \otimes D_T^{1/2} U_T^\top B_0 \quad  D_S^{1/2}U_S^\top A_0^\top \otimes D_T^{1/2}U_T^\top)\\
    &=(D_S^{1/2} U_S^\top \otimes \Phi^{1/2}D_T^{1/2} U_T^\top B_0 \quad D_S^{1/2}U_S^\top A_0^\top \otimes \Phi^{1/2}D_T^{1/2}U_T^\top)
\end{align*}
    therefore its column space is the sum of the two columns spaces:
    
(i) $\operatorname{col}(D_S^{1/2} U_S^\top \otimes \Phi^{1/2}D_T^{1/2} U_T^\top B_0 ) = \mathbb{R}^{r_S} \otimes \operatorname{col}(\Phi^{1/2}D_T ^{1/2}U_T^\top B_0 )$.

(ii) $\operatorname{col}(D_S^{1/2}U_S^\top A_0^\top \otimes \Phi^{1/2}D_T^{1/2}U_T^\top) = \operatorname{col}(D_S^{1/2}U_S^\top A_0^\top)\otimes \operatorname{col}(\Phi^{1/2}D_T^{1/2}U_T^\top)= \operatorname{col}(D_S^{1/2}U_S^\top A_0^\top)\otimes \mathbb{R}^{r_T}$.

    Then following the same arguments as the proof of Lemma \ref{lemma:k_vs_kr}, we obtain that 
    $$\operatorname{col}(Q)^\perp = \operatorname{col}(D_S^{1/2}U_S^\top A_0^\top)^\perp\otimes  \operatorname{col}(\Phi^{1/2}D_T^{1/2}U_T^\top B_0)^\perp$$
    Therefore writing $g=\operatorname{vec}(\nabla)$
    \begin{align*}
        &\Vert \hat{f}(X) - \hat f_{A_0,B_0}(X) \Vert_2\\&=\Vert G^{-1/2}[(I_{r_S}-\Pi_{D_S^{1/2}U_S^\top A_0^\top})\otimes (I_{r_T}-\Pi_{\Phi^{1/2}D_T^{1/2}U_T^\top B_0})]G^{-1/2}(D_S^{-1/2}U_S^\top \otimes D_T^{-1/2}U_T^\top) g\Vert_2 \\
        &= \Vert (I_{r_S} \otimes \Phi^{-1/2} )[(I_{r_S}-\Pi_{D_S^{1/2}U_S^\top A_0^\top})\otimes (I_{r_T}-\Pi_{\Phi^{1/2}D_T^{1/2}U_T^\top B_0})] (I_{r_S} \otimes \Phi^{-1/2} )(D_S^{-1/2}U_S^\top \otimes D_T^{-1/2}U_T^\top) g\Vert_2 \\
        &=\Vert [(I_{r_S}-\Pi_{D_S^{1/2}U_S^\top A_0^\top})D_S^{-1/2}U_S^\top ]\otimes [(\Phi^{-1/2} (I_{r_T}-\Pi_{\Phi^{1/2}D_T^{1/2}U_T^\top B_0}) \Phi^{-1/2}) D_T^{-1/2}U_T^\top]g\Vert_2\\
        &= \Vert \Phi^{-1/2}(I_{r_T}-\Pi_{\Phi^{1/2}D_T^{1/2}U_T^\top B_0}) F  (I_{r_S}-\Pi_{D_S^{1/2}U_S^\top A_0^\top})\Vert_F
    \end{align*}
    where $F=\Phi^{-1/2}D_T^{-1/2}U_T^\top \nabla U_S D_S^{-1/2}$.
\end{proof}


\subsubsection{Proof of Proposition \ref{prop:upp_bound_preds} }
\begin{proof}
    Let $A_0=U_SD_S^{-1/2}V_{\Phi^{-1/2}F,r}$, then $P_A=\Pi_{V_{\Phi^{-1/2}F,r}}$.
    Hence, 
    $$\Vert \Phi^{-1/2}F(I-P_A)\Vert_F=(\sum_{i\geq r+1} s_i^2(\Phi^{-1/2}F))^{1/2} .$$
    Let $B_0=U_TD_T^{-1/2}\Phi^{-1/2}EQ$, then $P_B=\Pi_{E}$. 
    This implies the desired result 
    \begin{align*}
    \Vert \Phi^{-1/2}(I_{r_T}-P_B)F(I_{r_S}-P_A)\Vert_F &=\Vert  \Phi^{-1/2}(I_{r_T}-P_B)\Phi^{1/2} [\Phi^{-1/2}F(I_{r_S}-P_A)] \Vert_F\\
    &\leq  \Vert\Phi^{-1/2}(I_{r_T}-P_B)\Phi^{1/2}\Vert_2 \Vert \Phi^{-1/2}F(I_{r_S}-P_A)\Vert_F \\
    &=\Vert\Phi^{-1/2}(I_{r_T}-\Pi_E)\Phi^{1/2}\Vert_2(\sum_{i\geq r+1} s_i^2(\Phi^{-1/2}F))^{1/2} . 
\end{align*}
\end{proof}

\subsubsection{Proof of Proposition \ref{prop:opt_q}}
\begin{proof}
The proof follows mutatis mutandis the proof of Proposition \ref{prop:opt_q} by noting that $E$ is full column rank hence the pseudo inverses can be simplified.   
\end{proof}

\subsubsection{Proof of Theorem \ref{thm:cg_lora_binary_improv}}

\begin{proof}
    By definition of Algorithm \ref{alg:bal_svd}, our proposed $A_0,B_0$ naturally verify $\operatorname{col}(A_0)\subset \operatorname{col}(R_r)=\operatorname{col}(U_SD_S^{-1/2}V_{\Phi^{-1/2}F,r})$ and $\operatorname{col}(B_0)=\operatorname{col}(L_r)=\operatorname{col}(U_TD_T^{-1/2}\Phi^{-1/2}E)$. Hence, this implies that $A_0=U_SD_S^{-1/2}V_{\Phi^{-1/2},r}Q_A$ and $B_0=U_TD_T^{-1/2}\Phi^{-1/2}EQ_B$ for some invertible $Q_A,Q_B$. Therefore, Proposition \ref{prop:upp_bound_preds} immediately applies and gives the upper bound.

\end{proof}

\subsubsection{Proof of Lemma \ref{lemma:hess_multi}}
The multiclass cross entropy loss is defined as
$$l(f_i,y_i)=-\sum_{c=1}^Cy_{i,c}\log p_{i,c}=-\sum_{c=1}^Cy_{i,c}f_{i,c} +\log(\sum_{c=1}^C e^{f_{i,c}}),$$
where $p_{i,c}=\frac{e^{f_{i,c}}}{\sum_{c'=1}^C e^{f_{i,c'}}}$. Hence
$$\nabla_f(l(f_i,y_i))=-y_i+p_i.$$
Furthermore, $\nabla^2_f(l(f_i,y_i))|=\nabla_f p_i=\operatorname{Diag}(p_i)-p_ip_i^\top$. 
Stacking inputs yields the desired result.

\subsubsection{Proof of Proposition \ref{lemma:psi_inv}}

Let $i=1,...,n$. Let us first show that the kernel of $\Lambda_i$ is indeed $\operatorname{span}(1_C)$. Let any vector $v$ then
$$v^T \Lambda_iv = v^\top \operatorname{Diag}(p_i)v - (v^\top p_i)^2 = \sum_{c=1}^C p_{i,c} v_c^2 - (\sum_{c=1}^C p_{i,c} v_c)^2.$$
We recognize the variance of the numbers $v_1,...,v_C$ under probability weights $p_{i,1},...p_{i,C}$. This already proves that $\Lambda_i$ is positive semi definite.

Furthermore, if we assume that each $p_{i,c}>0$ then by definition of a variance, it can be zero if and only if the coordinates $v_c$ are all equal or equivalently $v\in \operatorname{span}(1_C)$.

Going back to our lemma, let $v$ a vector then clearly $v^\top\Psi^\top \Lambda_i \Psi v >0$ since 
$\Psi v\perp 1_c$ by definition of $\Psi$. This proves that $\tilde \Lambda_i$ is invertible.

\subsubsection{Proof of Proposition \ref{prop:proj_ntk_pred}}

First, we observe that by definition of $\tilde \Lambda_i$, we immediately obtain that 
$$\Lambda = (I_n\otimes \Psi)\tilde \Lambda (I_n\otimes \Psi^\top ), $$
hence we have 
$$J^\top \Lambda J = J^\top (I_n\otimes \Psi)\tilde \Lambda (I_n\otimes \Psi^\top ) J = \tilde J\tilde \Lambda \tilde J.$$
Now, let $y_i$ and $p_i$ be the label and predicted probability of sample $i$. Since $y_i$ is a one hot label and $p_i$ is a probability vector then $p_i^\top 1_C=y_i^\top 1_C=1$ which implies that $p_i-y_i\in1_C^\perp$ or equivalently $p_i-y_i=\Psi\Psi^\top (p_i-y_i)$. Stacking all inputs together, we obtain that 
$$(I_n\otimes \Psi\Psi^\top)(p-Y)= p-Y=g$$ which implies that 
$$\tilde J^\top \tilde g=J^\top(I_n\otimes \Psi\Psi^\top)g = J^\top g.$$
By definition of $\hat f(X)$, we have 
\begin{align*}
    (I_n\otimes \Psi^\top)\hat f(X)&=(I_n\otimes \Psi^\top) (f(W_0;X)+ J( J^\top  \Lambda J)^{\dagger} J^\top g) \\
    &=(I_n\otimes \Psi^\top) f(W_0;X) + \tilde J ( J^\top  \Lambda J)^{\dagger} J^\top g \\
    &= (I_n\otimes \Psi^\top) f(W_0;X) + \tilde J (\tilde J^\top  \Lambda \tilde J)^{\dagger} \tilde J^\top \tilde g. 
\end{align*}

\subsubsection{Proof of Theorem \ref{thm:cg_lora_multi}}

The proof follows mutatis mutandis the proof of Theorem \ref{thm:cg_lora_binary_improv} and is omitted.

\newpage

\section{Obtaining activation signals}
\label{app:hutchinson}

We wish to compute the pre-activation derivatives signal, defined by the \textit{K-FAC} design, $T, \tilde T,  \Theta$ which have general expression
$$\sum_{i=1}^n \delta_i M_i\delta_i^\top,$$
where: 1) $M_i=I_C$ for $ T$, 2) $M_i=(I_C-\frac{1}{C}1_C1_C^\top)$ for  $\tilde T$, 3) $M_i=\Lambda_i=\operatorname{Diag}(p_i)-p_ip_i^\top$ for  $\Theta$. 

This task is challenging for two reasons: 1) we need to compute derivatives which involve backward passes, 2) the quantities $\delta_i,M_i$ scale with the number of classes $C$ (or $NC$ for generative tasks where $N$ is sequence length) which can be extremely large.

For the first point, let us observe that our desired quantity can be obtained by  vector-Jacobian products (VJP). For a linear layer with batched input $h\in\mathbb{R}^{n\times d_\mathrm{in}}$ and output $u\in\mathbb{R}^{n\times d_\mathrm{out}}$, consider for any $z\in\mathbb{R}^{n\times C}$ the following
\begin{align*}
    \phi 
    = (\delta_1z_1 \ ...
          \delta_n z_n)^\top
    \in\mathbb{R}^{n\times d_\mathrm{out}},
\end{align*}
Hence $\phi^\top \phi =  \sum_{i=1}^p\delta_i z_iz_i^\top\delta_i^\top$ is a good candidate. In particular, computing $\phi^\top \phi \Omega$ or $Q^\top\phi^\top\phi Q$ for $\Omega, Q\in\mathbb{R}^{d_\mathrm{out}\times m}$ costs $O(nd_\mathrm{out} m)$. 

For the second point, the question becomes to know which $z_i$ to choose and what is the incurred cost.
\subsection{Exact computations (small $C$)}

\begin{enumerate}
    \item For computing $T$, let $c=1,...,C$ and  $z_i=e_c$ for any $i=1,...,n$,  where $e_c=(0 ... 0 \ 1 \ 0 ... 0)^\top$ is the canonical basis vector at position $c$. We compute the corresponding $\phi^{(c)}$ for each $c=1,...,C$ and take the sum to obtain 
    $$\sum_{c=1}^C(\phi^{(c)})^\top \phi^{(c)}= \sum_{c=1}^C\sum_{i=1}^n \delta_ie_ce_c^\top\delta_i^\top=\sum_{i=1}^n \delta_i\delta_i^\top=T.$$
    \item For computing $\tilde T$, we repeat the above process with $z_i= e_c-\frac{1}{C}1_C$  for $c=1,...,C$.
    \begin{align*}
        \sum_{c=1}^C(\phi^{(c)})^\top \phi^{(c)}&= \sum_{c=1}^C\sum_{i=1}^n \delta_i(e_c-\frac{1}{C}1)(e_c-\frac{1}{C}1_C)^\top\delta_i \\
        &=\sum_{i=1}^n\delta_i[(\sum_{c=1}^C e_ce_c^\top)-\frac{1}{C}(\sum_{c=1}^Ce_c)1_C^\top-\frac{1}{C}1_C(\sum_{c=1}^Ce_c^\top)+\frac{1}{C^2}(\sum_{c=1}^C1_C1_C^\top)]\delta_i^\top \\
        &=\sum_{i=1}^n\delta_i[I_C-\frac{1}{C}1_C1_C^\top]\delta_i^\top=\tilde T
    \end{align*}
    where we used the fact that $\sum_{c=1}^Ce_c=1_C$.
    \item For computing $\Theta$,  we repeat the above process with $z_i= \sqrt{p_{i,c}}(e_c-p_i)$  for $c=1,...,C$. 
    \begin{align*}
        \sum_{c=1}^C(\phi^{(c)})^\top \phi^{(c)}&= \sum_{i=1}^n \delta_i[\sum_{c=1}^Cp_{i,c}(e_c-p_i)(e_c-p_i)^\top]\delta_i \\
        &=\sum_{i=1}^n \delta_i[\sum_{c=1}^Cp_{i,c} (e_ce_c^\top -e_cp_i^\top -p_ie_c^\top +p_ip_i^\top) ]\delta_i^\top\\
        &=\sum_{i=1}^n \delta_i[ \operatorname{Diag}(p_i)-(\sum_{c=1}^Cp_{i,c}e_c) p_i^\top -p_i (\sum_{c=1}^Cp_{i,c}e_c^\top) + (\sum_{c=1}^Cp_{i,c})p_ip_i^\top]\delta_i^\top \\
        &=\sum_{i=1}^n \delta_i[ \operatorname{Diag}(p_i)-p_i p_i^\top ]\delta_i^\top  =\Theta
    \end{align*}
    where we used the fact that $\sum_{c=1}^Cp_{i,c}=1$ and $\sum_{c=1}^Cp_{i,c}e_c=p_i$.
\end{enumerate}
In all above cases, this requires doing $C$ backward passes as well as a computational cost of $O(C)$ due to the sum.
\subsection{Approximate computation (large $C$)}
When $C$ is large as in generative modeling (where $C$ gets replaced by $NC$), we propose the following alternative. 
We consider random variables $z_i$ with $\mathbb{E}(z_iz_i^\top)=M_i$ such that only
 $$\mathbb{E}(\phi^\top \phi) =\sum_{i=1}^p\delta_iM_i\delta_i^\top$$
 holds.
\begin{enumerate}
\item For computing $T$, we consider any $z_i$ such that $\mathbb{E}(z_iz_i^\top)=I_C$ such as Rademacher or standard normal random variables. 
\item For computing  $\tilde T=\sum_{i=1}^n\delta_i\Psi\Psi^\top\delta_i^\top$, we consider  Rademacher/standard normal variables and additionally center them as following
$$\tilde z_i=(I_C-\frac{1}{C}1_C1_C^\top) z_i=z_i-(\frac{1}{C}\sum_{c=1}^Cz_{i,c})1_C.$$
\item For computing $\Theta=\sum_{i=1}^n\delta_i\Lambda_i\delta_i$, we use a more involved argument.
Let $g_i\sim\mathcal{N}(0,I_C)$ and
$$z_i=(I_C-p_i1_C^\top)\operatorname{Diag}(\sqrt p_i)g_i$$
therefore $z_i$ is itself Gaussian with 
$$\operatorname{Cov}(z_i)=(I_C-p_i1_C^\top)\operatorname{Diag}(p_i)(I_C-1_Cp_i^\top)$$
Note that $ (I_C-p_i1_C^\top)\operatorname{Diag}(p_i)=\operatorname{Diag}(p_i)-p_ip_i^\top$ since $1_C^\top\operatorname{Diag}(p_i)=p_i^\top$.

And 
$$(\operatorname{Diag}(p_i)-p_ip_i^\top)(I_C-1_Cp_i^\top)=\operatorname{Diag}(p_i)-p_ip_i^\top -p_ip_i^\top+p_i(p_i^\top1_C)p_i^\top = \operatorname{Diag}(p_i)-p_ip_i^\top=\Lambda_i$$
Therefore, the Hutchinson  estimator enjoys
$$\mathbb{E}(\phi^\top \phi) =\mathbb{E}( \sum_{i=1}^n \delta_{i}z_iz_i^\top  \delta_i^\top)=\sum_{i=1}^n\delta_i \operatorname{Cov}(z_i)\delta_i^\top = \sum_{i=1}^n \delta_i\Lambda_i \delta_i^\top$$
\end{enumerate}
In all above cases, we compute $\phi^{(p)}$ for $p=1,...,P$ and take the average $\frac{1}{P}\sum_{p=1}^P(\phi^{(p)})^\top \phi^{(p)}$ which requires $P$ backward passes. Furthermore, the computational complexity is now in $O(P)$.

\subsection{Number of backward passes}\label{app:nb_passes}

Each operation requiring $\delta$, namely $\delta(\delta^\top \Omega)$, $(Q_T^\top\delta)( \delta^\top Q_T)$ in Algorithm \ref{alg:subspace} or $\Theta=\sum_{i=1}^n \delta_i\Lambda_i\delta_i^\top$ in Algorithm \ref{alg:cg-lora_ce} requires $\min(C,P)$ backward passes depending on whether we implement exact/approximate VJP. Hence, in total, we perform $2(1+q)\min(P,C)$ backward passes where $q$ is the number of power iterations chosen for the Rayleigh Ritz method. With $q=0$, compared to existing LoRA-GA and LoRA-One, our method has $2\min(P,C)$ additional backward passes. Nonetheless,
we compensate those backward passes with our efficient Algorithm \ref{alg:subspace}. 

\newpage
\section{Experimental details}\label{app:exp}

\subsection{Baselines}
We used the available implementations of LoRA-GA and LoRA-One by sourcing their provided repository, available at \url{https://github.com/Outsider565/LoRA-GA/tree/main} and \url{https://github.com/YuanheZ/LoRA-One/tree/main} respectively. We also applied their recommended scaling parameter $\gamma$. For LoRA+, we set the ratio to their proposed default when using Kaiming initialization namely $\eta_B/\eta_A=2^4$.
\subsection{Training details for Natural Language Understanding}

\textit{LoRA parameters}. Rank $r=8$, $\alpha=2r$. LoRA layers are all query/value (24 layers for RoBERTa-base, 72 layers for T5-base)

\textit{Preprocessing.} We use default tokenizer from Hugging face with additional dynamic padding. The maximum input sequence length is set at $\Omega=128$. For T5-base on GLUE datasets, input sentences are further preprocessed to be task-specific ("classify the grammaticality of \{sentence\}" for CoLA) such that classification is possible, and maximum output length is set to 8.

\textit{Number of samples for non-zero based initialization.} CG-LoRA, LoRA-GA, LoRA-One are run on the same single random batch of size $B_\mathrm{init}=32$ sampled using class stratification.

\textit{Training.}   
Batch size $B_\mathrm{fine}=32$ samples and perform a single full pass on the training dataset.  AdamW optimizer for all methods, with default PyTorch values and no weight decay.

\textit{Learning rate.} The learning rate is found by grid search for each dataset and each model. Candidate values are chosen from $\{2e-5, 5e-5, 1e-4, 2e-4, 5e-4, 8e-4,1e-3\}$, and they are compared using a separate validation dataset. During finetuning, we also use a cosine schedule with warm-up ratio equal to 0.03.

\textit{Precision.} Backbone and LoRA adapters are in FP32.

\textit{Hardware:} RoBERTa-base/T5-base : Single Nvidia A40. 

\subsection{Training details for Natural Language Generation}

\textit{LoRA parameters}. Rank $r=8$, $\alpha=2r$. LoRA layers are all query/value (64 layers for LLaMa 2-7B)

\textit{Preprocessing.} We reused the exact implementation of LoRA-GA/LoRA-One for preprocessing. The maximum input sequence length is set at $\Omega=1024$. 

\textit{Generation.} Greedy decoding. The maximum output length is $512$. 

\textit{Number of samples for non-zero based initialization.} CG-LoRA, LoRA-GA, LoRA-One are run on the same random batch of size $B_\mathrm{init}=32$ using gradient accumulation with micro batches of size 1. 

\textit{Training.} Batch size $B_\mathrm{fine}=1$ sample. Single full pass on the training dataset,
with gradient accumulation every $32$ steps. AdamW optimizer for all methods, with default PyTorch values and no weight decay.

\textit{Learning rate.} The learning rate is found by grid search for each dataset and each model. Candidate values are chosen from  $\{2e-5, 5e-5, 1e-4, 2e-4, 5e-4, 8e-4,1e-3\}$, and they are compared using a separate validation dataset. During finetuning, we also use a cosine schedule with warm-up ratio equal to 0.03.

\textit{Precision.} Backbone is in FP16 and LoRA adapters are in FP32.

\textit{Hardware:} LLaMa 2-7B : Single Nvidia A100. 

\subsection{Algorithmic details for CG-LoRA}

For our experiments proposed in Section \ref{sec:exp}, we implement the cross-entropy loss based initialization (Algorithm \ref{alg:cg-lora_ce}) for all experiments. A more thorough analysis of differences between the two algorithms is proposed in Appendix \ref{app:ce_vs_squared}.

\textit{Float precision}. FP64 for initializing $(A_0,B_0)$ which are then moved back to FP32 for finetuning.


\textit{Exact vs approximate computation of output derivatives }. For RoBERTa-base and T5-base, we considered exact computation for all GLUE datasets. For LLaMa 2-7B, we considered  approximate computations with $P=1$. See Appendix \ref{app:hutchinson} for more details on their differences.

\textit{Hyperparameters.} 1) the oversampling factor $m$ for the sketching matrix, 2) the number of power iterations $q$. 3) the scaling factor $\gamma$.  

We consider $m =2r$ for all models/datasets. 

For RoBERTa-base: For the noisier experiments, namely CG-LoRA (shift) on all datasets, and CG-LoRA (no-shift) on small datasets CoLA and MRPC, we take $q=1$. Otherwise, we take $q=0$. For T5-base and LLaMA 2-7B, we consider $q=0$ for all datasets/methods.


Finally, we choose $\gamma=16$ for Natural Language Understanding and $\gamma=64$ for Natural Language Generation, as proposed in \cite{gawang2024lora}.

\subsection{ Results for T5-base and additional plots} \label{app:results}

\begin{table}[t!]
\small
\centering
\caption{Test accuracy of finetuned T5-base. LoRA rank $r=8$. Value in bold represents the best LoRA method.}
\vspace{1em}
\label{tab:glue-t5}
\renewcommand{\arraystretch}{1.15}
\setlength{\tabcolsep}{6pt}
\adjustbox{width=0.9\linewidth}{
\begin{tabular}{lccccc}
\toprule
& \textbf{MNLI} & \textbf{SST-2} & \textbf{CoLA} & \textbf{QNLI} & \textbf{MRPC}  \\
\midrule
Full
& 86.00$_{\pm0.12}$
& 94.22$_{\pm0.14}$
& $81.33_{\pm0.35}$
& 93.16$_{\pm0.10}$
& 84.96$_{\pm 1.02}$  
 \\
rsLoRA
& 85.18$_{\pm0.12}$
& 93.76$_{\pm 0.19}$
& $76.22_{\pm 3.54}$
& 93.05$_{\pm0.11}$
& 84.47$_{\pm1.02}$
 \\
 LoRA+ & $84.44_{\pm 0.12}$ & 93.92$_{\pm 0.09}$ & $79.08_{\pm 1.41}$ & $92.89_{\pm 0.13}$ & $84.31_{\pm 0.69}$\\
 \midrule
LoRA-GA
& 85.12$_{\pm0.15}$
& 93.96$_{\pm0.05}$
& $79.41_{\pm0.56}$
& 92.91$_{\pm 0.32}$
& 83.90$_{\pm1.02}$
 \\
 LoRA-One
& 85.53$_{\pm0.15}$
& \textbf{94.15}$_{\pm 0.33}$
& 79.16$_{\pm0.83}$
& 92.88$_{\pm0.11}$
& 84.31$_{\pm0.52}$\\

\midrule

CG-LoRA (no-shift)
& 85.40$_{\pm0.11}$
& \textbf{94.15}$_{\pm0.33}$
& \textbf{80.24}$_{\pm0.61}$
& 92.96$_{\pm0.08}$
& \textbf{85.94}$_{\pm 0.50}$ \\
CG-LoRA (shift)
& \textbf{85.62}$_{\pm0.13}$
& 93.96$_{\pm 0.23}$
&  79.45$_{\pm 0.78}$
& \textbf{93.08}$_{\pm0.15}$
& 83.82$_{\pm0.35}$ \\
\bottomrule
\end{tabular}}
\end{table}

For T5-base, we compute the output logit functions by teacher forcing. When computing the K-FAC subspaces in Algorithm \ref{alg:subspace}, we further multiply each input tokens by their encoding mask,  i.e $h_{i,w} \gets h_{i,w}\times\mathrm{mask}_{i,w}$,  before doing the pooling average to avoid contributions from padded tokens. 

Table  \ref{tab:glue-t5} in Appendix show that the results are more balanced compared to those for RoBERTa-base. Since T5-base is  a text to text architecture, the pretrained loss gradients are more informative for the finetuning task, and CG-LoRA performs on par with related state of the art algorithms. Nonetheless, we still observe a more significant advantage on smaller datasets where CG-LoRA tends to converge at a faster rate than other LoRA methods. Both shifted and not shifted versions perform similarly well which can be explained by the fact that the gradients are not noisy and the adapter initialization are correctly scaled as discussed by \cite{li2025beyond}. 

\begin{figure}[h!]
    \centering
    \vspace{-0.5em}
    \includegraphics[width=0.45\linewidth]{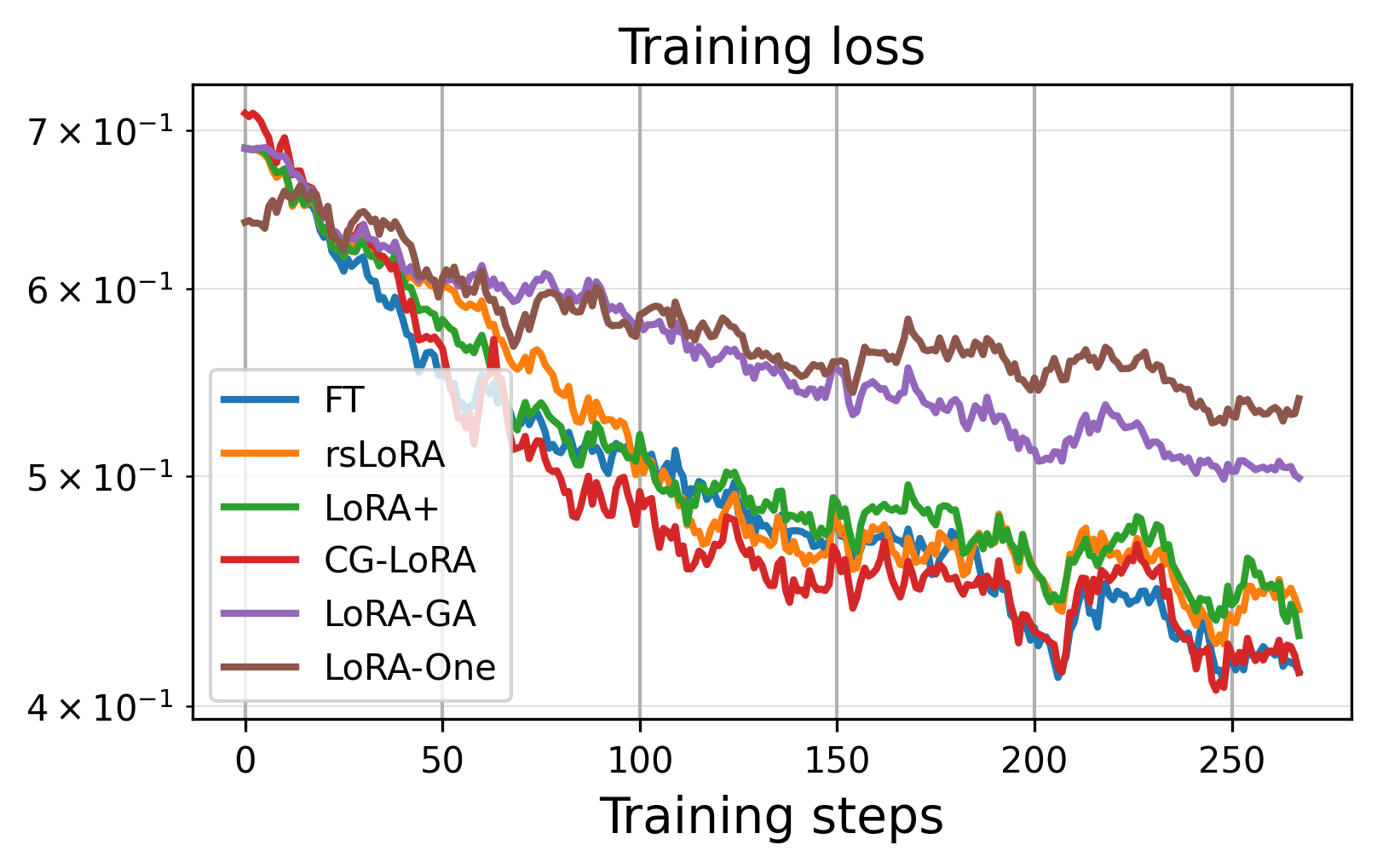} \includegraphics[width=0.45\linewidth]{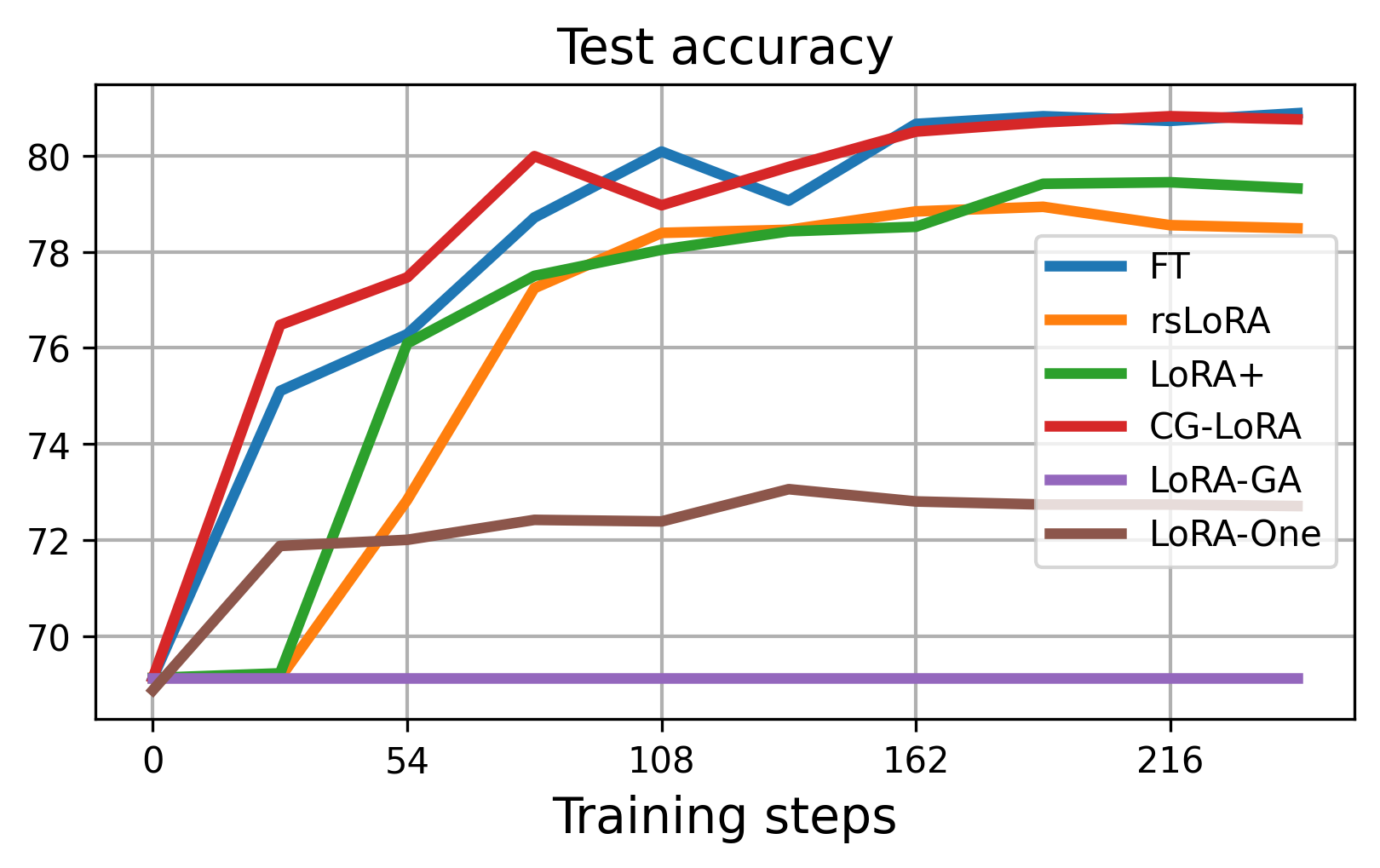} 
    \caption{Performance of finetuned RoBERTa-base on CoLA. CG-LoRA (no shift) is represented.}
    \label{fig:roberta_cola_loss}
\end{figure}


\begin{figure}[h!]
    \centering
    \includegraphics[width=0.45\linewidth]{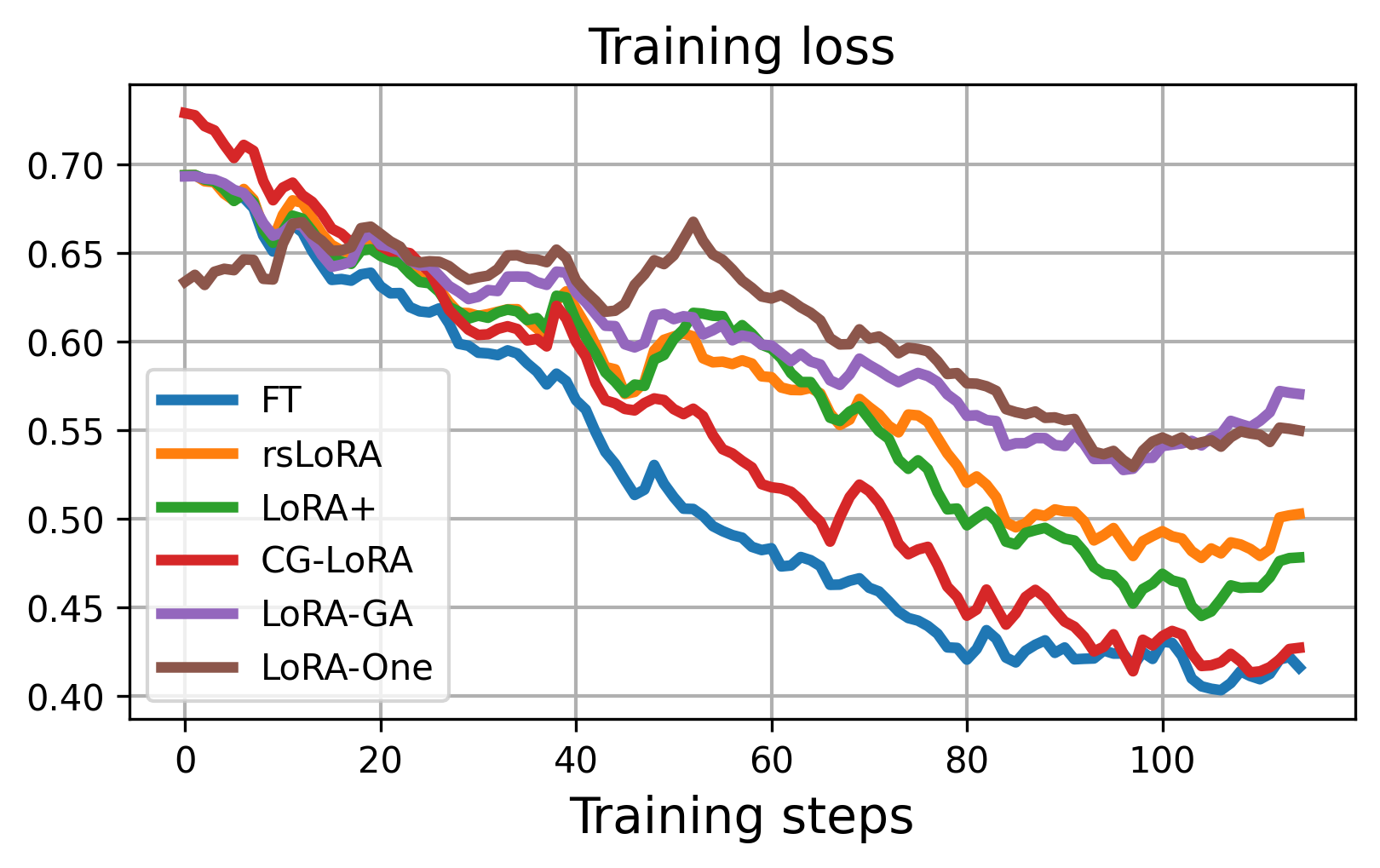} \includegraphics[width=0.45\linewidth]{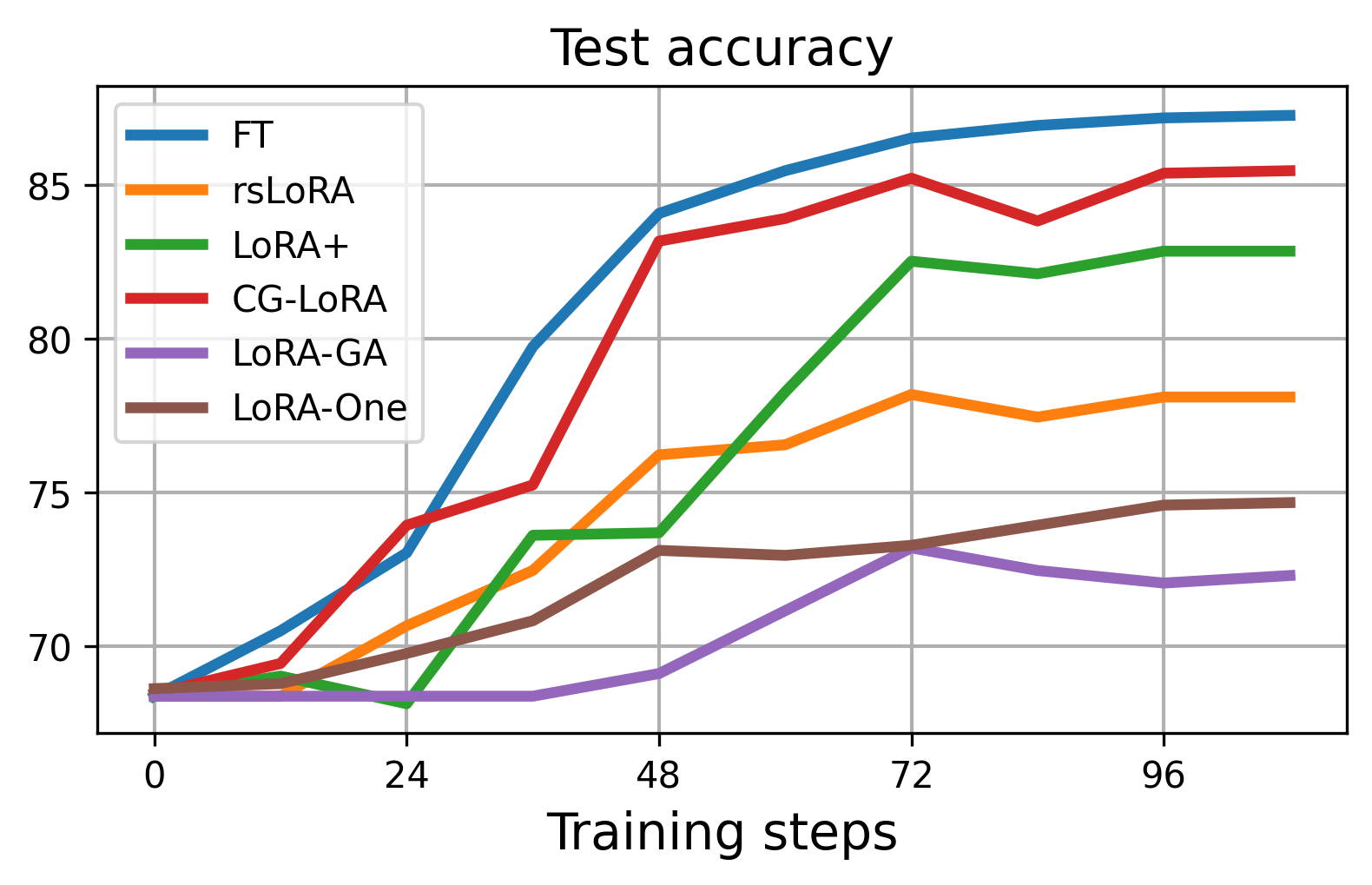} 
    \caption{Performance of finetuned RoBERTa-base on MRPC. }
    \label{fig:roberta_mrpc}
\end{figure}

\begin{figure}[h!]
    \centering
    \includegraphics[width=0.45\linewidth]{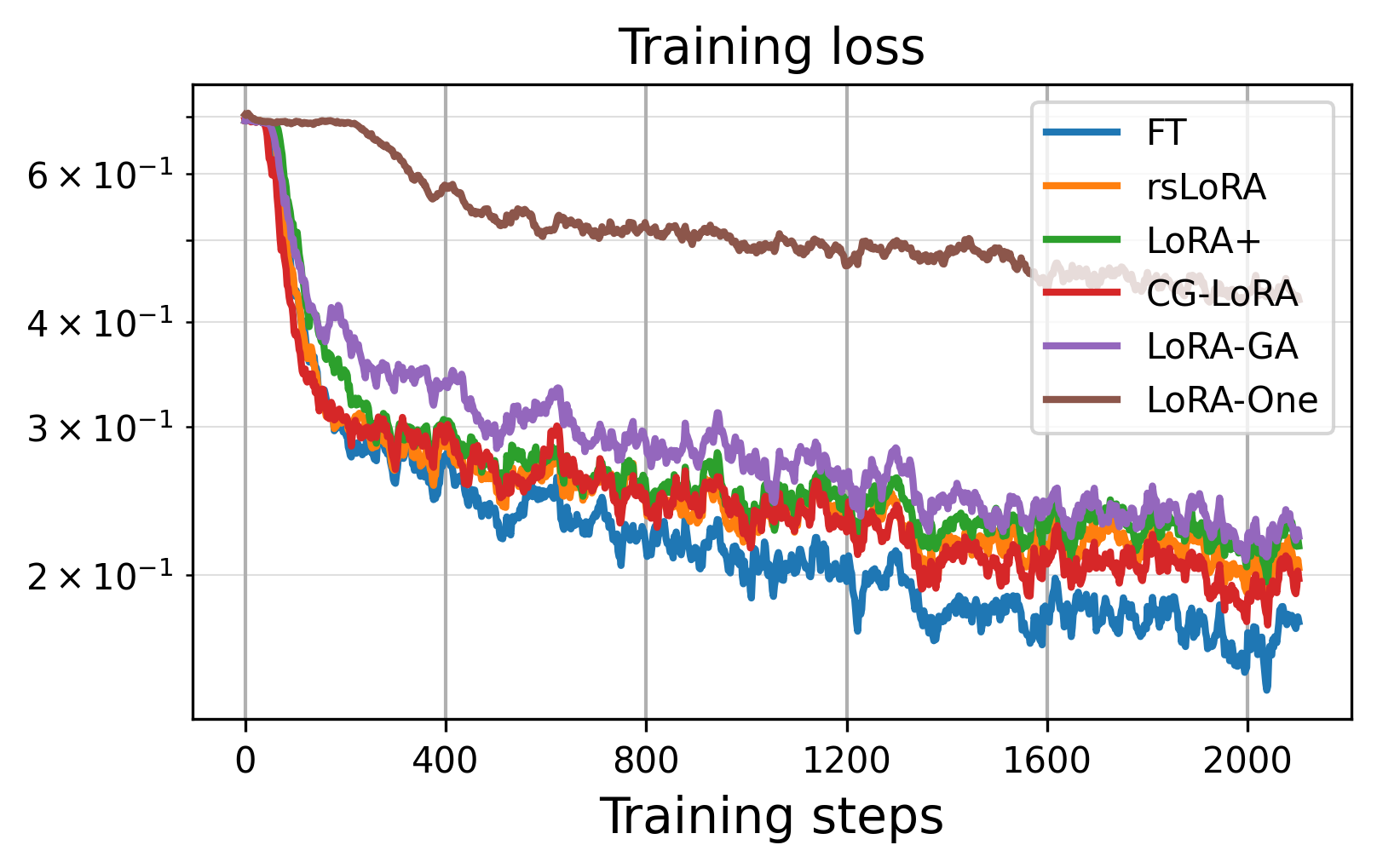} \includegraphics[width=0.45\linewidth]{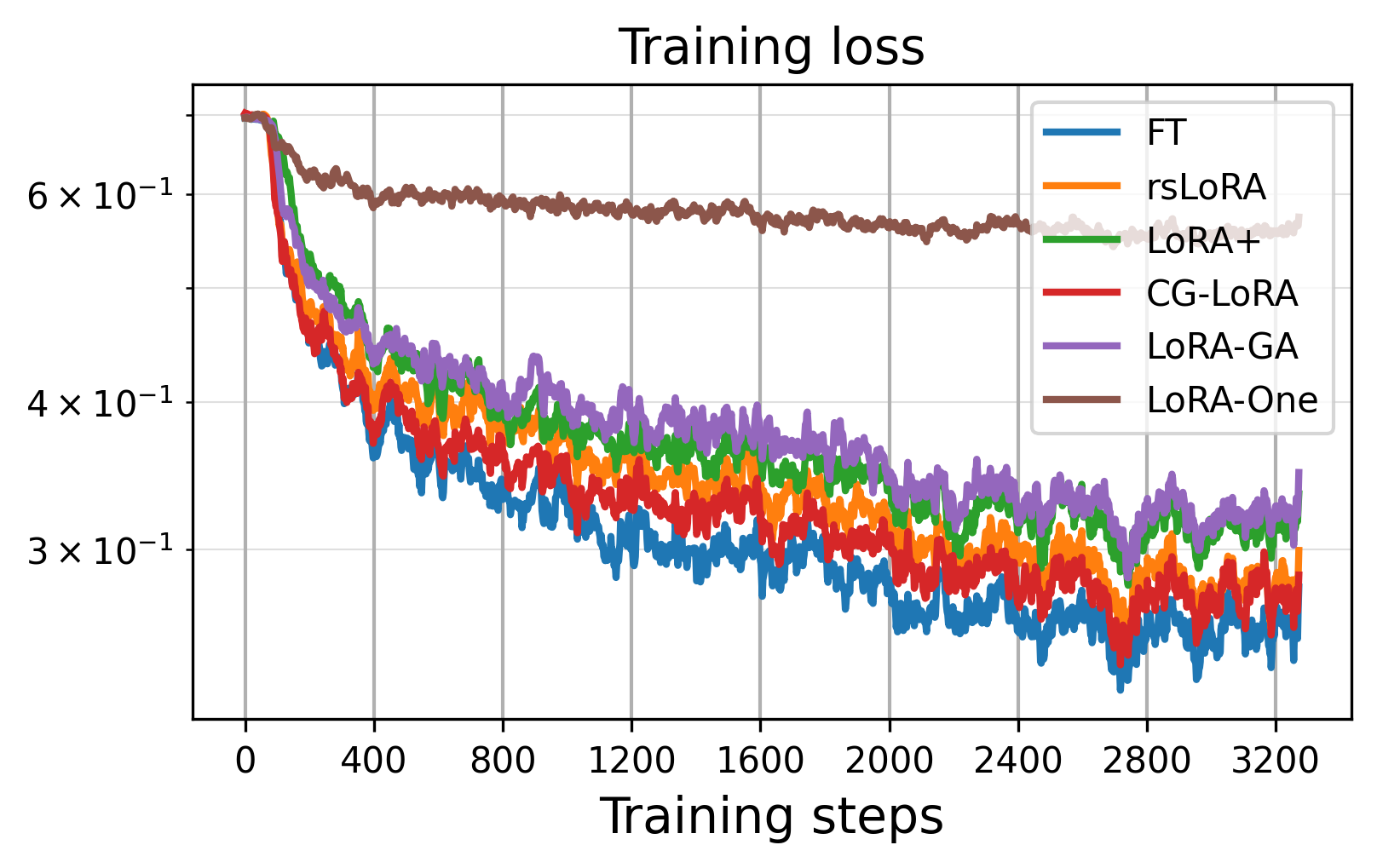} 
    \caption{Performance of finetuned RoBERTa-base on SST2, QNLI. }
    \label{fig:roberta_sst2}
\end{figure}

\begin{figure}[h!]
    \centering
    \includegraphics[width=0.38\linewidth]{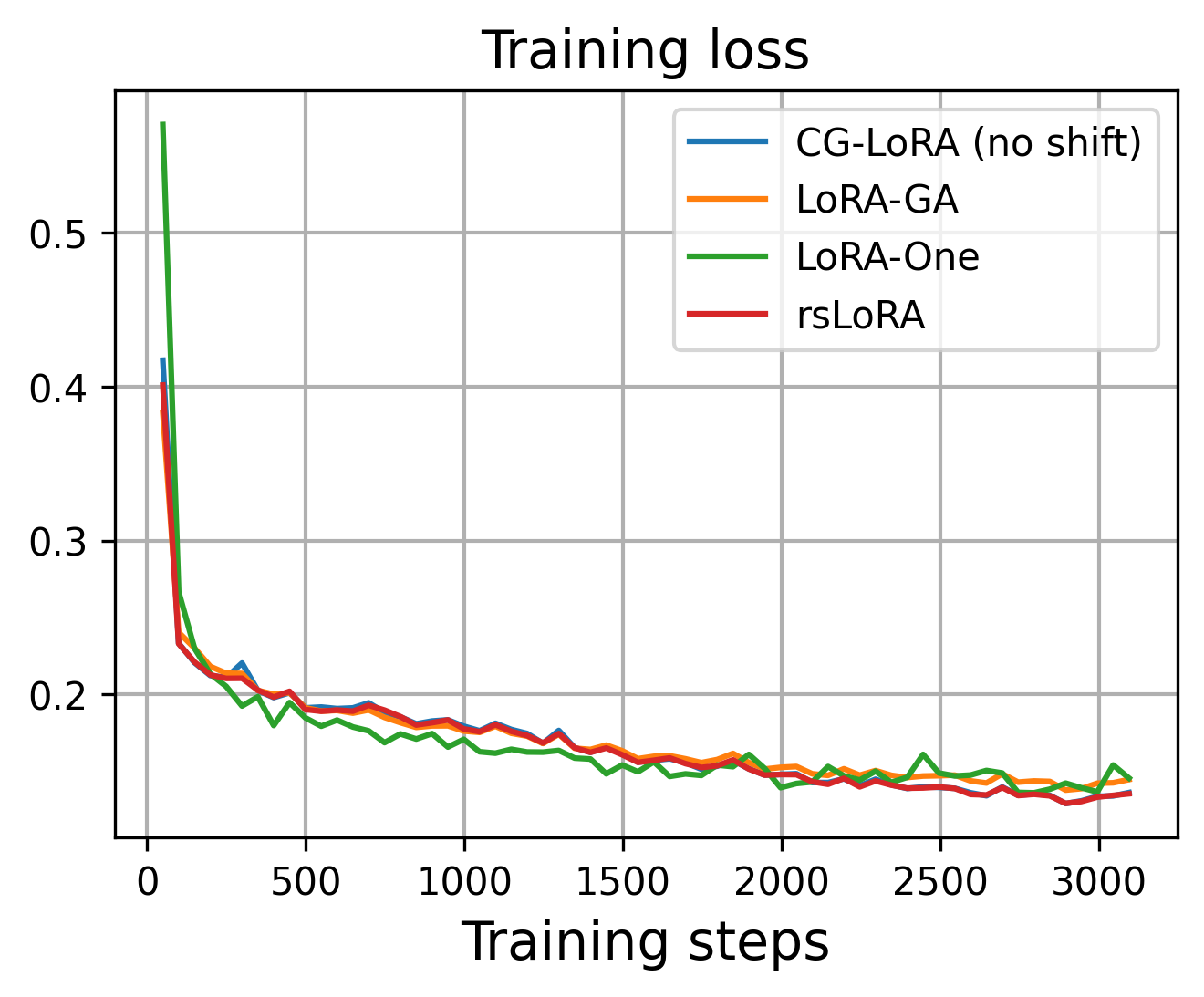} \includegraphics[width=0.45\linewidth]{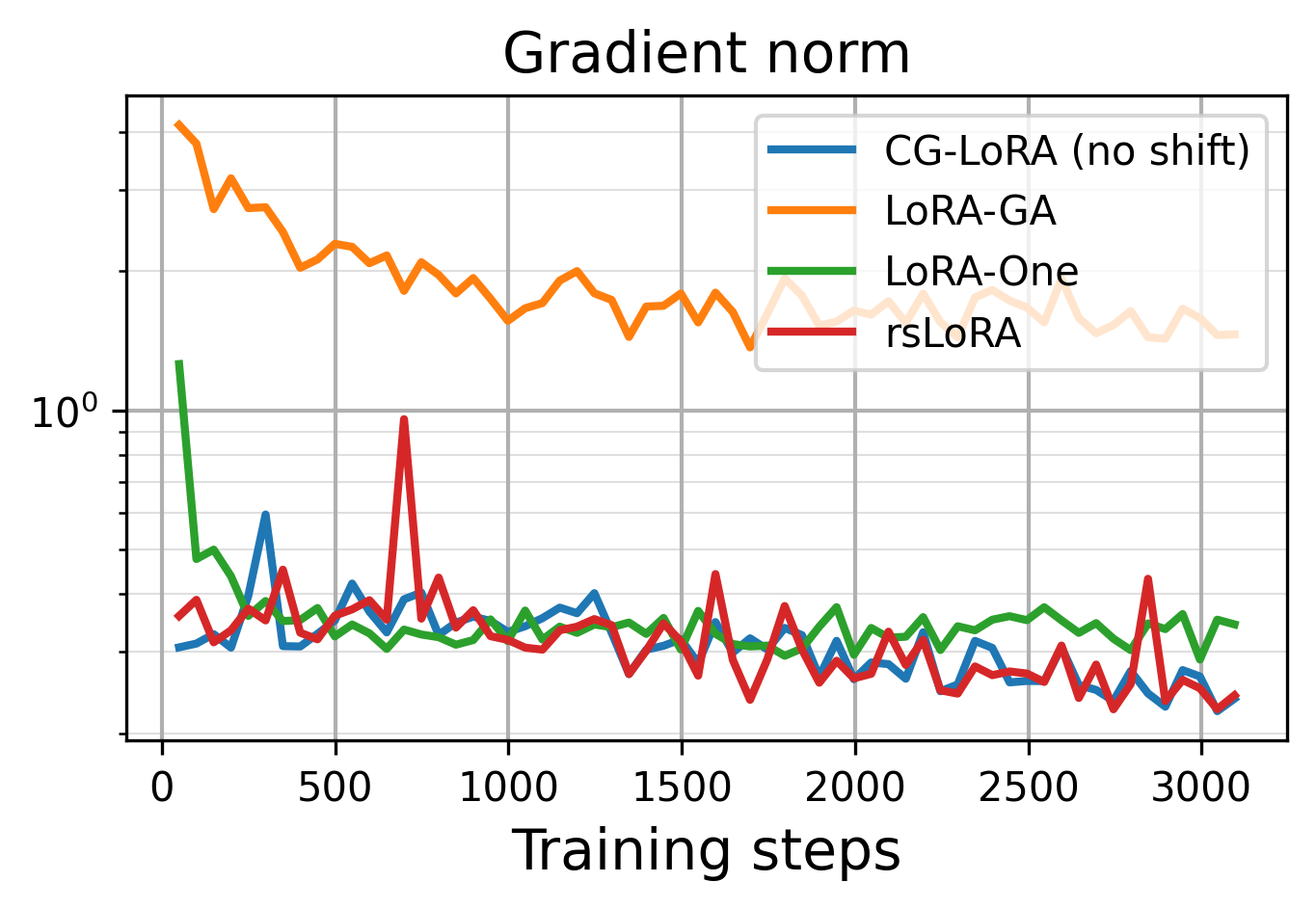} 
    \caption{Training dynamics of finetuned LLaMa 2-7B on MetaMath. }
    \label{fig:llama_metamath}
\end{figure}

\begin{figure}[h!]
    \centering
    \includegraphics[width=0.4\linewidth]{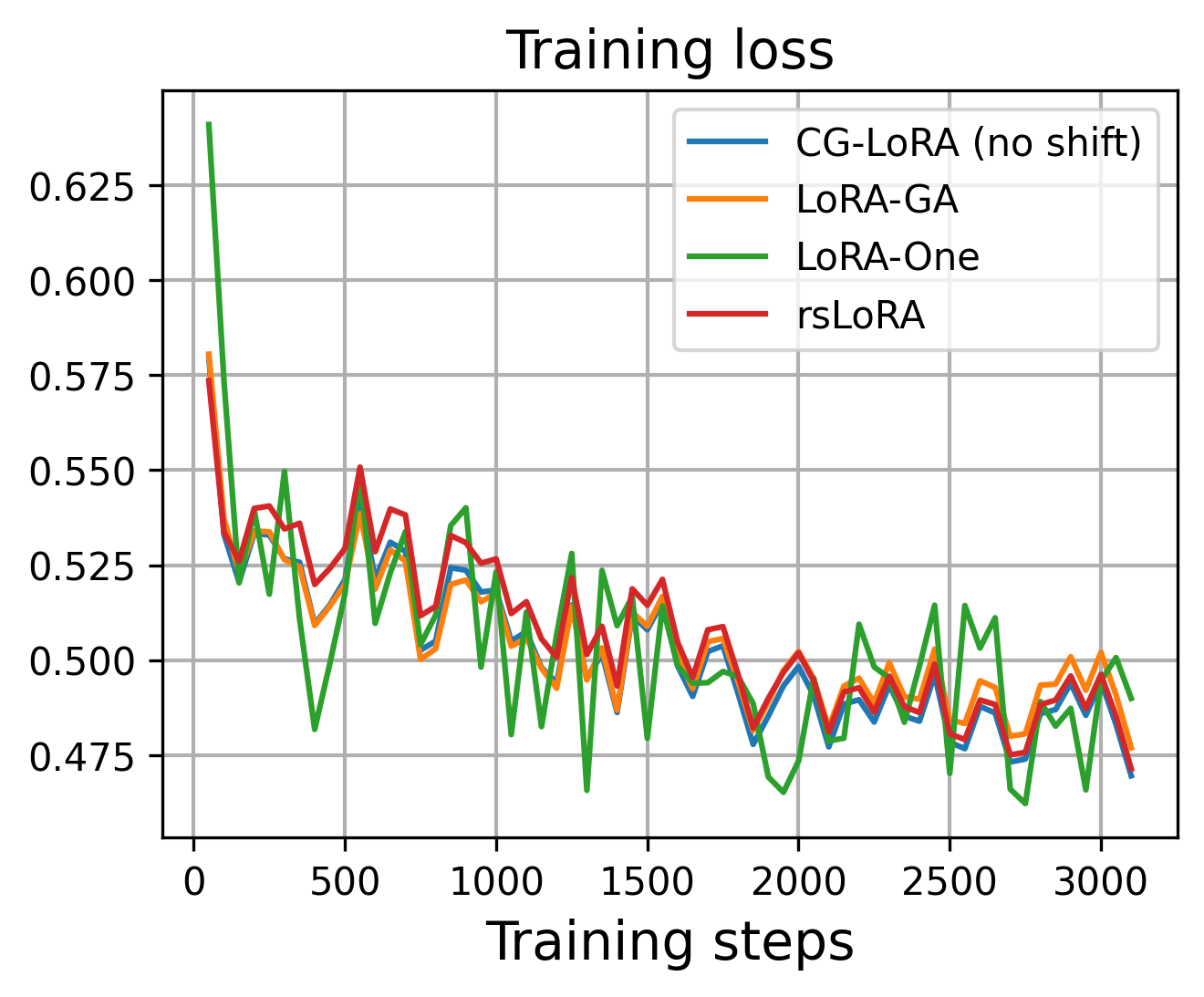} \includegraphics[width=0.45\linewidth]{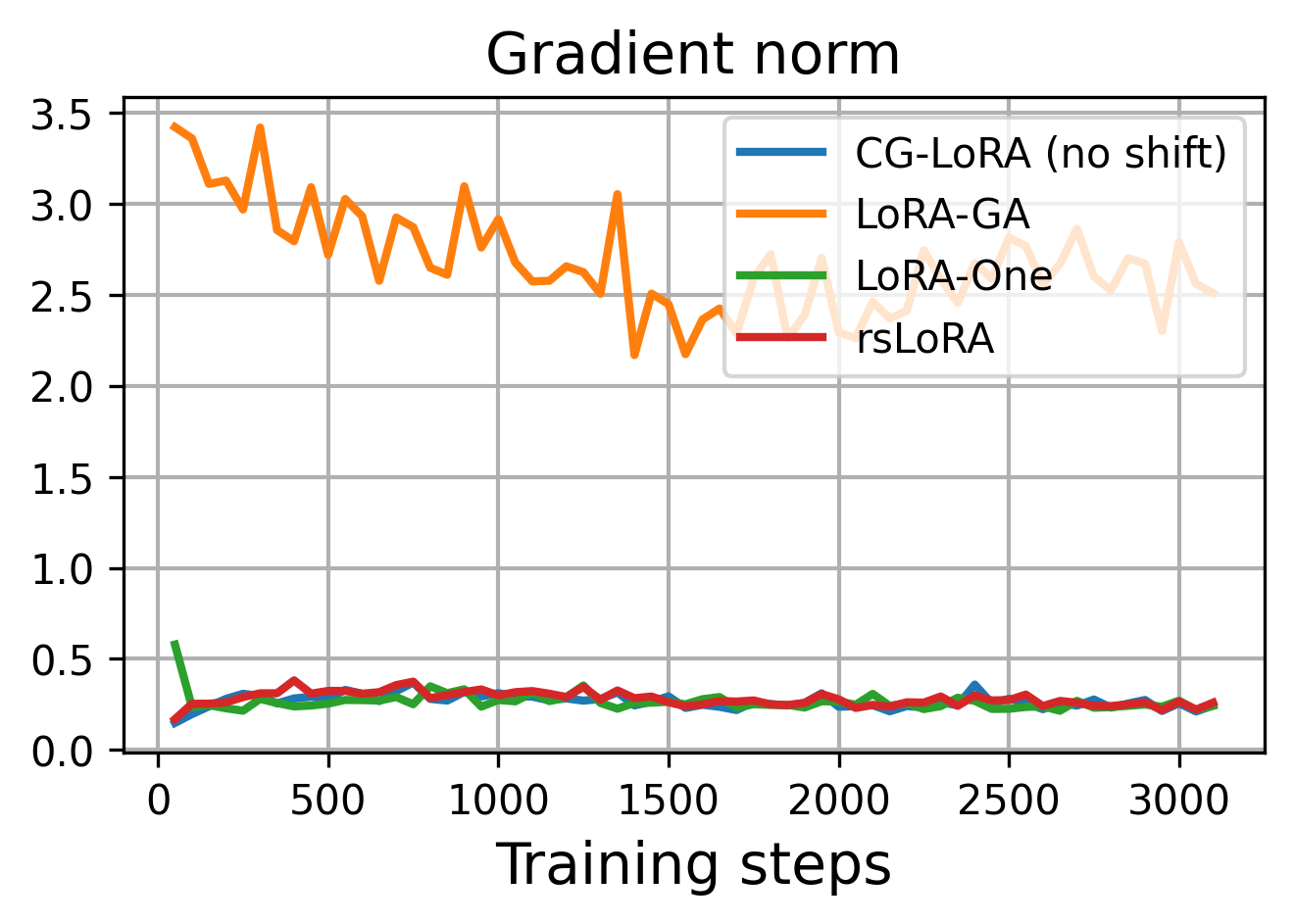} 
    \caption{Training dynamics of finetuned LLaMa 2-7B on CodeFeedback. }
    \label{fig:llama_code}
\end{figure}

\begin{figure}[h!]
    \centering
    \includegraphics[width=0.4\linewidth]{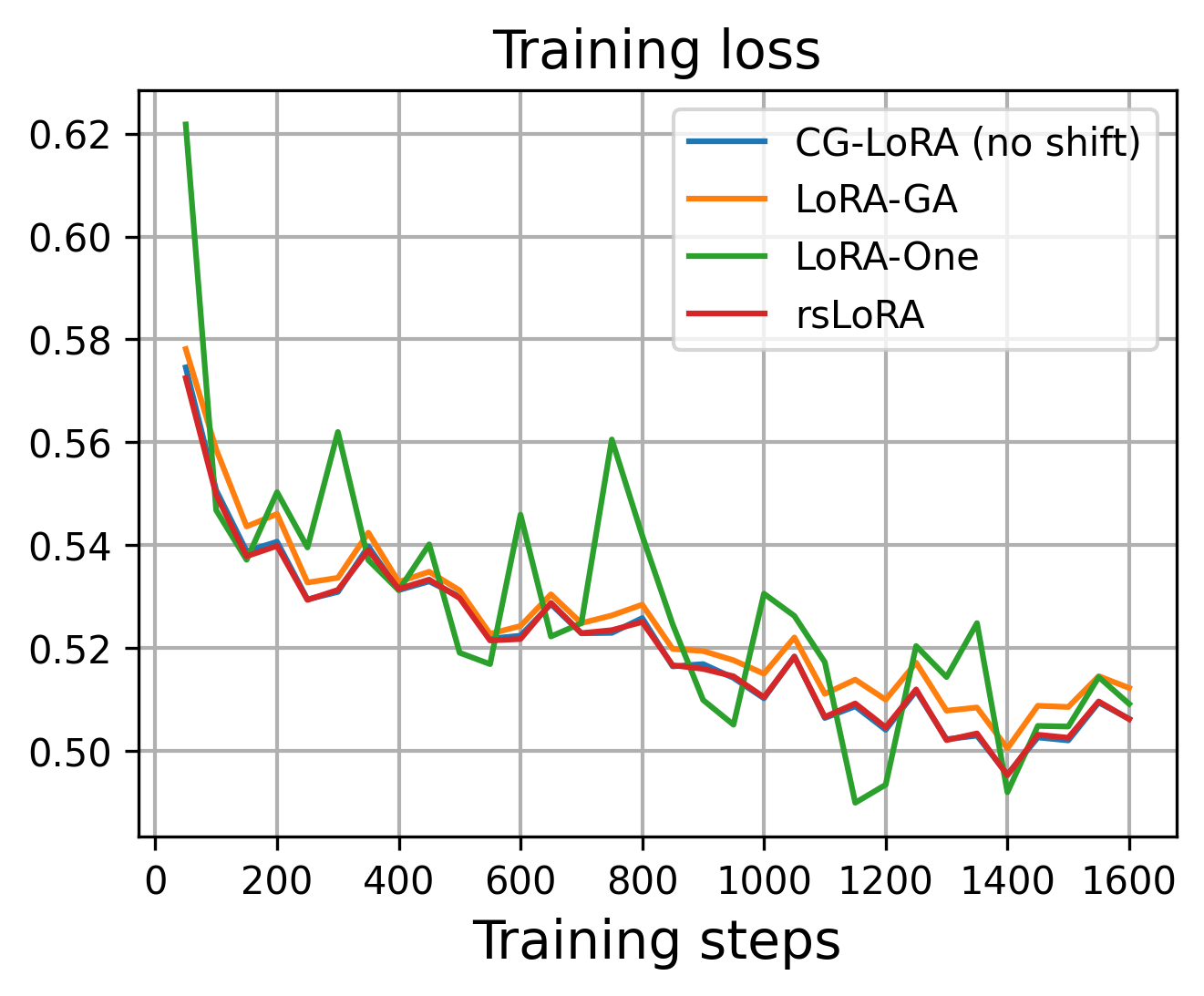} \includegraphics[width=0.45\linewidth]{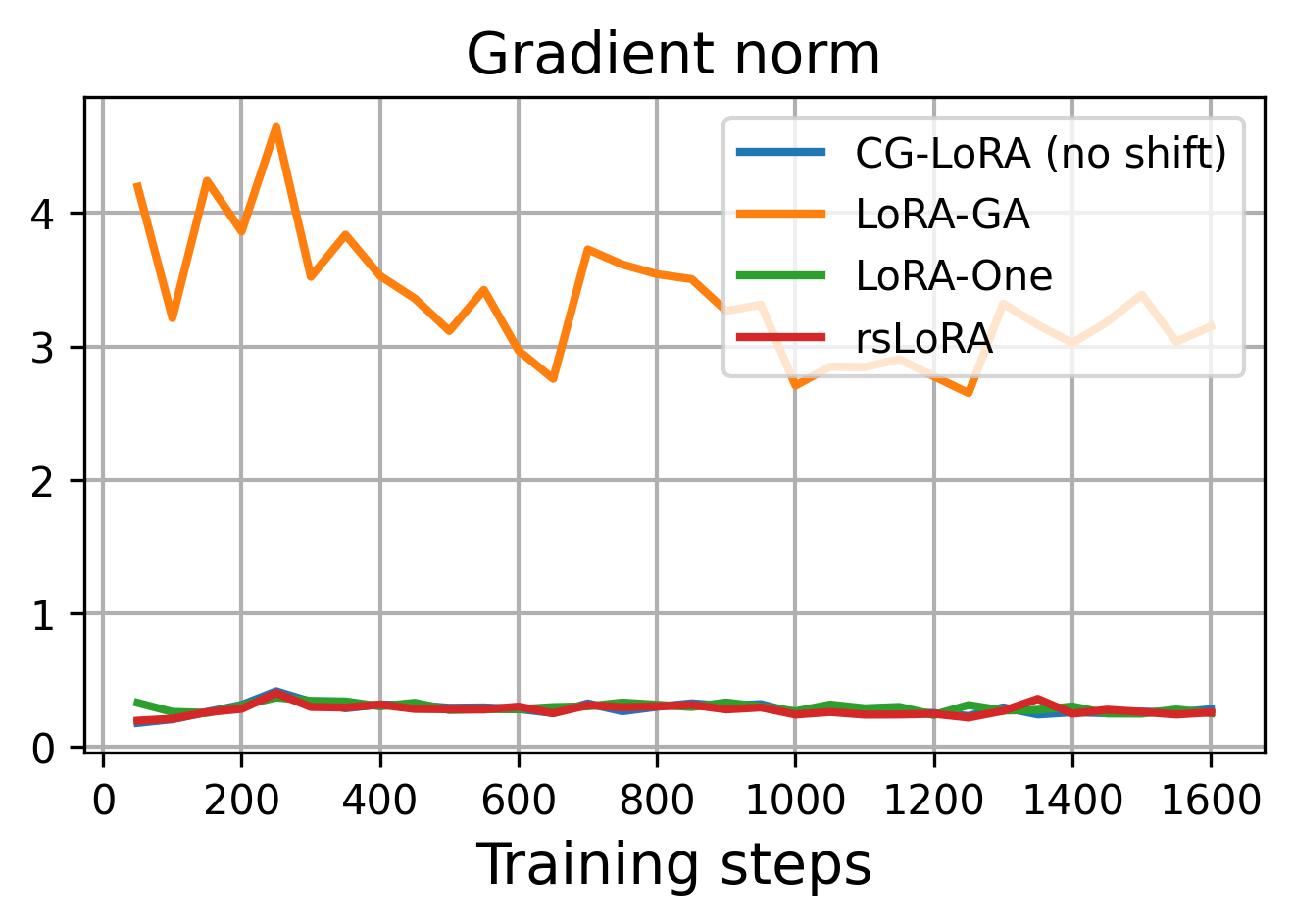} 
    \caption{Training dynamics of finetuned LLaMa 2-7B on WizardLM. }
    \label{fig:llama_wizard}
\end{figure}
\newpage

\section{Ablation study}\label{app:abla}



\subsection{On the cross-entropy loss correction}\label{app:ce_vs_squared}

In this part, we investigate whether our proposed Algorithm \ref{alg:cg-lora_ce} for the cross-entropy loss truly improves finetuning performance compared to our initial analysis with the square-loss resulting in Algorithm \ref{alg:cg-lora}. In Table \ref{tab:ce_vs_squaured}, we compare the performance of both RoBERTa-base and T5-base on the two small GLUE classification tasks, CoLA and MRPC, with and without the correction. While both versions are competitive especially when compared to previous baselines, our proposed correction can drastically improve downstream task performance.
\begin{table}[h!]
\small
\centering
\caption{Test accuracy of finetuned CG-LoRA (no shift) with and without cross-entropy correction. LoRA rank $r=8$. Query/Value LoRA layers.  }

\vspace{1em}
\label{tab:ce_vs_squaured}

\begin{tabular}{lcc}
\toprule
& \textbf{CoLA} & \textbf{MRPC}  \\
\midrule
RoBERTa-base, Algo \ref{alg:cg-lora_ce}

& \textbf{80.76} $_{\pm0.45}$ 
& \textbf{85.04}$_{\pm 1.70}$ 
 \\
RoBERTa-base,  Algo  \ref{alg:cg-lora}

& $80.66_{\pm 0.32}$
& $82.92_{\pm1.40}$
 \\
 \midrule
 T5-base, Algo \ref{alg:cg-lora_ce}
& \textbf{80.24} $_{\pm0.61}$ 
& \textbf{85.94}$_{\pm 0.50}$ 
 \\
T5-base,  Algo  \ref{alg:cg-lora}

& \textbf{79.89} $_{\pm0.19}$ 

& $84.88_{\pm1.10}$
 \\

\bottomrule
\end{tabular}
\end{table}

\subsection{Exact versus approximate computation of output derivatives}

We test finetuning RoBERTa-base with CG-LoRA on CoLA, using our proposed Hutchinson estimator in Subsection \ref{app:hutchinson} with different number of Hutchinson probes $P$. In this case, we already know that exact computation is possible by backpropagating each output coordinate ($C=2$). Nonetheless, we still seek to investigate in that simpler case whether  our theoretically proposed Hutchinson estimator is a valid alternative. Figure \ref{fig:probes} shows that for small values of $P$, we are close to the exact case.
\begin{figure}[h!]
    \centering
    \includegraphics[width=0.45\linewidth]{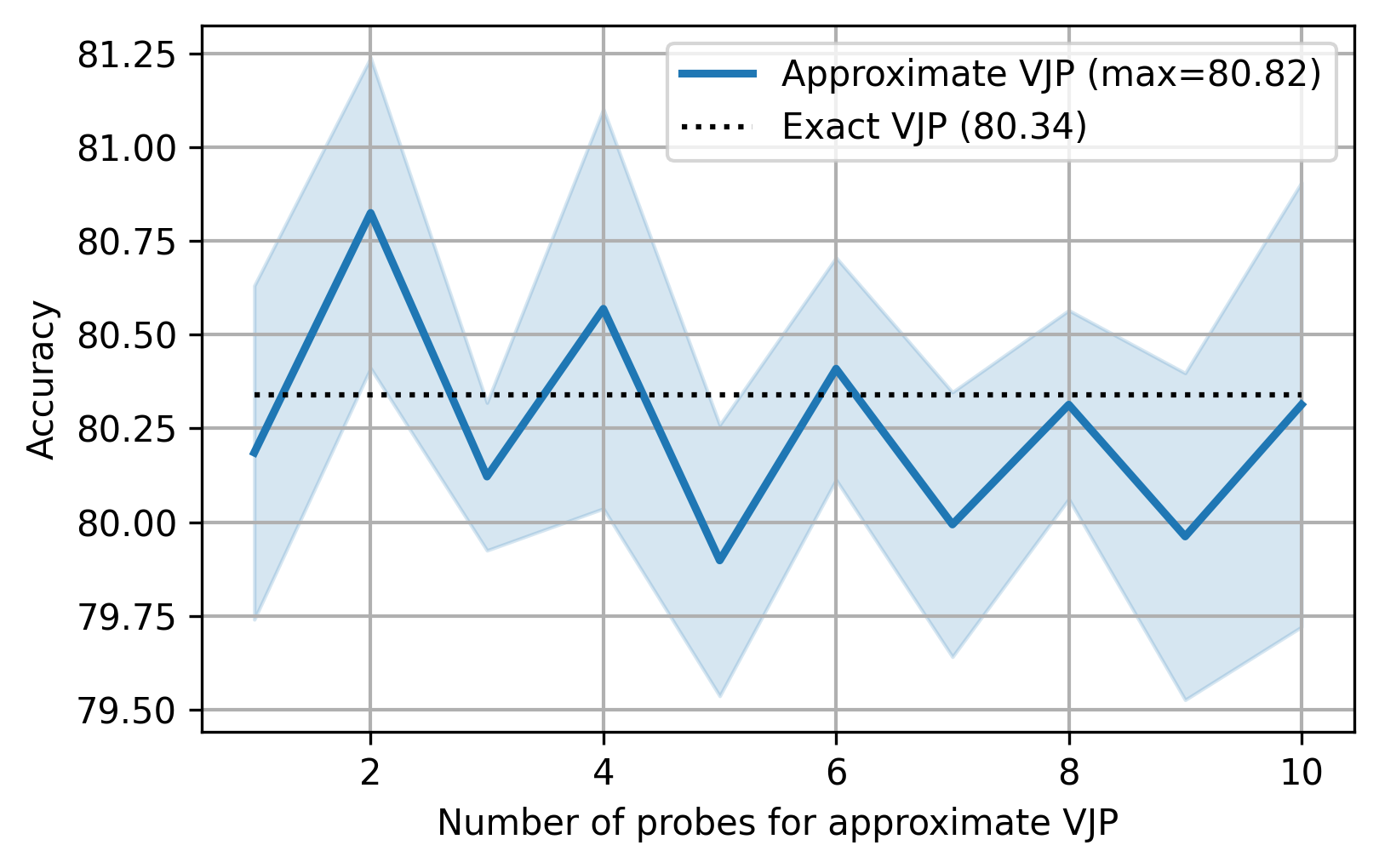}
    \caption{Accuracy versus number of Hutchinson probes}
    \label{fig:probes}
\end{figure}

\end{document}